\documentclass[journal]{IEEEtran}
\usepackage{amsmath, amsfonts, amsthm, amssymb}
\usepackage{algorithm, algorithmicx, algpseudocode}
\usepackage{bm}
\usepackage{array}
\usepackage[caption=false,font=normalsize,labelfont=sf,textfont=sf]{subfig}
\usepackage{textcomp}
\usepackage{stfloats}
\usepackage{url}
\usepackage{verbatim}
\usepackage{balance}
\usepackage{glossaries}
\usepackage[svgnames, dvipsnames, table, xcdraw]{xcolor}
\usepackage{mathtools}
\usepackage{tabularray}
\usepackage{multirow}
\usepackage{siunitx}
\usepackage{enumitem}
\usepackage[hidelinks]{hyperref}

\def\BibTeX{{\rm B\kern-.05em{\sc i\kern-.025em b}\kern-.08em
    T\kern-.1667em\lower.7ex\hbox{E}\kern-.125emX}}

\newacronym{rl}{RL}{Reinforcement Learning}
\newacronym{drl}{DRL}{Deep Reinforcement Learning}
\newacronym{hrl}{HRL}{Hierarchical Reinforcement Learning}
\newacronym{saferl}{SafeRL}{Safe Reinforcement Learning}
\newacronym{avi}{AVI}{Approximate Value-Iteration}
\newacronym{api}{API}{Approximate Policy-Iteration}
\newacronym[plural=MDPs, firstplural=Markov Decision Processes (MDPs)]{mdp}{MDP}{Markov Decision Process}
\newacronym{cmdp}{CMDP}{Constrained Markov Decision Processes}
\newacronym{safeexp}{SafeExp}{Safe Exploration}
\newacronym{kl}{KL}{Kullback-Leibler Divergence}
\newacronym{gae}{GAE}{Generalized Advantage Estimation}
\newacronym{papi}{PAPI}{Projections for Approximate Policy Iteration}
\newacronym{her}{HER}{Hindsight Experience Replay}
\newacronym{ham}{HAM}{Hierarchy of Abstract Machines}
\newacronym{mom}{MOM}{Measure of Manipulability}
\newacronym{bo}{BO}{Bayesian Optimization}
\newacronym{hebo}{HEBO}{Heteroscedastic Evolutionary Bayesian Optimisation}
\newacronym{ucb}{UCB}{Upper Confidence Bound}
\newacronym{pi}{PI}{Probability of Improvement}
\newacronym{ei}{EI}{Expected Improvement}
\newacronym{nl}{NLP}{Nonlinear Programming}
\newacronym{lp}{LP}{Linear Programming}
\newacronym{qp}{QP}{Quadratic Programming}
\newacronym{aqp}{AQP}{Anchored Quadratic Programming}
\newacronym{ode}{ODE}{Ordinary Differential Equation}
\newacronym{atacom}{ATACOM}{Acting on the TAngent Space of the COnstraint Manifold}
\newacronym{ivp}{IVP}{Initial Value Problem}
\newacronym{rref}{RREF}{Reduced Row Echlon Form}
\newacronym{rcef}{RCEF}{Reduced Column Echlon Form}
\newacronym{cpo}{CPO}{Constrained Policy Optimization}
\newacronym{trpo}{TRPO}{Trust Region Policy Optimization}
\newacronym{rmp}{RMP}{Riemannian Motion Policies}
\newacronym{dnn}{DNN}{Deep Neural Networks}
\newacronym{sdf}{SDF}{Signed Distance Function}
\newacronym{redsdf}{ReDSDF}{Regularized Deep Signed Distance Fields}
\newacronym{apf}{APF}{Artificial Potential Fields}
\newacronym{hri}{HRI}{Human-Robot Interaction}
\newacronym{poi}{PoI}{Point of Interest}

\newacronym{mpc}{MPC}{Model Predictive Control}
\newacronym{dcs}{DCS}{Directly Controllable State}
\newacronym{dus}{DUS}{Directly Uncontrollable State}

\newacronym{cbf}{CBF}{Control Barrier Function}
\newacronym{iss}{ISS}{Input-to-State Stable}

\def\Int{\mathrm{Int}}
\def\Bound{\mathrm{\partial}}

\def\EV{\mathbb{E}}
\def\RR{\mathbb{R}}
\def\ME{\mathcal{E}}
\def\MS{\mathcal{S}}

\def\MM{\mathcal{M}}
\def\MC{\mathcal{C}}
\def\MD{\mathcal{D}}
\def\TanS{\mathrm{T}}

\def\vzero{{\bm{0}}}

\def\vmu{{\bm{\mu}}}

\def\vepsilon{\bm{\epsilon}}
\def\vvepsilon{\bm{\varepsilon}}

\def\vpsi{{\bm{\psi}}}

\def\va{{\bm{a}}}
\def\vb{{\bm{b}}}

\def\vn{{\bm{n}}}

\def\vp{{\bm{p}}}
\def\vq{{\bm{q}}}

\def\vs{{\bm{s}}}

\def\vu{{\bm{u}}}
\def\vv{{\bm{v}}}

\def\vx{{\bm{x}}}
\def\vy{{\bm{y}}}

\def\mA{{\bm{A}}}
\def\mB{{\bm{B}}}

\def\mG{{\bm{G}}}

\def\mJ{{\bm{J}}}

\def\mQ{{\bm{Q}}}

\def\mT{{\bm{T}}}
\def\mU{{\bm{U}}}
\def\mV{{\bm{V}}}

\def\mX{{\bm{X}}}

\def\mSigma{{\bm{\Sigma}}}

\newtheorem{theorem}{Theorem}
\newtheorem{proposition}{Proposition}
\newtheorem{definition}{Definition}
\newtheorem{assumption}{Assumption}
\newtheorem{lemma}{Lemma}
\theoremstyle{remark}
\newtheorem{remark}{Remark}
\newtheorem{example}{Example}

\newcommand{\reviseRone}[1]{\textcolor{black}{#1}}
\newcommand{\reviseRtwo}[1]{\textcolor{black}{#1}}
\newcommand{\reviseRthree}[1]{\textcolor{black}{#1}}

\newcounter{reviewer}
\newcounter{point}[reviewer]
\renewcommand{\thepoint}{P\,\thereviewer.\arabic{point}} 

\newcommand{\shortreply}[2][]{\medskip \noindent \begin{sf}\textbf{Reply}:\  #2
	\ifthenelse{\equal{#1}{}}{}{ \hfill \footnotesize (#1)}%
	\medskip \end{sf}}

\def\StateN{s}
\def\StateVar{\bm{\StateN}}
\def\StateControllableN{q}
\def\StateControllableVar{\bm{\StateControllableN}}
\def\StateUncontrollableN{z}
\def\StateUncontrollableVar{\bm{\StateUncontrollableN}}
\def\Slack{\mu}
\def\SlackVar{\bm{\Slack}}

\def\StateSpace{\MS}
\def\ControlSpace{\mathcal{U}}
\def\AugmentedSpace{\MD}
\def\SafeSet{\MC}
\def\ConstrManifold{\MM}
\def\SingularSet{\mathcal{Y}}
\def\StateSpaceControllable{\mathcal{Q}}
\def\StateSpaceUncontrollable{\mathcal{Z}}

\def\DimState{S}
\def\DimConstr{K}
\def\DimConstrEquality{L}
\def\DimAugmentedState{N}
\def\DimControl{U}

\def\Constr{k}
\def\ConstrVec{\bm{\Constr}}
\def\ConstrEq{c}
\def\ConstrEquality{l}
\def\ConstrEqVec{\bm{\ConstrEq}}

\def\JacConstr{J_\Constr}
\def\JacConstrMat{\mJ_\Constr}
\def\JacConstrEquality{J_\ConstrEquality}
\def\JacConstrEqualityMat{\mJ_\ConstrEquality}
\def\JacConstrEq{J_\ConstrEq}
\def\JacConstrEqMat{\mJ_\ConstrEq}
\def\JacInputG{J_G}
\def\JacInputGMat{\mJ_G}
\def\JacInput{J_u}
\def\JacInputMat{\mJ_u}
\def\JacConstrControllable{J_\StateControllableN}
\def\JacConstrUncontrollable{J_\StateUncontrollableN}
\def\JacConstrControllableMat{\mJ_\StateControllableN}
\def\JacConstrUncontrollableMat{\mJ_\StateUncontrollableN}

\def\BasisVec{\vb}
\def\BasisTangentInput{B_u}
\def\BasisTangentInputMat{\mB_u}

\def\Dynf{f}
\def\DynfVec{\bm{\Dynf}}
\def\DynG{G}
\def\DynGMat{\bm{\DynG}}
\def\ControlInput{\vu_s}
\def\ControlSlack{\vu_\mu}

\def\ControlInputControllable{\vu_q}

\def\DynSlack{\alpha}
\def\DynSlackFun{A}
\def\DynSlackMat{\bm{\DynSlackFun}}

\def\DriftFun{\psi}
\def\DriftVec{\vpsi}

\hypersetup{
  colorlinks   = true, 
  urlcolor     = blue, 
  linkcolor    = black, 
  citecolor   = black 
}

\begin{document}
\title{Safe Reinforcement Learning on the Constraint Manifold: Theory and Applications}
\author{
Puze Liu \IEEEmembership{Member, IEEE}, 
Haitham Bou-Ammar, 
Jan Peters \IEEEmembership{Fellow, IEEE}, 
Davide Tateo \IEEEmembership{Member, IEEE}
\thanks{P. Liu and D. Tateo are the Intelligent Autonomous Systems Group at the Technical University of Darmstadt, Germany. Correspondence to: puze@robot-learning.de}
\thanks{H. Bou-Ammar is with Huawei R\&D London, United Kingdom}
\thanks{P. Liu and J. Peters are also with the Department of Systems AI for Robot Learning, German Research Center for AI} 
\thanks{J. Peters is also with the Hessian Centre for Artificial Intelligence and the Centre of
Cognitive Science}
}

\IEEEspecialpapernotice{
\textnormal{
\textcolor{magenta}{This work has been submitted to the IEEE for possible publication.\\
Copyright may be transferred without notice, after which this version may no longer be accessible.}}
\vspace{-2em}
}


\maketitle

\begin{abstract}
    Integrating learning-based techniques, especially reinforcement learning, into robotics is promising for solving complex problems in unstructured environments. 
    However, most existing approaches are trained in well-tuned simulators and subsequently deployed on real robots without online fine-tuning. 
    \reviseRone{In this setting, extensive engineering is required to mitigate the sim-to-real gap, which can be challenging for complex systems.
    Instead, learning with real-world interaction data offers a promising alternative: it not only eliminates the need for a fine-tuned simulator but also applies to a broader range of tasks where accurate modeling is unfeasible.}
    One major problem for on-robot reinforcement learning is ensuring safety, as uncontrolled exploration can cause catastrophic damage to the robot or the environment. Indeed, safety specifications, often represented as constraints, can be complex and non-linear, making safety challenging to guarantee in learning systems.
    In this paper, we show how we can impose complex safety constraints on learning-based robotics systems in a principled manner, both from theoretical and practical points of view. 
    Our approach is based on the concept of the Constraint Manifold, representing the set of safe robot configurations. 
    Exploiting differential geometry techniques, i.e., the tangent space, we can construct a safe action space, allowing learning agents to sample arbitrary actions while ensuring safety. 
    We demonstrate the method's effectiveness in a real-world Robot Air Hockey task, showing that our method can handle high-dimensional tasks with complex constraints. 
    Videos of the real robot experiments are available  \href{https://puzeliu.github.io/TRO-ATACOM}{publicly}\footnote{https://puzeliu.github.io/TRO-ATACOM}.
\end{abstract}

\begin{IEEEkeywords}
Safe Reinforcement Learning, Constraint Manifold, Safe Exploration
\end{IEEEkeywords}

\section{Introduction}
\IEEEPARstart{D}{eploying} robots in real-world environments to solve various tasks is a challenging objective. To achieve this objective, we need to solve many open problems in robotics, including perception, long-term planning, reactive motion generation, and interaction with humans.
Unfortunately, while extremely successful in controlled environments, classical robotics techniques struggle to deal with the complexity of the real world. 
To deal with these issues, researchers have started to incorporate machine learning approaches in robotic systems, allowing robots to achieve control performance and reactiveness to disturbances that are on par with or even outperform the best classical approaches existing, e.g., robot parkour~\cite{zhuang2023robot_parkour, hoeller2024anymal_parkour}, in-hand manipulation~\cite{chen2023visual}, and drone racing~\cite{song2023reaching}.
\begin{figure}
    \centering
    \includegraphics[width=\linewidth]{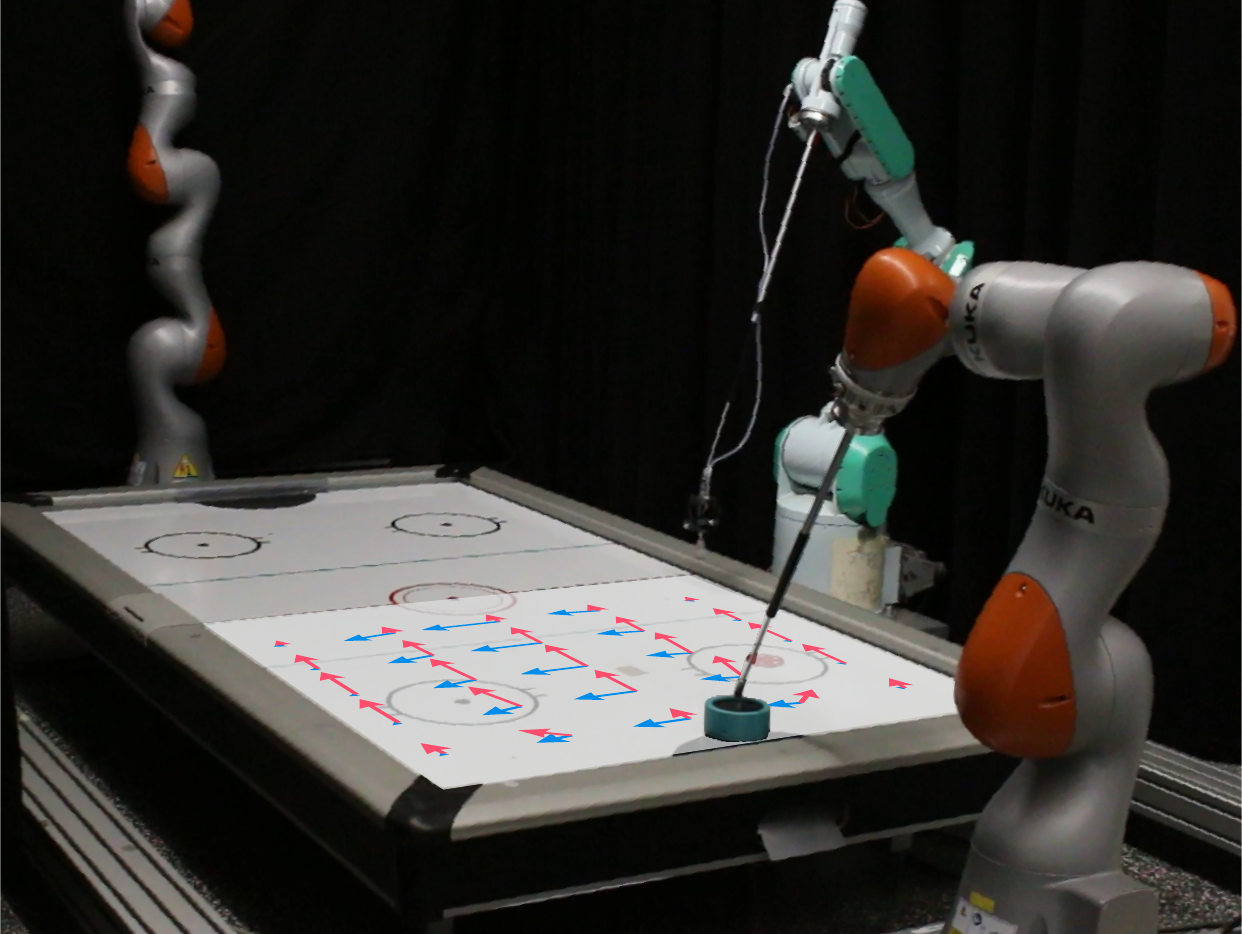}
    \caption{The robot air hockey task. The objective is to strike the puck to the opponent's goal. The vector field shows the velocity of the end-effector at different locations when a positive unit action is applied in the first two dimensions of the safety action space using \gls{atacom}. The blue (resp. red) arrow corresponds to a unit action applied in the first (resp. second) dimension.}
    \label{fig:real_robot_air_hockey}
\end{figure}
Nevertheless, the learning process is often only performed before deployment, as most existing approaches rely on offline data and simulators, not allowing online adaptation while interacting in the real world.
Indeed, allowing the robot control policy to change and improve during its deployment is essential to deal with dynamic environments, actuator wear and tear, the simulator-reality gap, malfunctioning sensors, and unexpected environmental conditions~\cite{ibarz2021train}.

Theoretically, this online adaptation is feasible using the \gls{rl} framework, which describes an optimization technique for the policy under a reward function. However, many key problems are preventing the usage of online \gls{rl} techniques in the real world. On the one hand, \gls{rl} method often requires an unreasonable large amount of data to learn the desired policy and to adapt to environmental variations. On the other hand, the learning process can arbitrarily change the policy and require explorative actions, possibly causing the robotic system to take dangerous actions that may harm people or cause damage to the environment or to the robot itself.
While the effort of the \gls{rl}, in general, and the Robot Learning community, in particular, has primarily focused on improving the efficiency of the learning algorithm, the importance of acting under safety constraints keeps increasing.
Complying with safety constraints is fundamental for allowing robots to operate in the real world, but, unfortunately, imposing safety constraints on a learning system is not as straightforward as in a classical robotic system.

Among many techniques to enforce safety constraints on learning agents, the most popular one is to exploit the \gls{saferl} framework. This framework is based on the \gls{cmdp} formalism~\cite{altman1999constrained}, and the objective is to ensure the learned policy fulfills the safety constraints while maximizing the reward function. By construction, these methods do not ensure the fulfillment of safety constraints during the learning process; they only converge to a final safe policy. Therefore, these algorithms are unsuitable for online adaptation of the learning agents in the real world, where any safety constraint violation should be avoided during the training process.
\reviseRone{
To address this challenge, researchers have developed Safe Exploration approaches that ensure safety throughout the entire learning process~\cite{hans2008safe, moldovan2012safe} by integrating prior knowledge, such as robot dynamics~\cite{shao2021reachability}, pre-defined constraints~\cite{ames2019control}, or demonstrations~\cite{robey2020learning}, to derive effective safety constraints. Once these constraints are established, safety is typically enforced using a safety filter. This filter can involve solving a constrained optimization problem, such as a quadratic program (QP), or substituting unsafe actions with pre-defined backup actions. Alternatively, \gls{atacom} provides a novel approach that constructs a safe action space, eliminating the need for solving a constrained optimization iteratively or designing  a backup policy.
}

In this paper, we extend and analyze in depth the method, \gls{atacom}, which was proposed in~\cite{liu2022robot} and extended in~\cite{liu2023safe}. \gls{atacom} is a simple but effective approach for achieving safe exploration 
\reviseRone{exploiting the knowledge of the robot dynamics and constraints.} 
\gls{atacom} is based on the concept of \textit{Constraint Manifold}~\cite{boumal2023introduction, berenson2009manipulation}, which converts the constraint optimization to an optimization problem on the manifold. We build a safe action space on the tangent space of the constraint manifold, leveraging the known dynamics of the robotic system and tools from differential geometry. Using this approach, any action sampled from the safe action space will keep the system on the constraint manifold, ensuring safety. 
In this paper, we extend the \gls{atacom} method with following contributions:
\begin{enumerate}[wide, labelindent=0pt]
    \item \reviseRthree{We provide a new formulation of the method by constructing a dynamic model for the slack variable. Despite the mathematical equivalence, the new formulation provides an interpretation from the dynamical system perspective, allowing us to conduct a theoretical analysis using LaSalle’s principle and input-to-state stability.}
    \item \reviseRthree{We analyze the \gls{atacom} controller with the existence of disturbances, showing that \gls{atacom} controller ensures bounded constraint violations with bounded disturbances.}
    \item \reviseRthree{We present a novel method for constructing a smoothly varying tangent basis, effectively addressing the issue of discontinuous bases.}
    \item We show how to extend the proposed method to different environment settings, including partially controllable systems, high-order dynamics, and equality constraints. 
    \item We comprehensively evaluate individual techniques introduced in the method using simple low-dimensional tasks to gain a better understanding.
    \item \reviseRthree{For the first time, we demonstrate that ATACOM can perform online fine-tuning of the policy in a contact-rich, dynamic task (Robot Air Hockey) safely, whereas former works on \gls{atacom} only deployed pre-trained policies on real robots.}
\end{enumerate}
The paper is organized as follows: Section~\ref{sec:related_work} provides a short overview of the related work in the field of safe reinforcement learning. Then, we introduce some key ideas from differential geometry and Lasalle's principle in Section~\ref{sec:preliminaries}. In Section~\ref{sec:atacom}, we focus on the theoretical analysis of the method, showing that \gls{atacom} is guaranteed to be safe under mild conditions. Next, we introduced practical techniques for the method implementation in Section~\ref{sec:practical_implementation}. Furthermore, extensions to different environment settings, experiments, and conclusions can be found in Section~\ref{sec:extensions}, \ref{sec:experiments}, and \ref{sec:conclusions}, respectively.
\section{Related Work}
\label{sec:related_work}
In the last decades, \gls{cmdp}~\cite{altman1998constrained,altman1999constrained} 
have garnered significant attention from \gls{rl} researchers for addressing constrained control problems. Various forms of constraints have been explored within this framework. One prominent constraint is the expected cumulative cost, where the objective is to maximize the expected return while ensuring that the expected cost remains below a specified threshold~\cite{achiam2017constrained, chow2018lyapunov, tessler2019reward, liu2020ipo, stooke2020responsive, ding2021provably, ammar2015safe, cowen2022samba, yu2019convergent}. This constraint has been adapted into different variants, including risk-sensitive constraint \cite{borkar2014risk, ying2022towards, kim2022efficient, yang2023safety} and the probabilistic constraint~\cite{wagener2021safe, peng2021separated, pfrommer2022safe}. 
\reviseRone{Several optimization techniques have been applied to update policies under constraints, such as the trust-region method~\cite{achiam2017constrained, kim2022efficient}, the interior point method~\cite{liu2020ipo}, Lagrangian relaxation~\cite{altman1998constrained, tessler2019reward, stooke2020responsive, ding2021provably, borkar2014risk, ying2022towards, yang2023safety}, and Lyapunov function-based methods~\cite{chow2018lyapunov, chow2019lyapunov, sikchi2021lyapunov}. While these approaches aim to derive a safe policy by the end of training, they often violate constraints during the learning process, making them unsuitable for direct application on real robots or for online fine-tuning tasks.}

To address these shortcomings, \gls{safeexp} focuses on satisfying state-dependent constraints at every time step, and it presents two significant challenges: (1) constructing safe constraints~\cite{wachi2024survey} and (2) obtaining safe action. 
\reviseRone{When constructing constraints, they should account for the safety of both the current state and future trajectories. 
To address the first challenge, methods such as control barrier functions~\cite{ames2019control,xiao2022high_order, taylor2020learning, cheng2019end} and reachability analysis~\cite{fisac2018general, selim2022safe,zheng2024safe, shao2021reachability, kochdumper2023provably} utilize prior knowledge of the system dynamics or learned dynamics model to construct the safety constraint. To alleviate the engineering effort of manually designing constraints, learning-based techniques have been applied to learn a safety constraint~\cite{dawson2023safe, lavanakul2024safety}. In this paper, we assume the presence of state-dependent constraints and focus on the challenge of obtaining safe actions.} 

\reviseRone{To obtain a safe actions, the most common choice  is to add a safety layer on top of the policy, which will correct the unsafe action to a safe one~\cite{fisac2019bridging, pham2018optlayer,dalal2018safe}. }
Nevertheless, determining an action that satisfies the safety constraint or corrects the constraint violation is almost impossible without incorporating prior knowledge. 
Exploiting the dynamics model to find the safe action has been applied to safe-critical control and learning tasks~\cite{ames2019control, emam2022safe}. Alternatively, defining a task-specific backup policy enables safe exploration in low-dimensional tasks, such as navigation and pendulum~\cite{hans2008safe, garcia2012safe, berkenkamp2017safe}.  Learning dynamic models using offline datasets and deriving safe actions using the learned models are applicable to tasks where dynamic models are unavailable or inaccurate~\cite{koller2018learning, hewing2020learning, cheng2019end, dalal2018safe}. 
\reviseRone{Furthermore, exploiting the regularity of the Gaussian Process enables the agent to determine the next state to be explored without the requirement of known constraints~\cite{sui2015safe, berkenkamp2023bayesian, wachi2018safe}. An analytical model is then used to derive a policy to reach that state. However, this approach is limited to simple systems, as safely reaching the target state for high-dimensional systems remains challenging.}
\reviseRone{Our approach also exploits the knowledge of the dynamics model and the constraints to construct a safe action space.}
\reviseRone{Unlike existing approaches, \gls{atacom} does not require an initial safe policy, a backup policy, or solving a constrained optimization problem. Instead, \gls{atacom} constructs a safe action space where all sampled actions are guaranteed to be safe, making it compatible with any \gls{rl} algorithm.}

Safe learning for robotics, as the core intersection between the \gls{saferl} and robotics~\cite{brunke2022safe}, has been raised for various types of application scenarios, varying from manipulation \cite{sukhija2022scalable,martinez2015safe} and navigation \cite{weerakoon2022terp,bajcsy2019efficient}, to locomotion~\cite{marco2021robot} and \gls{hri}~\cite{pang2021towards,pandya2022safe, thumm2022provably, schepp2022sara}. Previous work of \gls{atacom} has demonstrated the effectiveness of the method in manipulation, navigation, and \gls{hri} tasks~\cite{liu2022robot,liu2023safe}. In this paper, we conduct a rigorous theoretical analysis, provide clear illustrations of individual components using simple low-dimensional environments, and ultimately validate the method through its application in a high-dimensional, real-world robotic task.

\reviseRone{\section{Preliminaries and Problem Statement}}
\label{sec:preliminaries}
\paragraph{\textbf{Differential Geometry}}
We briefly recall some concepts from differential geometry relevant to our approach. For a more comprehensive study of related topics, see \cite{lee2012smooth, boumal2023introduction}. Let $\mathcal{M}$ be a \textsl{Differentiable Manifold}, and $\TanS_p\MM$ denote the \textsl{Tangent Space} of the manifold $\MM$ at $p\in \MM$. The dimension of the manifold $\mathrm{dim}{\MM} = n$ if each point has a neighborhood that is homeomorphic (bijective and continuous) to an open subset of $\RR^n$. The \textsl{Tangent Bundle} $\TanS\MM = \sqcup_{p\in\MM} \TanS_p \MM$ is the disjoint union of all tangent space in $\MM$.
Suppose two smooth manifold $\MM$ and $\mathcal{N}$, given a smooth map $\Phi: \MM \rightarrow \mathcal{N}$, the \textsl{Rank} of $\Phi$ at $p\in \MM$ is the rank of the linear map $\mathrm{D}\Phi_p:\TanS_p{\MM}\rightarrow \TanS_{\Phi(p)}{\mathcal{N}}$. An \textsl{Embedded Submanifold} is a subset $\mathcal{S}\in \mathcal{M}$, the \textsl{Codimension} is the difference $\mathrm{dim}\MM - \mathrm{dim}\mathcal{S}$. Next, we present the \textsl{Constant-Rank Level Set Theorem}~\cite{lee2012smooth}, used later in section~\ref{sec:constraint_manifold} to define the constraint manifold.

\begin{theorem}[Constant-Rank Level Set Theorem]
Let $\MM$ and $\mathcal{N}$ be smooth manifolds, and let $\Phi: \MM \rightarrow \mathcal{N}$ be a smooth map with constant rank $r$, then each level set of $\Phi$ is a properly embedded submanifold of codimension $r$ in $\MM$.
\end{theorem}

Let $f:\ME \rightarrow \RR^k$ be a smooth function. $\ME$ is a Euclidean space of dimension $d>k$ with inner product $\langle \cdot, \cdot \rangle$ and induced norm $\| \cdot \|$. If $\mathrm{D}f(\vx)$ has full rank $k$ for all $x \in \MM$, The set 
\begin{equation}
    \MM = \left\{ \vx \in \ME: f(\vx)=\vzero \right\}
\end{equation}
is a (smooth) \textsl{Embedded Submanifold} of $\ME$ of dimension $d - k$. 

The \textsl{Tangent Space} at $\vx \in \MM$ is given by
\begin{align*}
    \TanS_x\MM 
    &= \left\{ \vv \in \ME: \langle \mathrm{grad}f_i(\vx), \vv \rangle= 0, i\in \{1, \dots, k\} \right\}
\end{align*}
The dimension of the tangent space is $\mathrm{dim}{\MM} = d-k$.


\paragraph{\textbf{LaSalle's Invariance Principle}}
We present the \textsl{LaSalle's Invariance Principle}~\cite{khalil2002control}, used later in the discussion to prove the safety. Consider the autonomous system 
\begin{equation}
    \dot{\vx} = f(\vx)
    \label{dynamic_system}
\end{equation}
where $f:D\rightarrow \RR^n$ is a locally Lipschitz map from a domain $D\subset \RR^n$ into $\RR^n$. We have the following definitions: 
\begin{definition}
A set $\MC \subset \RR^n$ is said to be 
\begin{itemize}
    \item an \textbf{invariant set} with respect to \eqref{dynamic_system} if: $\vx(0)\in \MC \Rightarrow \vx(t) \in \MC, \forall t\in \RR$
    \item a \textbf{positively invariant set} with respect to \eqref{dynamic_system} if: $\vx(0)\in \MC \Rightarrow \vx(t) \in \MC, \forall t\geq 0$
\end{itemize}
\end{definition}

\begin{theorem}[LaSalle’s Invariance Principle]
Let $\Omega \subset D \subset \RR^n$ be a compact positively invariant set concerning the dynamic system~\eqref{dynamic_system}. Let $V:D\rightarrow \RR$ be a continuously differentiable function such that $\dot{V}\leq 0$ in $\Omega$. Let $E := \{\vx \in \Omega: \dot{V}(\vx) =0\}$. Let $M$ be the largest invariant set in $E$. Then, every solution starting in $\Omega$ approaches $M$ as $t\rightarrow \infty$.
\end{theorem}

\paragraph{\textbf{Problem Statement}}
In \gls{saferl}, we model the environment as a \gls{cmdp}. A \gls{cmdp} is defined by the tuple $\langle \mathcal{S}, \mathcal{A}, P, \gamma, R, \mu_0, \mathcal{K} \rangle$ with the state space $\mathcal{S}$, the action space $\mathcal{A}$, the state transition probability kernel $P:\mathcal{S}\times \mathcal{A} \times \mathcal{S} \rightarrow \RR^{+}$, the discount factor $\gamma \in (0, 1]$, the reward function $R: \mathcal{S} \times \mathcal{A}\rightarrow \RR$, the initial state distribution $\mu_0: \mathcal{S} \rightarrow \RR^{+}$, and the set of constraint functions $\mathcal{K}:=\lbrace k_i:\mathcal{S}\rightarrow \RR, i \in \{1, \dots, K\}\rbrace$. 
In this paper, we approach the safety problem in \gls{rl} from the point of view of \textsl{Safe Exploration}. The objective is to prevent constraint violations throughout the whole learning process, considering the following constrained optimization problem:
\begin{align}
    \pi^{*} = \max_{\pi} & \quad \EV_{\tau \sim \pi} \left[\sum_{t=1}^{T}\gamma^t r(\vs_t, \va_t)\right], \label{eq:safe_exploration}\\ 
    \mathrm{ s.t.} & \quad k_i(\vs_t) \leq 0, \; \forall i \in \{1, \dots, K\}, \nonumber
\end{align}
with the trajectory $\tau = [\vs_0, \va_0, \dots, \vs_T]$, the state $\vs_t \in \mathcal{S} \subset \RR^S $, and the action $\va_t \in \mathcal{A}\subset \RR^A$ sampled from a policy $\pi$. The inequality constraints $k_i:\mathcal{S}\rightarrow \RR$ specify the safety requirements at each time step among all trajectories.

Modern \gls{rl} algorithms solve this optimization by evaluating a sequence of policies $\lbrace\pi_0, \dots \pi_j \rbrace$ in the environment, where the learning algorithm produces a new policy $\pi_{j+1}$ using $\mathcal{D}_j$, i.e. the dataset of environment interactions using the previous policies $\lbrace\pi_0, \dots,\pi_j\rbrace$. In this scenario, Safe Exploration requires that every state $\vs$ from every trajectory $\tau$ sampled from each policy $\pi_j$ is safe, i.e., $k_i(\vs) \leq 0, \forall i$. 
This requires that the optimization algorithm selects each policy $\pi_j$ from $\Pi_\text{safe}$, i.e., the space of policies respecting the safety constraints. 
Most safe exploration approaches design a space $\tilde{\Pi}\subseteq\Pi_\text{safe}$ and solve the optimization problem in \eqref{eq:safe_exploration} generating $\pi_j\in\tilde{\Pi}$.


In the rest of this work, we focus on the control affine system that is local Lipschitz continuous, described as follows:
\begin{equation}
    \dot{\StateVar} = \Dynf(\StateVar) + \DynG(\StateVar) \ControlInput
    \label{eq:nonlinear_affine_system}
\end{equation}
where $\StateVar \in \StateSpace \subset \RR^\DimState$ denotes the state of the system, $\ControlInput \in \ControlSpace \subset \RR^\DimControl$ is the $\DimControl$-dimensional control input, $\Dynf:\MS\rightarrow \RR^\DimState$ and $\DynG:\MS\rightarrow \RR^{\DimState \times \DimControl}$ are two Lipschitz mappings.  We make the following assumption 

\begin{assumption}
The state space $\StateSpace$ of interest is compact and positively invariant with respect to \eqref{eq:nonlinear_affine_system}.
\label{asm:compactness}
\end{assumption}
\begin{figure*}[t]
    \centering
    \includegraphics[width=\textwidth]{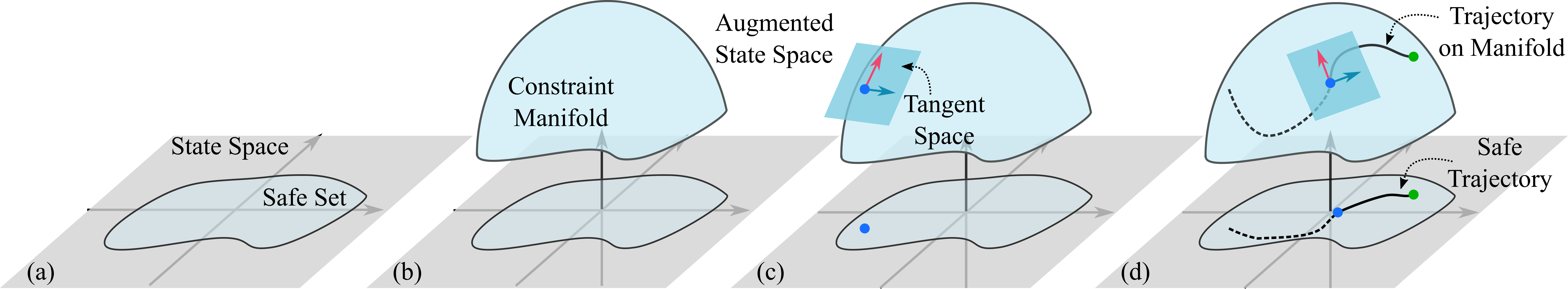}
    \caption{Conceptual illustration of \gls{atacom}. (a) The safe set is defined by the constraint in the original state space. (b) Construct the constraint manifold in the augmented state space. (c) Determine the tangent space for each point on the constraint manifold. (d) The trajectory moves on the tangent space resulting in a safe trajectory projected in the original state space.}
    \label{fig:atacom_illustration}
\end{figure*}

\section{Methods}
\label{sec:atacom}
\glsreset{atacom}
In this section, we introduce the technical details of our proposed approach \reviseRone{\gls{atacom}} and discuss under which conditions the designed controller is safe. Briefly, our approach tries to construct a safe action space tangent to the \textit{Constraint Manifold}. Combining the basis of the tangent space with a coordinate determined by the controller, that can be hand-crafted or an arbitrary \gls{rl} method, we can guarantee the safety of the overall control action. The conceptual illustration of the approach is shown in Fig.~\ref{fig:atacom_illustration}. Our approach is particularly suitable for \gls{rl} as the actions are sampled randomly from a safe space, and the policy is safe by construction. 

As it is fundamental for our approach, we will first rigorously define the \textit{Constraint Manifold} and the \textit{Tangent Space} at each point on the manifold. Then, we introduce an augmented state dynamic model that incorporates the dynamics of the slack variable. Next, we derive a safe controller that moves on the constraint manifold equipped with augmented dynamics. We then theoretically analyze the conditions of the controller to guarantee safety. \reviseRone{At the end of this section, we extend our safety proof from a Lipschitz continuous system to switched systems, as actions drawn from a stochastic policy in the RL settings break the assumption of the Lipschitz continuity.}

\subsection{Definition of the Constraint Manifold}
\label{sec:constraint_manifold}


We assume that our safety problem is completely defined by a set of $K$ inequality constraints
\begin{equation}
    \Constr(\StateVar) \leq \vzero
    \label{eq:inequality_constraint}
\end{equation}
with $\Constr: \StateSpace \subset \RR^\DimState \rightarrow \RR^\DimConstr$. Throughout this paper, we make the following assumption
\begin{assumption}
    \label{asm:diff_constr}
    Constraint function $\Constr$ is of class $C^1$. 
\end{assumption} %
While this assumption is a limitation of the proposed approach, it is easy to approximate the vast majority of constraints required by a robotic system using differentiable approximations, such as neural networks.

Given the constraint definition, we define the safe set as the sublevel set $\SafeSet \coloneqq \{\StateVar \in \StateSpace \subset \RR^\DimState: \Constr(\StateVar) \leq \vzero \}$. 
We proceed now by introducing a set of slack variables $\SlackVar \in [0, +\infty)^\DimConstr$ and rewriting the inequality constraints as  equality ones:
\begin{equation}
    \ConstrEq(\StateVar, \SlackVar) \coloneqq \Constr(\StateVar) + \SlackVar
    \label{eq:equality_slack}
\end{equation}
We refer to the Jacobian of $\ConstrEq$ as $\JacConstrEq(\StateVar, \SlackVar) = \begin{bmatrix} J_k(\StateVar) & \mathbb{I}_\DimConstr \end{bmatrix}$, with a $\DimConstr$ dimensional identity matrix $\mathbb{I}_\DimConstr$. Using the equality constraint formulation, we define the safe set as
\begin{equation}
    \ConstrManifold \coloneqq \left\{ (\StateVar, \SlackVar) \in \AugmentedSpace: \ConstrEq(\StateVar, \SlackVar) = \vzero \right\}
    \label{eq:constraint_manifold}
\end{equation}
where $\AugmentedSpace \coloneqq \StateSpace \times [0, +\infty)^\DimConstr \subset \RR^\DimAugmentedState$ is an \textit{Augmented State Space} of dimension $\DimAugmentedState = \DimState + \DimConstr$. 
\begin{remark}
The safe set $\SafeSet$ is a projection of the set $\ConstrManifold$ on the original state space. For any point $(\StateVar, \SlackVar) \in \ConstrManifold$, the projected point is in the safe set $\StateVar \in \SafeSet$.
\label{remark:safeset_projection}
\end{remark}

\begin{proposition}
The safe set $\ConstrManifold$~\eqref{eq:constraint_manifold} is a $S$-dimensional submanifold embedded in $\RR^\DimAugmentedState$.
\end{proposition}
\begin{proof}
    The Jacobian of $\ConstrEq(\StateVar, \SlackVar)$ is
    $\JacConstrEq(\StateVar, \SlackVar) = \begin{bmatrix} \JacConstr(\StateVar) & \mathbb{I}_\DimConstr \end{bmatrix}$. $\JacConstrEq(\StateVar, \SlackVar)$ has a constant rank $\DimConstr$. The Constant Rank Level-Set Theorem shows that the level set $\ConstrManifold$ is an embedded submanifold of codimension $K$. The dimension of the submanifold is $\DimAugmentedState - \DimConstr = \DimState$.
\end{proof}

We call the safe set defined in Eq.~\eqref{eq:constraint_manifold} the \textit{Constraint Manifold}.
The \textit{Interior of the Constraint Manifold} is $\Int\ConstrManifold \coloneqq \{ (\StateVar, \SlackVar) \in \ConstrManifold: \Constr_i(\StateVar) < 0, \forall i\}$ and the \textit{Boundary} is $\Bound\ConstrManifold \coloneqq \{ (\StateVar, \SlackVar) \in \ConstrManifold: \Constr_i(\StateVar) = 0, \exists i \}$. 
The \textit{Tangent Space} of the constraint manifold at the point $(\StateVar, \SlackVar) \in \ConstrManifold$ is a linear subspace $ \TanS_{(\StateN, \Slack)}\ConstrManifold \coloneqq \left\{ \vv \in \RR^{\DimAugmentedState} : \langle \JacConstrEq(\StateVar, \SlackVar), \vv \rangle = 0 \right\} = \mathrm{ker}\JacConstrEq(\StateVar, \SlackVar)$. The dimension of the tangent space is $S = N - K$. We can construct the \textit{Basis} of the tangent space $\TanS_{(\StateN, \Slack)}{\ConstrManifold}$ in matricial form as $B(\StateVar, \SlackVar) \coloneqq \left[ \BasisVec_{(\StateN, \Slack)}^1 \; \cdots \; \BasisVec_{(\StateN, \Slack)}^{S}\right] \in \RR^{N \times S}$. Each column  $\BasisVec_{(\StateN, \Slack)}^i$ is a basis vector of the tangent space satisfying $\langle \JacConstrEq(\StateVar, \SlackVar), \vb_{(\StateN, \Slack)}^{i} \rangle = 0$.

\begin{example}
Consider the inequality constraint \mbox{$k(s) \coloneqq s \leq 0$}. The constraint manifold $\ConstrManifold = \{(\StateN,\Slack)\in \RR \times [0, \infty): c(\StateN,\Slack)\coloneqq s+\Slack =0\}$ is a 1-dimensional submanifold embedded in $\RR^2$. 
The constraint Jacobian is $\JacConstrEq(\StateN, \Slack) = [1 \quad 1]$. We can choose the tangent space basis as $B(\StateN, \Slack) = \frac{1}{\sqrt{2}}[1, -1] ^\intercal$. 
\label{exp:simple_linear}
\end{example}

\subsection{Augmented State Dynamics}
We introduce a controlled system for the slack variable where the dynamics for each dimension $i$ is
\begin{equation}
    \dot{\Slack}_i = \DynSlack_i(\Slack_i) u_{\Slack, i}, \quad i\in \{1, \dots, K\} 
    \label{eq:slack_dynamics}
\end{equation}
where $\DynSlack_i:[0, \mu_\eta)\rightarrow [0, +\infty)$ is a class $\mathcal{K}$ function\footnote{class $\mathcal{K}$ function: (1) continuous; (2)strictly increasing; (3)$\alpha(0)=0$} being locally Lipschitz continuous. $\mu_\eta$ is a sufficient big value bounding the maximum constraint violation. $\ControlSlack \in [\ControlSlack^{-}, \ControlSlack^{+}]\subset \RR^{\DimConstr}$ denotes the vitual control input and $\ControlSlack^{-} < \vzero < \ControlSlack^{+}$. Since the dynamics of each dimension are independent, we develop the following lemma for the 1-dimensional system without loss of generality.

\begin{lemma}
Consider the dynamics $\dot{\Slack}= \DynSlack(\Slack){u_\Slack}$ with $\alpha$ being of class $\mathcal{K}$ and locally Lipschitz for all $\mu \in [0, \mu_\eta)$, \reviseRone{where $\mu_\eta < \infty$ is a sufficient big value that upper bounds $\mu$}. For every initial state $\Slack(0)>0$, there exists $\epsilon > 0$ such that $\Slack(t) \geq \epsilon$, $\forall t \geq 0$, $\forall u_\Slack \in [u^{-}_\Slack, u^{+}_\Slack],\, u_\Slack^{-} < 0 < u_\Slack^{+}$.
\label{lemma:slack_nondegenerate}
\end{lemma}

\begin{proof}
$\DynSlack$ is $L$-Lipschitz continuous and $\DynSlack(0) = 0$, we have $|\DynSlack(\Slack)| \leq L \Slack$ for all $\Slack \in [0, \mu_\eta)$,  For all $u_\Slack \in [u^{-}_\Slack, u^{+}_\Slack]$, the following inequality holds 
\begin{align*}
\dot{\Slack} &=\DynSlack(\Slack)u_\Slack \geq \inf_{u_\Slack}[ \DynSlack(\Slack)u_\Slack ] &  \\
&= \DynSlack(\Slack) u_\Slack^{-} & (\alpha(\mu) > 0)\\
& \geq L u_\Slack^{-} \Slack = L'\Slack & (u_\Slack^{-} < 0)
\end{align*}
where $L'=Lu_\Slack^{-}$. The lower bound of the trajectory starting from the state $\Slack(0)>0$ can be determined as
$\Slack(t) \geq \epsilon = \Slack(0) e^{(L't)} > 0, \forall t \geq 0$. 
\end{proof}

Lemma~\ref{lemma:slack_nondegenerate} shows that the slack variable with the dynamics basis of the \eqref{eq:slack_dynamics} will not reach zero in finite time if the initial value is not zero.  

We construct the \textit{Augmented Dynamic Model} as follows
\begin{align}
    \begin{bmatrix} \dot{\StateVar} \\ \dot{\SlackVar} \end{bmatrix} &= 
    \begin{bmatrix} \Dynf(\StateVar) \\ 0 \end{bmatrix} + 
    \begin{bmatrix} \DynG(\StateVar) & 0 \\ 0 & \DynSlackFun(\SlackVar) \end{bmatrix} \begin{bmatrix}
    \ControlInput \\ \ControlSlack
    \end{bmatrix} 
    \label{eq:augmented_systems}
\end{align}
where $\DynSlackFun(\SlackVar): \RR^\DimConstr \rightarrow \RR^{\DimConstr \times \DimConstr} $ denotes a diagonal matrix with entry $\DynSlackFun_{ii} = \DynSlack_i(\Slack_i)$ and $\ControlSlack = [u_{\Slack, 1}, \dots, u_{\Slack, \DimConstr}]^\intercal$. 

\subsection{Safe Controller on the Tangent Space of the Constraint Manifold}
\label{sec:atacom_basic}

In this section, we introduce a provably safe controller for the system~\eqref{eq:augmented_systems} with respect to the constraint manifold~\eqref{eq:constraint_manifold}. As stated in Remark~\ref{remark:safeset_projection}, the safe set $\SafeSet(\Constr)$ is a projection of $\ConstrManifold$. The original state will stay inside the safe set if the augmented state stays on the constraint manifold. We can design a controller that drives the augmented state in the direction tangent to the constraint manifold. The velocity of the augmented state should stay in the tangent space $\begin{bmatrix}\dot{\StateVar} \quad \dot{\SlackVar}\end{bmatrix}^{\intercal} \in \TanS_{(\StateN, \Slack)}\ConstrManifold $. 
To ensure this property, the following equality must hold
\begin{equation}
\dot{\ConstrEq}(\StateVar, \SlackVar) = \JacConstrEq(\StateVar, \SlackVar) \begin{bmatrix}\dot{\StateVar} \\ \dot{\SlackVar}\end{bmatrix} =\vzero 
\label{eq:tangent_velocity}
\end{equation}
with $\JacConstrEq(\StateVar, \SlackVar) = \begin{bmatrix} \JacConstr(\StateVar) & \mathbb{I}_{\DimConstr} \end{bmatrix}$. 
Substituting the dynamics~\eqref{eq:augmented_systems} into equality~\eqref{eq:tangent_velocity}, we obtain
\begin{equation}
    \DriftFun(\StateVar) + J_u(\StateVar, \SlackVar) \begin{bmatrix} \vu_\StateN \\ \vu_\Slack \end{bmatrix} = \vzero
    \label{eq:system_requirement}
\end{equation}
with $\JacInput(\StateVar, \SlackVar) = \begin{bmatrix} \JacInputG(\StateVar) & \DynSlackFun(\SlackVar) \end{bmatrix}$, $\JacInputG(\StateVar)=\JacConstr(\StateVar)\DynG(\StateVar)$ and the \textit{Constraint Drift} $\DriftFun(\StateVar) = \JacConstr(\StateVar) \Dynf(\StateVar)$ induced by the system drift $\Dynf(\StateVar)$. 

\begin{algorithm}[b]
\caption{Safe Reinforcement Learning with ATACOM}\label{alg:rl_atacom}
\begin{algorithmic}[1]
\State \reviseRtwo{\textbf{Initialize}: policy $\pi$, replay buffer $\mathcal{D}$, number of steps $N$} 
\For {\reviseRtwo{$k \leftarrow 1$ to $N$} }
\State \reviseRtwo{Sample action $\vu \sim \pi(\cdot|\StateVar)$.}
\State \reviseRtwo{Obtain safe control $\ControlInput \leftarrow \text{ATACOM}(\vs, \vu)$. \Comment{Alg.~\ref{alg:atacom}} }
\State \reviseRtwo{Execute $\ControlInput$ and obtain $\StateVar'$ and reward $r$.}
\State \reviseRtwo{Update replay buffer $\mathcal{D}\leftarrow (\StateVar, \vu, \StateVar', r)$.}
\State \reviseRtwo{Update RL agent using the replay buffer $\mathcal{D}$.}
\EndFor
\end{algorithmic}
\end{algorithm}

If the above linear system~\eqref{eq:system_requirement} is solvable, the general solution has the following form
\begin{equation*}
    \begin{bmatrix} \ControlInput \\ \ControlSlack \end{bmatrix} = -\JacInput^\dagger(\StateVar, \SlackVar) \psi(\StateVar) + \BasisTangentInput(\StateVar, \SlackVar) u(\StateVar), 
\end{equation*}
where $\JacInput^{\dagger}(\StateVar, \SlackVar)$ is the pseudo inverse of $\JacInput(\StateVar, \SlackVar)$ and $\BasisTangentInput(\StateVar, \SlackVar)$ is the tangent space basis in the matrix form such that $\JacInput(\StateVar, \SlackVar) \BasisTangentInput(\StateVar, \SlackVar) = \vzero$. 
The function $u(\StateVar):\StateSpace \rightarrow \RR^{\DimControl}$ is a task-specific feedback controller. In the rest of the paper, we use bold characters and omit the arguments of the mapping to simplify the notation. For example, we define $\JacInputMat$ and $\BasisTangentInputMat$ as $\JacInputMat \coloneqq \JacInput(\StateVar, \SlackVar)$ and $\BasisTangentInputMat \coloneqq \BasisTangentInput(\StateVar, \SlackVar)$, respectively.

We can construct the safe controller as
\begin{equation}
    \begin{bmatrix} \ControlInput \\ \ControlSlack \end{bmatrix} = -\JacInputMat^\dagger \DriftVec -  \lambda \JacInputMat^\dagger \ConstrEqVec + \BasisTangentInputMat \vu, 
    \label{eq:atacom_controller}
\end{equation} 
with a constant $\lambda > 0$. The first term on the Right Hand Side (RHS) is the \textit{Drift Compensation Term}, which corrects the drift caused by the system. The middle term is the \textit{Contraction Term} that retracts the state to the manifold. When $\Constr(\StateVar)<0$, there always exists $\SlackVar \in [0, \infty)^\DimConstr$ ensuring $\ConstrEq(\StateVar, \SlackVar)=\vzero$, so the contraction term will be zero. The last term is the \textit{Tangential Term}, which generates the vector field tangent to the Constraint Manifold. 
The \gls{atacom} controller can be treated as a generalized solution of the linear system
\begin{equation*}
    \DriftVec + \lambda \ConstrEqVec + \JacInputMat \begin{bmatrix} \ControlInput \\ \ControlSlack \end{bmatrix} = \vzero
\end{equation*}

The tangent space basis $\BasisTangentInputMat$ can be solved by various approaches, such as QR/SVD decomposition. We will discuss computing a smooth varying basis in Section~\ref{sec:tangent_space_basis}. 
\reviseRtwo{Alg.~\ref{alg:rl_atacom} and Alg.~\ref{alg:atacom} illustrate how to integrate \gls{atacom} into the \gls{rl} framework. Generally speaking, \gls{atacom} constructs a safe action space that morphs all input actions into safe actions. Therefore, the \gls{rl} agent will learn the policy directly from the safe action space instead of the original one.} 

\begin{algorithm}[t]
        \caption{\gls{atacom}}\label{alg:atacom}
        \begin{algorithmic}[1]
        \Statex \textbf{Input:} $\StateVar$, $\vu$ \Comment{At each step}
        \State \reviseRone{Determine the slack variable} 
        \Statex \reviseRone{$\SlackVar \leftarrow \max(-\Constr(\StateVar), \mathrm{tol})$}
        \State \reviseRone{Compute the Jacobians and the drift }
        \Statex \reviseRone{$\JacInputGMat\leftarrow \JacConstr(\StateVar)\DynG(\StateVar)$}
        \Statex \reviseRone{$\JacInputMat(\StateVar, \SlackVar)\leftarrow\begin{bmatrix}\JacInputGMat(\StateVar) & A(\SlackVar)\end{bmatrix}, $ }
        \Statex \reviseRone{$\DriftVec(\StateVar) \leftarrow \JacConstr(\StateVar) \Dynf(\StateVar)$}
        \State \reviseRone{Compute the tangent space basis \Comment{Alg.~\ref{alg:tangent_basis}}}
        \Statex \reviseRone{$\BasisTangentInputMat\leftarrow \text{SmoothBasis}(\JacInputMat)$ }
        \State \reviseRone{Compute the constraint value} 
        \Statex \reviseRone{$\ConstrEq(\StateVar, \SlackVar)\leftarrow k(\StateVar) + \SlackVar$}
        \State \reviseRone{Compute safe control output $\ControlInput$ \Comment{Eq.~\eqref{eq:atacom_controller}}}
        \State \reviseRone{\textbf{Output:} $\ControlInput$ }
        \end{algorithmic}
\end{algorithm}


Next, we will prove that \gls{atacom} enforces the safety constraints under mild assumptions. 

\begin{assumption}
The set $\ConstrManifold$ defined in \eqref{eq:constraint_manifold} is non-empty.
\label{asm:non_empty}
\end{assumption}

Assumption~\ref{asm:non_empty} ensures that there exists a non-empty safe set, i.e., $\SafeSet \neq \emptyset$. Then, we make another assumption to ensure the set $\ConstrManifold$ is a manifold. From the Constant-Rank Level Set Theorem, the 0-level set $\ConstrManifold$ is a submanifold if the rank of $\JacInputMat$ is full row rank. To make proper mathematical statements, we first explain some of the notations.

Let $\iota_{\Constr=0}(\StateVar)$ be a mapping that output the index set $\{i\in \mathbb{Z} : \Constr_i(\StateVar)=0\}$ where the $i$-th element of $\Constr(\StateVar)$ is 0 (and $\iota_{\Constr \neq 0}(\StateVar)$ for $\{i\in \mathbb{Z} : \Constr_i(\StateVar) \neq 0 \}$, respectively). Let $(\cdot)_{[\iota, :]}$ be a submatrix of the matrix $(\cdot)$ where all rows of the index set $\iota$ are selected. The size of the index set is denoted as $|\iota|$.
\begin{assumption}
The rank of $(\JacInputGMat)_{[\iota_{k=0}, :]}$ equals to $|\iota_{k=0}(\StateVar)|$, for all $(\StateVar, \SlackVar) \in \partial\ConstrManifold$.
\label{asm:constant_rank}
\end{assumption}
Assumption~\ref{asm:constant_rank} ensures $\mathrm{rank}(\JacInputMat)$ is $K$ anywhere in $\ConstrManifold$. From the definition, we know that $\JacInputMat = [\JacInputGMat \quad \DynSlackMat]$ and $\DynSlackMat$ is a diagonal matrix.  For the interior points $(\StateVar, \SlackVar) \in \Int \ConstrManifold$ where $k(\StateVar)<\vzero$ and $\SlackVar>\vzero$, we know $A_{ii} > 0, \forall i$. Therefore, $\mathrm{rank}(\JacInputMat) = \mathrm{rank}(\DynSlackMat) = K$. On the boundary $\partial \ConstrManifold $ where $\Constr_i(\StateVar)=\Slack_i=0$, we have $\DynSlack_i(\Slack_i) = 0$ and $\mathrm{rank}(\DynSlackMat) = \mathrm{rank}\left((\DynSlackMat)_{[\iota_{k\neq 0}, :]}\right) = \DimConstr - |\iota_{\Constr \neq 0}| $. We know that the row vectors in $ (\JacInputMat)_{[\iota_{\Constr =0}, :]}$ are linearly independent from the row vectors in $(\JacInputMat)_{[\iota_{\Constr \neq 0}, :]}$. From Assumption~\ref{asm:constant_rank}, we have $\mathrm{rank}(\JacInputMat)=\mathrm{rank}\left((\JacInputMat)_{[\iota_{\Constr \neq 0}, :]}\right) + \mathrm{rank}\left((\JacInputMat)_{[\iota_{\Constr = 0}, :]}\right) = \mathrm{rank}\left((\DynSlackMat)_{[\iota_{\Constr \neq 0}, :]}\right) + \mathrm{rank}\left((\JacInputGMat)_{[\iota_{\Constr = 0}, :]}\right) = \DimConstr$. 

\begin{remark}
Assumption~\ref{asm:constant_rank} indicates that the number of constraints reaching their boundary is not bigger than the dimension of control input, i.e., $|\iota_{\Constr = 0}(\StateVar)| \leq \DimControl$.
\end{remark}
Since the matrix $(\JacInputGMat)_{[\iota_{k=0}, :]}$ is of dimension ${|\iota_{k=0}| \times \DimControl}$, we have the following relation $\mathrm{rank}\left((\JacInputGMat)_{[\iota_{k=0}, :]}\right) = |\iota_{k=0}(\StateVar)| \leq \min(|\iota_{k=0}(\StateVar)|, \DimControl) $. Therefore, a necessary condition is $|\iota_{\Constr = 0}(\StateVar)| \leq \DimControl$. 

\reviseRthree{
\begin{remark}
    \label{remark:existence}
    Assumption~\ref{asm:constant_rank} is a necessary condition to ensure that \eqref{eq:constraint_manifold} forms a manifold, as dictated by the constant-rank level set theorem. It is also a sufficient condition to guarantee the existence of a solution to the linear system \eqref{eq:system_requirement}. However, when considering the input constraint $\ControlInput \in \ControlSpace$, this assumption alone is not sufficient to ensure the existence of a solution.
\end{remark}}

To present the safety theorem and proof, we introduce the \textsl{Singular Set}, used in defining the (safe) region of contraction.
\begin{definition}
The \textsl{Singular Set} is  defined as
\begin{equation*}
    \SingularSet \coloneqq \{(\StateVar, \SlackVar) \in \AugmentedSpace: \JacInputMat^{\intercal} \ConstrEqVec = \vzero, \Vert \SlackVar \Vert = 0, \Vert \ConstrEq(\StateVar, \SlackVar) \Vert \neq 0 \}
\end{equation*}
\end{definition}

\begin{remark}
Notice that if $\SingularSet \neq \emptyset$, the Singular Set $\SingularSet$ and the Constraint Manifold $\ConstrManifold$ are disjoint sets, $\mathrm{dist}(\ConstrManifold, \SingularSet)>0$.
\end{remark}
\begin{remark}
We know that the kernel of a matrix $\mX$ has the following property: $\mathrm{ker}(\mX^\dagger) = \mathrm{ker}(\mX^\intercal)$. Therefore, the Singularity Set is equivalent to $\{\StateVar \in \StateSpace: \JacInputGMat^{\intercal} \Constr(\StateVar) = \vzero, \Vert \Constr(\StateVar) \Vert \neq 0\} \times \{\SlackVar: \SlackVar = \vzero \}$ and $\{\StateVar \in \StateSpace: \JacInputGMat^{\dagger} \Constr(\StateVar) = \vzero, \Vert \Constr(\StateVar) \Vert \neq 0\} \times \{\SlackVar: \SlackVar=\vzero \}$.  The Singular set is independent of the choice of slack dynamics.
\end{remark}

\begin{figure}[t]
\centering   
    \includegraphics[width=0.49\linewidth, trim=2.3cm 1cm 2.3cm 0cm, clip]{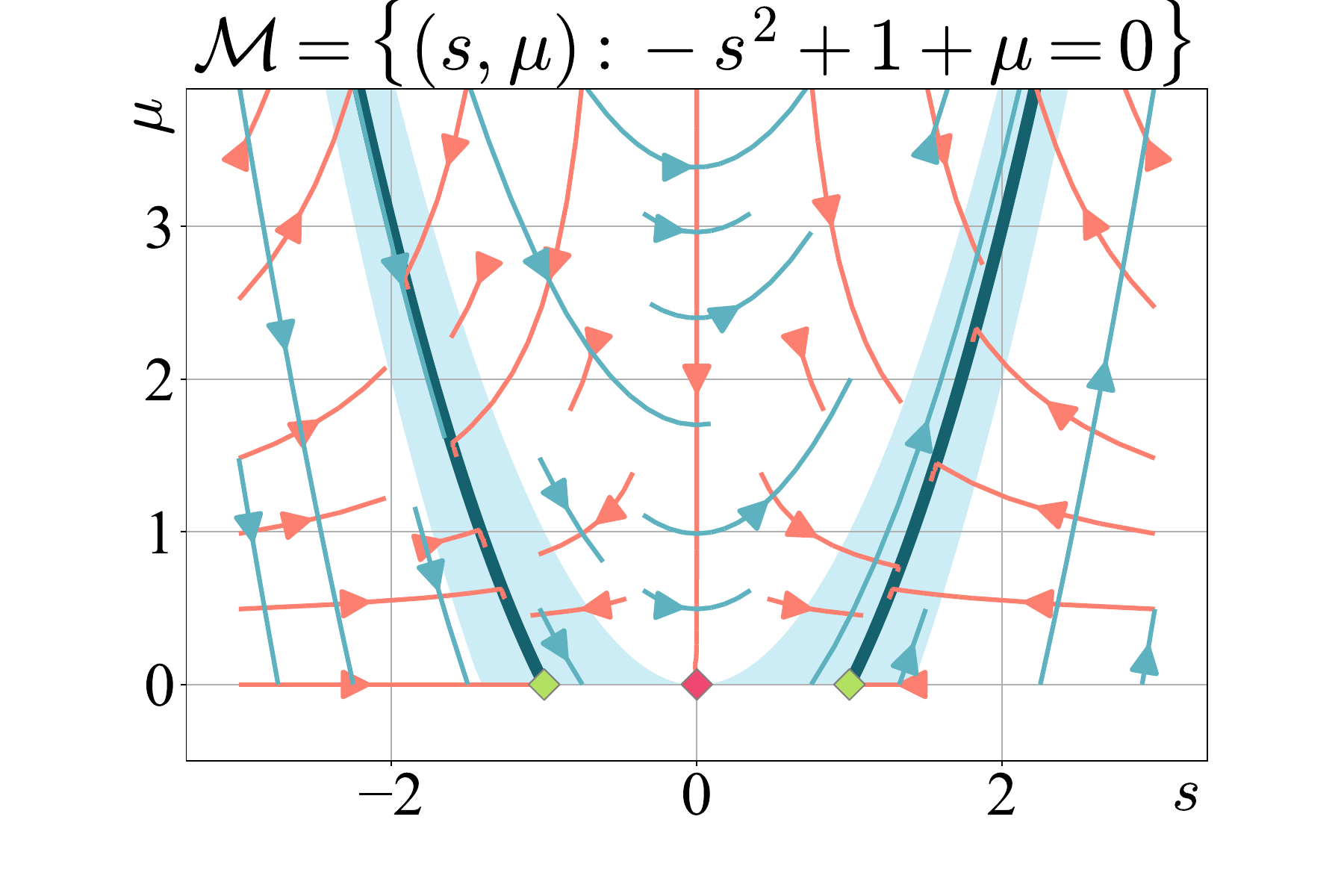}
    \includegraphics[width=0.49\linewidth, trim=2.3cm 1cm 2.3cm 0cm, clip]{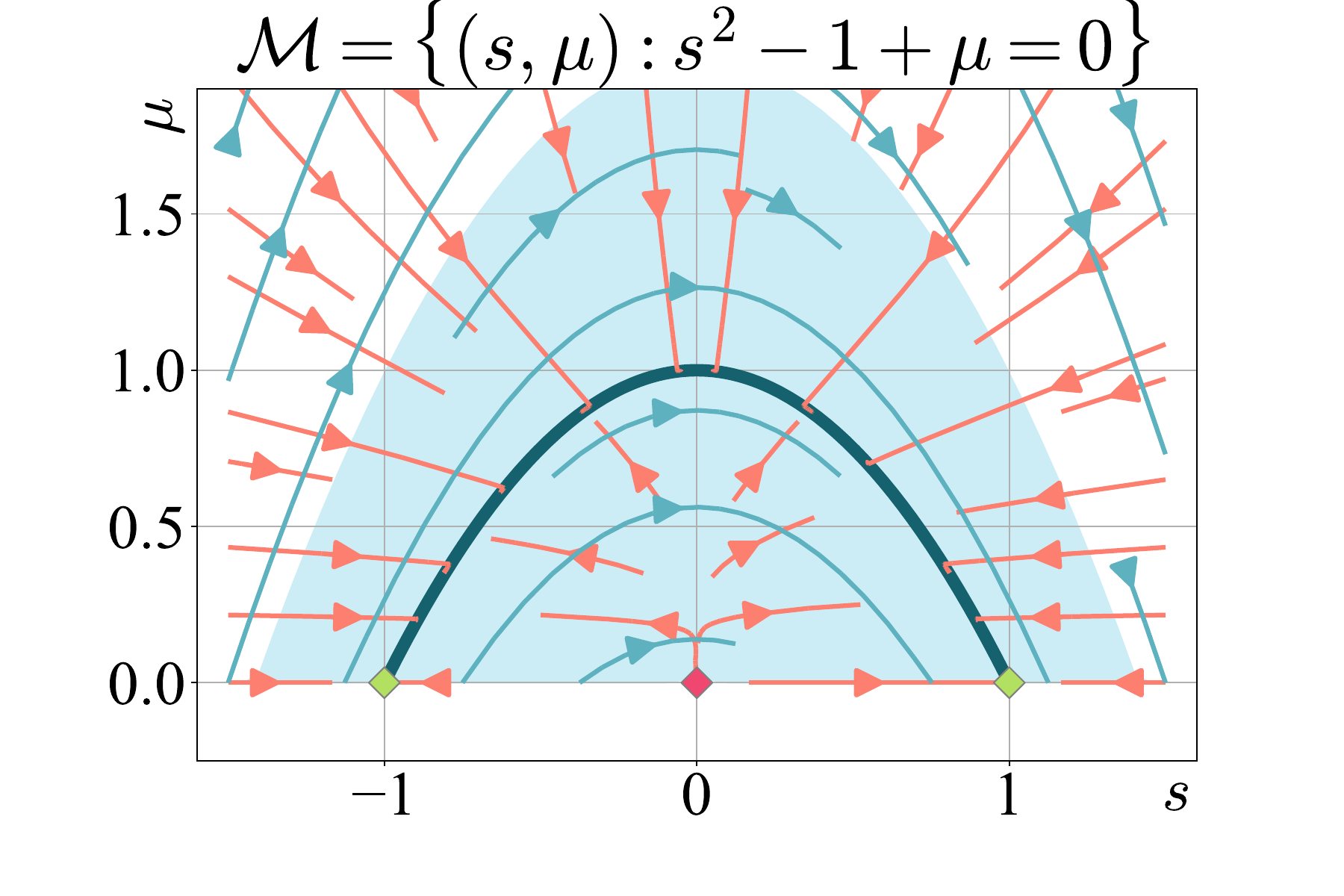}
    \includegraphics[width=0.9\linewidth, trim=1cm 4cm 0cm 2.5cm, clip]{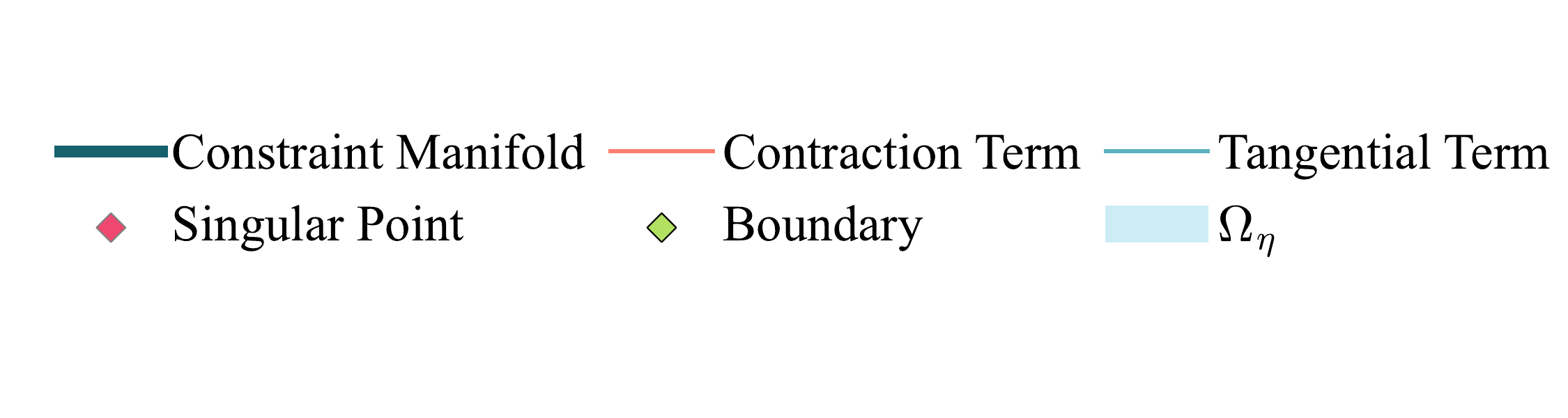}
    \caption{Illustration of the Constraint Manifolds, Singular Set, and the Region of Attraction in Example~\ref{example:singular_set}. The thick blue lines depict the constraint manifold, and the red points are singular points. The contraction term shrinks to zero at the singular point, indicating a saddle point of the Lyapunov function~$V$. The region of attraction $\Omega_\eta$ excludes the singular set, as the blue shaded area shown in the figure.}
    \label{fig:singular_illustration}
    \vspace{-1em}
\end{figure}

The Singular Set can be empty or non-empty, in the following, we provide two examples illustrating the two cases.

\begin{example}
Consider the inequality constraint $k(s):= s \leq 0$ with the first-order system $\dot{s}=u$. We have $f(s)=[0]$ and $\mG=[1]$. The Jacobian $\JacInputGMat = [1]$ is full rank. Therefore, the singular set is empty.
\end{example}

\begin{example}
Consider the simple inequality constraint $\Constr(s) = -s^2 + 1 \leq 0$ and the system $\dot{s}=u$. The constraint manifold is $\ConstrManifold=\{ (s, \Slack)\in \RR \times [0, +\infty): -s^2 + 1 + \Slack = 0 \}$. The Jacobian $\JacInputGMat = [-2s]$. At the singular point, we have $\JacInputGMat \Constr(s) = [-2s (-s^2 + 1)] = [0]$ and $k(\StateN) \neq 0$.  The singular set of this constraint manifold is $\SingularSet = \{ (0, 0) \}$. The constraint manifold and the singular point are depicted in Fig.~\ref{fig:singular_illustration}. We can notice that the contraction term (red) at the singular point is zero, indicating a saddle point of the Lyapunov function. If the tangential term (blue) is zero, the system will stay at the singular point. 
Similarly, the singular set for a different constraint $k(s)=s^2 - 1 \leq 0$ is $\SingularSet = \{ (0, 0) \}$, as shown in Fig.~\ref{fig:singular_illustration}~(right). In the example on the left, the system will approach the singular point $(0, 0)$ from the line $s=0$ with the tangential term equal to zero. In the example on the right, the system does not approach the manifold when starting from the singularity and holding the tangential component to zero.
\label{example:singular_set}
\end{example}

Next, we prove that the controller~\eqref{eq:atacom_controller} is safe (i.e., positively invariant) considering the Constraint Manifold~\eqref{eq:constraint_manifold}, following LaSalle's Invariance Principle. 
We define a Lyapunov-like function $V: \RR^{\DimAugmentedState} \rightarrow \RR$ as 
\begin{equation}
    V(\StateVar, \SlackVar) = \frac{1}{2} \ConstrEqVec^{\intercal} \ConstrEqVec
    \label{eq:lyapunov_function}
\end{equation}
Let $\eta' = \inf\{V(\StateVar, \SlackVar): (\StateVar, \SlackVar) \in \SingularSet\}$ if $\SingularSet$ is non-empty, otherwise $\eta'$ is a sufficient large value. 
From Assumption \ref{asm:compactness}, we can choose a constant $\eta$ such that $0 < \eta < \eta'$, and the set $\Omega_\eta = \{(\StateVar, \SlackVar) \in \AugmentedSpace: V(\StateVar, \SlackVar) \leq \eta \}$ is compact. $\Omega_\eta$ is a superset of the constraint manifold $\ConstrManifold \subset \Omega_\eta$. 

To prove the positive invariance of the controller, we require the following lemma
\begin{lemma}
Let $\mX \in \RR^{m\times n}$, $r = \mathrm{rank}(\mX) \leq m \leq n$, and $\vx \in \RR^{m}$. Let $\mX^{\dagger}, \mX^{\intercal}$ be the psuedoinverse and transpose of $\mX$, respectively. If $\vx^{\intercal}\mX \mX^{\dagger} \vx = 0$, then $\mX^{\intercal} \vx = \vzero$.
\label{lemma:inv_x_eq_tran_x}
\end{lemma}

\begin{proof}
Let the Singular Value Decomposition of $\mX$ be
$$ \mX = \begin{bmatrix} \mU_1 & \mU_2 \end{bmatrix} \begin{bmatrix} \mSigma_1 & \vzero \\ \vzero & \vzero \end{bmatrix} \begin{bmatrix} \mV_1^{\intercal} \\ \mV_2^{\intercal} \end{bmatrix} $$
where $\mSigma_1$ is the $r \times r$ diagonal matrix whose diagonal entries are the positive singular values of $\mX$. The pseudo-inverse is
$$ \mX^{\dagger} = \begin{bmatrix} \mV_1 & \mV_2 \end{bmatrix} \begin{bmatrix} \mSigma_1^{-1} & \vzero \\ \vzero & \vzero \end{bmatrix} \begin{bmatrix} \mU_1^{\intercal} \\ \mU_2^{\intercal} \end{bmatrix} $$
Using the above-defined decompositions, we can compute
$$ \mX \mX^{\dagger} = \begin{bmatrix} \mU_1 & \mU_2 \end{bmatrix} \begin{bmatrix} \mathbb{I}_r & \vzero \\ \vzero & \vzero \end{bmatrix} \begin{bmatrix} \mU_1^{\intercal} \\ \mU_2^{\intercal} \end{bmatrix} = \mU_1 \mU_1^{\intercal} $$
and $ \mX \mX^{\intercal} = \mU_1 \mSigma_1^2 \mU_1^{\intercal}$.
Since, by hypothesis, $\vx^{\intercal}\mX \mX^{\dagger} \vx = 0$, we have $\vx^{\intercal}\mU_1 \mU_1^{\intercal} \vx = 0$. Let $\vy=\mU_1^{\intercal} \vx$, from the previous equation, we have $\vy^{\intercal}\vy=0$. This implies $\vy=\vzero$, as the dot product of a vector by itself is zero only if the vector is the null vector.
We can now write $\vx^{\intercal}\mX \mX^{\intercal} \vx = \vx^{\intercal}\mU_1 \mSigma_1^2 \mU_1^{\intercal} \vx = \vy^{\intercal}\mSigma_1^2\vy = 0$. Therefore, we can conclude that $\mX^{\intercal} \vx = \vzero$. 
\end{proof}

Using the Lemma~\ref{lemma:inv_x_eq_tran_x}, we can prove the safety of the \gls{atacom} controller by deriving the following attraction theorem.

\begin{theorem}
Consider the nonlinear control affine system \eqref{eq:augmented_systems}. Let $\ConstrManifold$ be the constraint manifold defined in \eqref{eq:constraint_manifold}, the dynamics of the slack variable defined in \eqref{eq:slack_dynamics}. Under Assumptions~\ref{asm:compactness}-\ref{asm:constant_rank}, every trajectory, starting from $(\StateVar(0), \SlackVar(0)) \in \Omega_{\eta}$, equipped with the controller \eqref{eq:atacom_controller}, will approach $\ConstrManifold$ as $t \rightarrow +\infty$, if $\exists \vu_\StateN \in \ControlSpace$ such that \eqref{eq:system_requirement} holds for all $(\StateVar, \SlackVar) \in \Omega_\eta$.
\label{theorem:atacom_safety}
\end{theorem}

\begin{proof}
We will use LaSalle's Principle to prove that the system's controller is safe. We compute the time-derivative of the Lyapunov-like function~\eqref{eq:lyapunov_function}:
\begin{align*}
    \dot{V} &= \ConstrEqVec^{\intercal} \dot{\ConstrEqVec} 
    = \ConstrEqVec^{\intercal} \JacConstrEqMat \begin{bmatrix} \dot{\StateVar} \\ \dot{\SlackVar} \end{bmatrix} \stackrel{\eqref{eq:augmented_systems}}{=} \ConstrEqVec^{\intercal} \left[ \DriftVec + \JacInputMat \begin{bmatrix} \ControlInput \\ \ControlSlack \end{bmatrix} \right] \\
    &\stackrel{\eqref{eq:atacom_controller}}{=} \ConstrEqVec^{\intercal}  \left[ \DriftVec + \JacInputMat\left( -\JacInputMat^\dagger \DriftVec - \lambda \JacInputMat^\dagger \ConstrEqVec + \BasisTangentInputMat \vu \right)  \right] \\
    &= \ConstrEqVec^{\intercal} \left[ \DriftVec - \JacInputMat \JacInputMat^\dagger \DriftVec - \lambda \JacInputMat \JacInputMat^\dagger \ConstrEqVec + \JacInputMat \BasisTangentInputMat \vu \right] &
\end{align*}

By definition, we know that $\JacInputMat \BasisTangentInputMat = \vzero$. If Eq.~\eqref{eq:system_requirement} holds, we have $\DriftVec - \JacInputMat\JacInputMat^{\dagger} \DriftVec = \vzero$.
We can simplify the RHS and get
\begin{equation}
    \dot{V} = -\lambda \ConstrEqVec^{\intercal} \JacInputMat \JacInputMat^\dagger \ConstrEqVec \leq 0
    \label{eq:lyapunov_nsd}
\end{equation}

Due to the negative semi-definiteness of \eqref{eq:lyapunov_nsd}, the compact set $\Omega_\eta$ is positively invariant. We can find the set \mbox{$\mathcal{E} \coloneqq \{(\StateVar, \SlackVar)\in \Omega_\eta: \dot{V} = 0 \}$}. From Lemma~\ref{lemma:inv_x_eq_tran_x}, we have the equivalent set
\begin{align*}
    \mathcal{E} &= \{ (\StateVar, \SlackVar) \in \Omega_\eta: -\lambda \ConstrEqVec^{\intercal} \JacInputMat \JacInputMat^\dagger \ConstrEqVec = 0 \} \\
    &= \{ (\StateVar, \SlackVar) \in \Omega_\eta: \JacInputMat^\intercal \ConstrEqVec = \vzero \} 
\end{align*} 
Since $\Omega_\eta \cap \SingularSet = \emptyset$, we have $\mathcal{E} = \{ (\StateVar, \SlackVar) \in \Omega_\eta: c(\StateVar, \SlackVar) = \vzero \} = \ConstrManifold$. Then, we verify that $\ConstrManifold$ is positively invariant
\begin{equation*}
    \ConstrEq(\StateVar(t), \SlackVar(t)) = \ConstrEq(\StateVar(0), \SlackVar(0)) + \int_0^t \dot{\ConstrEq}(\StateVar(\tau), \SlackVar(\tau)) \mathrm{d}\tau
\end{equation*}
From a initial state where $\Constr (\StateVar(0),\SlackVar(0))=0$, we get
\begin{align*}
    \Constr(\StateVar(t), \SlackVar(t)) &= \int_0^t \DriftVec + \JacInputMat\left( -\JacInputMat^\dagger \DriftVec - \lambda \JacInputMat^\dagger \ConstrEqVec + \BasisTangentInputMat \right) \mathrm{d}\tau \\
    &= -\int_0^t \lambda \JacInputMat \JacInputMat^\dagger \ConstrEqVec \mathrm{d}\tau = \vzero
\end{align*}
Thus, for any time $t$, we have $\ConstrEq(\StateVar(t), \SlackVar(t))=0$ and $(\StateVar(t), \SlackVar(t))\in \ConstrManifold$. 
Therefore, $\ConstrManifold$ is the largest positively invariant set in $\mathcal{E}$.
\end{proof}

Theorem~\ref{theorem:atacom_safety} provides necessary conditions to guarantee the convergence to the constraint manifold $\ConstrManifold$ in the neighborhood $\Omega_\eta$, i.e., region of contraction. 

\begin{example}
Back to the constraints defined in Example~\ref{example:singular_set}, the system may get stuck at the singular point (red diamond) when the tangential component is zero. We can, therefore, construct the region of contraction excluding the singular point. The Lyapunov function at the singular point in both cases is $V(0, 0) = 1$. The region of contraction $\Omega_\eta$ can be determined by $0 <\eta < 1$, i.e. the blue-shaded area. 
\end{example} 

\subsection{Safety with Stochastic Policy in Reinforcement Learning}

Theorem \ref{theorem:atacom_safety} has shown the safety for a Lipschitz continuous controller. However, in reinforcement learning, the control inputs are drawn from a stochastic policy $\pi$ at each time step. The control input is no longer Lipschitz continuous. To deal with the reinforcement learning setting, we treat the system with stochastic policy as a switched system. The system switches to a new one at each time step when  a new action is drawn from the stochastic policy.
Many prior works have studied Lasalle's Invariance principle for switched systems~\cite{mancilla-aguilar_extension_2006, zhang_weak-invariance_2014, bacciotti_invariance_2005}. We present the most relevant one here: 

\begin{theorem}{\cite{bacciotti_invariance_2005}}
    Let $V(x):\RR^\DimAugmentedState \rightarrow [0, +\infty)$ be a weak common Lyapunov function for switched systems $\mathcal{F}=\{ f_p(x), p\in P \}$, where $P={1, ..., N}$ and $f_p(x)$ is continous. Let 
    $$\mathcal{E}=\left\{x \in \Omega_\eta: \exists p \in P \text{ such that } \nabla V(x) \cdot f_p(x) = 0\right\}.$$
    Let $\ConstrManifold$ be the union of all compact weakly invariant sets contained in $\mathcal{E}\cap \Omega_\eta$. Every solution $\varphi(t)$ has a nonvanishing dwell time such that $\varphi(0) \in \Omega_\eta$ is attracted by $\ConstrManifold$.
    \label{theorem:lasalle_switched_system}
\end{theorem}

The solution $\varphi(t)$ has a \textit{nonvanishing dwell time} if the sequence $\{t_j\}$ of switching times satisfies $\inf_j (t_{j+1} - t_j) \geq h > 0$. For stochastic controller/policy, the control input is sampled from a distribution $\vu_j \sim \pi(\cdot|\vs_j)$ at time $t_j$. The control action is kept the same for the time interval $[t_j, t_{j+1})$. 

Using Theorem~\ref{theorem:lasalle_switched_system}, we can prove the following proposition
\begin{proposition}
    Let $g_j(\StateVar, \SlackVar)$ be the switched system defined by \eqref{eq:augmented_systems} and \eqref{eq:atacom_controller} determined by control input $\vu_j$, where $\vu_j$ is a control inputs for the time interval $[t_j, t_{j+1}), j=1, ..., N$. Let $V$ be the common Lyapunov function defined in \eqref{eq:lyapunov_function}. Every trajectory, starting from $(\vs(0), \vmu(0))\in \Omega_\eta$ is attracted by $\ConstrManifold$.
\end{proposition}

\begin{proof}
    From the Theorem~\ref{theorem:atacom_safety}, we can verify that each system $g_j$ is attracted by $\ConstrManifold_j = \ConstrManifold$. The invariant set of the switched system is $\cup_{j=1}^{N} \ConstrManifold_j = \ConstrManifold$.
\end{proof}

\subsection{\reviseRtwo{Input-to-State Stability with disturbances}\label{sec:atacom_with_disturbances}}
\reviseRtwo{
In the previous analysis, we assume full knowledge of the dynamics is available. In practice, a perfect dynamic model is not available due to model mismatch and disturbances. In this section, we address this issue by considering the dynamic model with external disturbances as
\begin{equation}
    \StateVar = \Dynf(\StateVar) + \DynG(\StateVar)\ControlInput + \vepsilon 
    \label{eq:dynamics_with_disturbance}
\end{equation}
where $\Vert \vepsilon\Vert < \omega$ is a bounded random variable that captures the model mismatch and external disturbances. In the following theorem, we prove the dynamic system with disturbances~\eqref{eq:dynamics_with_disturbance} is \gls{iss}~\cite{angelilasalle}.
}

\reviseRtwo{
\begin{theorem}
Consider the dynamic system~\eqref{eq:dynamics_with_disturbance} and the safety constraint~\eqref{eq:inequality_constraint}. Let $\MM$ be the constraint manifold defined in~\eqref{eq:constraint_manifold} and the dynamics of the slack variable defined in~\eqref{eq:slack_dynamics}. The system with controller~\eqref{eq:atacom_controller} is Input-to-State Stable.  Every trajectory, starting from $\Omega_{\eta}$, will approach the neighborhood of constraint manifold $\{ (\StateVar, \SlackVar) | \Vert \ConstrEqVec\Vert = \Vert \Constr(\StateVar) + \SlackVar \Vert \leq \eta_c, \sqrt{2\eta} > \eta_c > 0 \}$, if $\Vert \JacConstrMat \Vert \leq \eta_J$, $\lambda \geq \frac{ \eta_J}{\eta_c}\omega$, and \eqref{eq:system_requirement} holds.
\end{theorem}
\begin{proof}
    Consider the Lyapunov-like function $V = \frac{1}{2}\ConstrEqVec^\intercal \ConstrEqVec$. Following a similar derivation from \eqref{eq:lyapunov_nsd}, we have
    $ \dot{V} = -\lambda \ConstrEqVec^{\intercal} \JacInputMat \JacInputMat^{\dagger} \ConstrEqVec + \ConstrEqVec^{\intercal} \JacConstrMat \vepsilon$. 
    From Assumption~\ref{asm:constant_rank}, $\JacInputMat$ is full row rank and $\JacInputMat \JacInputMat^{\dagger} = \mathbb{I}$. We have 
    \begin{align*}
        \dot{V} &\leq -\lambda \Vert \ConstrEqVec \Vert ^2 + \Vert \ConstrEqVec \Vert \Vert \JacInputMat \Vert  \Vert \vvepsilon \Vert  
    \end{align*}
    The system is \gls{iss} following~\cite{angelilasalle}~[Theorem 1]. When $\sqrt{2\eta} > \Vert \ConstrEqVec \Vert \geq \eta_c$ 
    \begin{align*}
        \dot{V} &\leq - \frac{\omega \eta_J}{\eta_c} \Vert \ConstrEqVec \Vert ^2 + \omega \eta_J \Vert \ConstrEqVec \Vert < 0
    \end{align*}
    The system starting from $(\StateVar, \SlackVar) \in \Omega_{\eta}$ will approach the neighborhood of the manifold $\{(\StateVar, \SlackVar) | \Vert \Constr(\StateVar) + \SlackVar \Vert  \leq \eta_c\}$. Given that $\Vert \ConstrVec \Vert \leq \Vert \ConstrEqVec \Vert$. The maximum constraint violation will also be bounded by $\eta$.
\end{proof}
}

\section{Practical Implementation}
\label{sec:practical_implementation}
In this section, we will introduce several practical techniques for implementing the \gls{atacom} method. We compared several different types of dynamics of the slack variable $\DynSlackMat$, introduced a method to obtain the continuously varying tangent space basis, and a drift clipping technique that leverages the system drift only when necessary to ensure safety. 

\subsection{Slack Variable Dynamics}
\label{sec:slack_dynamics_functions}
\begin{figure}[t]
    \centering
    \begin{minipage}{0.5\linewidth}
    \centering
    \includegraphics[width=\linewidth, trim=2cm 2cm 1cm 1cm, clip]{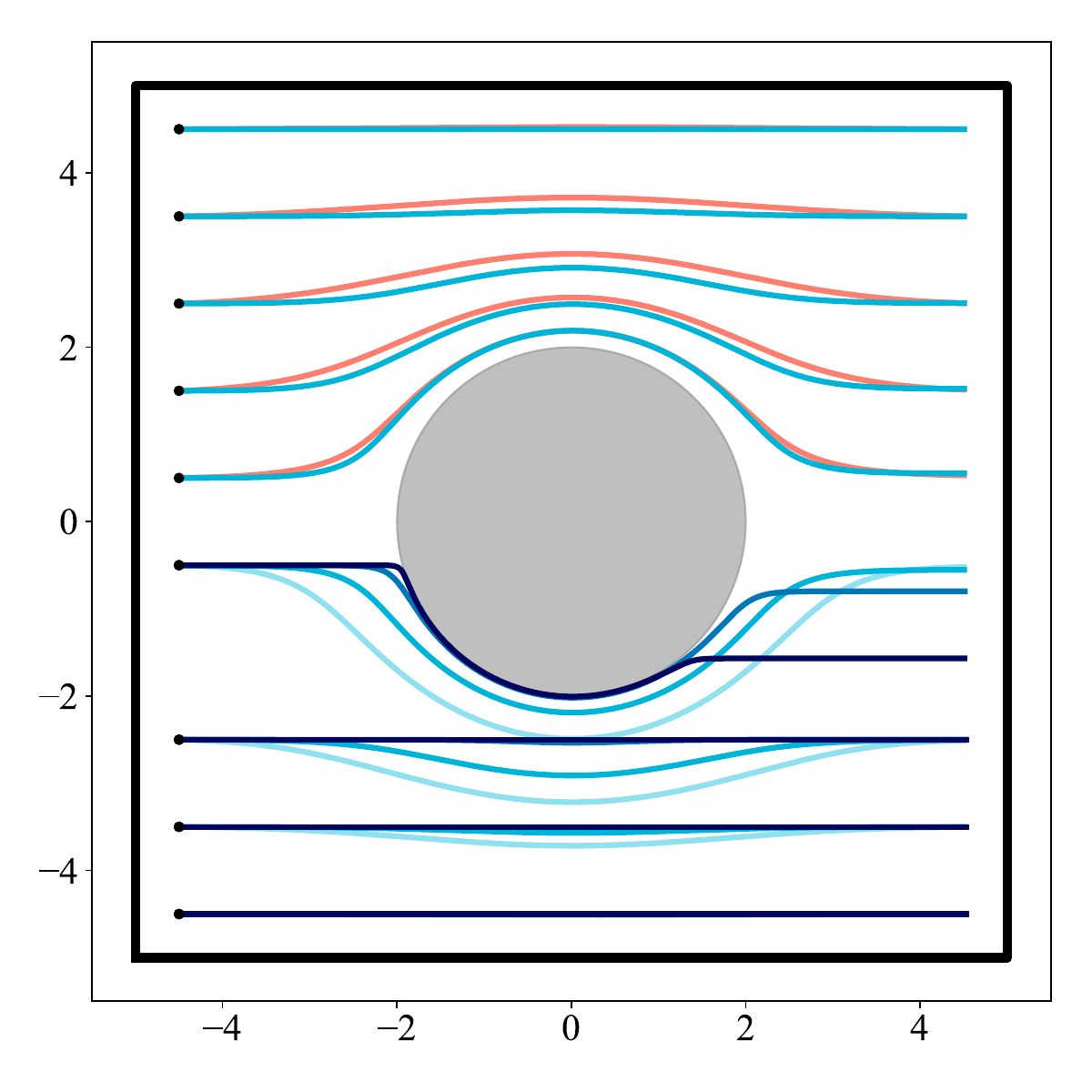}
    \end{minipage}
    \begin{minipage}{0.28\linewidth}
    \centering
        \includegraphics[width=\linewidth, trim=1.5cm 1cm 0.5cm 0.5cm, clip]{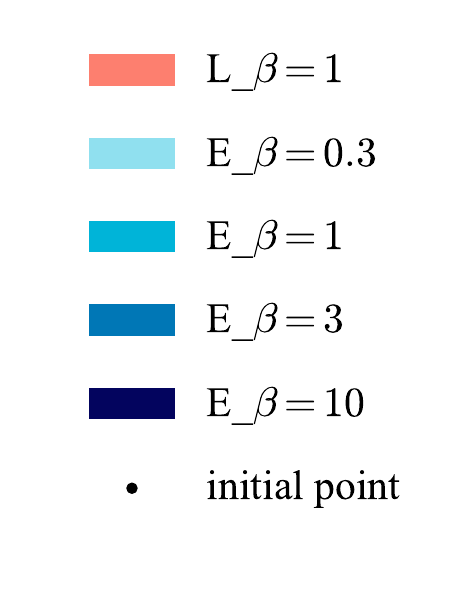}
    \end{minipage}
    
    \caption{Comparison of trajectories with different slack function $\alpha(\SlackVar)$ in a 2D Environment. 
    The grey area renders an obstacle in a 2D environment. The constraint is defined as $\Vert \StateVar - \vp_{o} \Vert> 0$. 
    The system is controlled by velocity $\dot{\StateVar} = \ControlInput$. The curves show the trajectories starting from different initial points with a constant control input $\ControlInput = [1 \quad 0]^\intercal$. The upper half shows the trajectories with the exponential slack dynamics (E) and linear ones (L). The lower half shows trajectory with different $\beta$ parameters using exponential slack dynamics.}
    \label{fig:comp_slack_type}
\end{figure}

Different types of dynamics of the slack variable $\DynSlackMat$ will affect how the tangent space is deformed. We use a simple example to understand how the tangent space is deformed.

\begin{example}
    Consider a simple constraint $k(s)\coloneqq s \leq 0$ whose Jacobian $\JacConstr = 1$. Let the dynamics be defined as $\dot{s} = u_s$. We get the $\JacInputMat = \begin{bmatrix} 1 & \alpha(\mu) \end{bmatrix}$. 
    We can compute the basis of the tangent space as $\BasisTangentInputMat = \left[ -\alpha(\mu)\quad 1 \right]^\intercal/\sqrt{1 + \alpha^2(\mu)}$.
    
    We, therefore, have $\lim_{\alpha(\mu)\rightarrow 0}\BasisTangentInputMat = \begin{bmatrix} 0 & 1 \end{bmatrix}^{\intercal}$ and  $\lim_{\alpha(\mu)\rightarrow +\infty}\BasisTangentInputMat = \begin{bmatrix} -1 & 0 \end{bmatrix}^{\intercal}$. The first element in $\BasisTangentInputMat$ approaches 0 when $\alpha(\mu)$ is close to zero, indicating that the control action $u_s$ on the actual system vanishes. On the contrary, the tangent basis will be aligned with the axis of $u_s$ when $\alpha(\mu)$ goes to infinity. 
    \label{example:tangent_basis}
\end{example}
When designing the slack $\alpha(\mu)$ dynamics, we follow a straightforward principle: \textit{The tangential basis should align with the basis of the original system as much as possible when the state is safe and far from the boundary $\partial\SafeSet$.} For example,
\begin{align*}
    &\text{Linear} & &\alpha(\mu) = \beta \mu \\
    &\text{Exponential} & &\alpha(\mu) = \exp(\beta \mu) - 1
\end{align*}

Fig.~\ref{fig:comp_slack_type} compares a different type of slack dynamics function with different hyperparameters $\beta$. The upper part shows the trajectory between Linear and Exponential dynamics. The lower part shows different hyperparameters $\beta$ for Exponential Slack. A stiff slack dynamics function will deform the action space more aggressively, while a soft slack dynamics function will lead to more conservative behavior. In general, it would be desirable to have as little deformation of the action space as possible to avoid performance loss. However, a stiff slack dynamics function will pose a numerical stability issue as the Jacobian matrix $\JacInputMat$ will be ill-conditioned. In addition, a stiff slack dynamics function also requires a higher control frequency as the tangent basis will change more drastically.

\subsection{Continuously Varying Tangent Space Basis}
\label{sec:tangent_space_basis}
\begin{figure*}
    \centering
    \begin{tabular}{c c c c c}
    \includegraphics[width=0.185\linewidth, height=3.5cm, trim=2.2cm 2cm 5cm 5cm, clip]{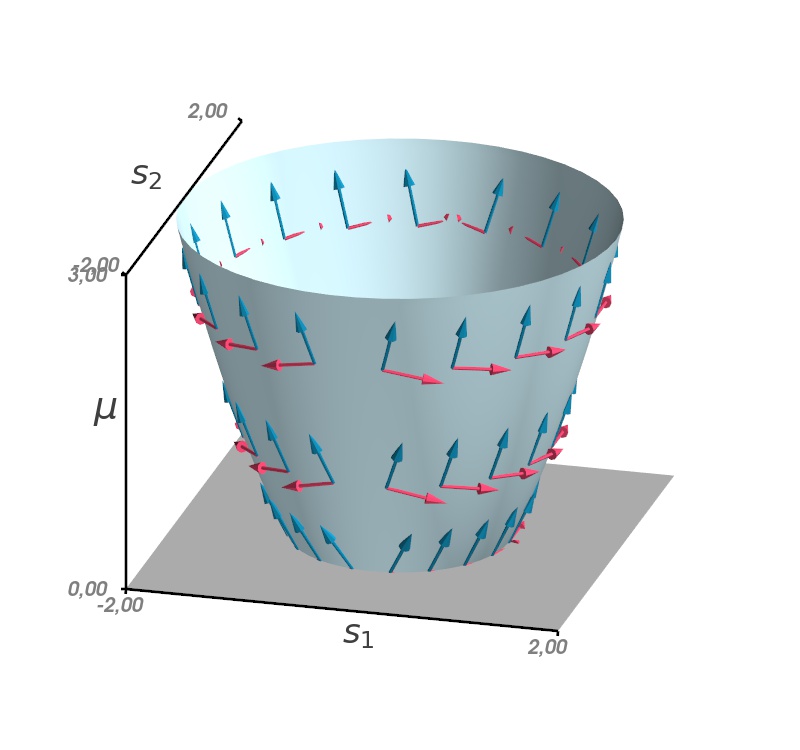} & 
    \includegraphics[width=0.17\linewidth, trim=1.2cm 0.5cm 1.0cm 0.5cm, clip]{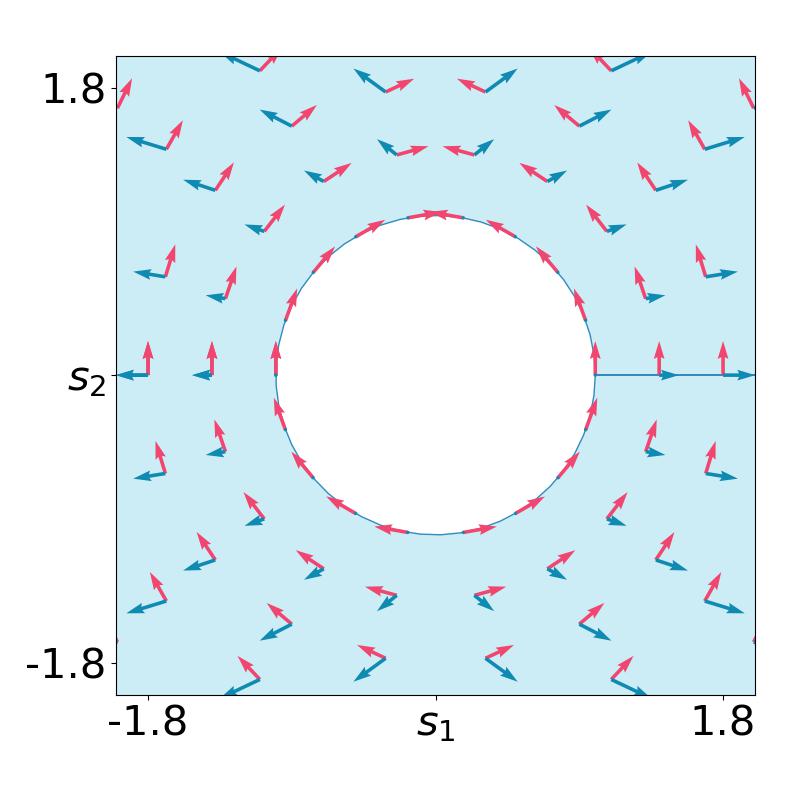} &
    \includegraphics[width=0.17\linewidth, trim=1.2cm 0.5cm 1.0cm 0.5cm, clip]{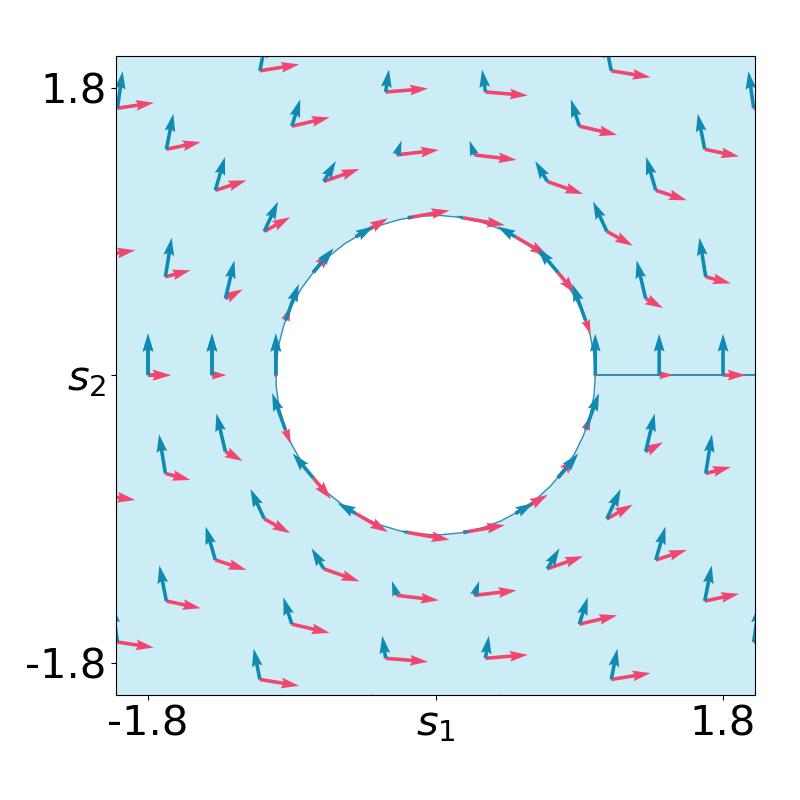} &
    \includegraphics[width=0.17\linewidth, trim=1.2cm 0.5cm 1.0cm 0.5cm, clip]{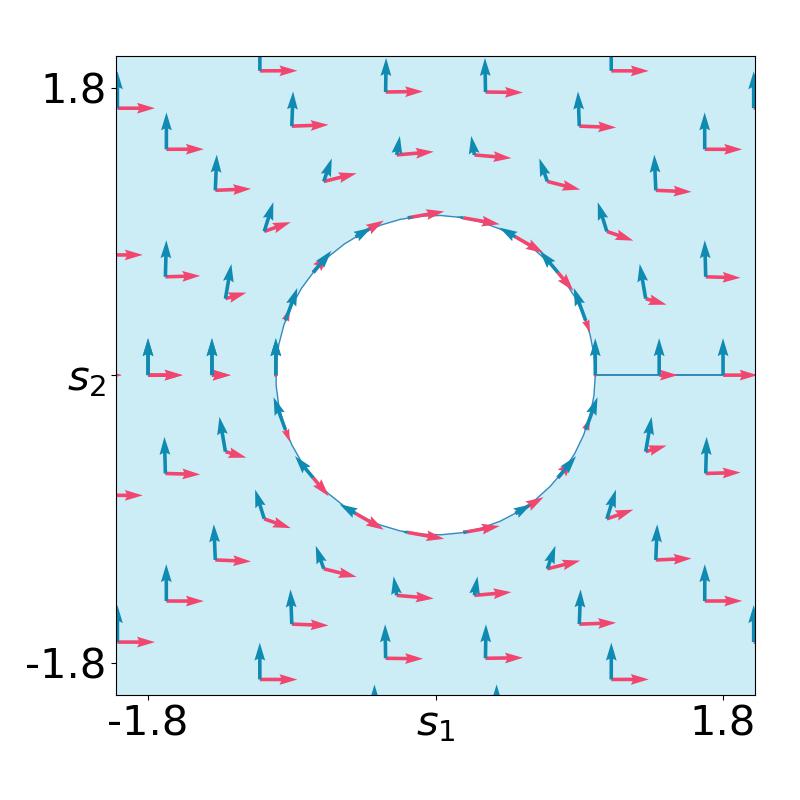} &
    \includegraphics[width=0.185\linewidth, height=3.5cm, trim=2.8cm 4cm 5.5cm 4cm, clip]{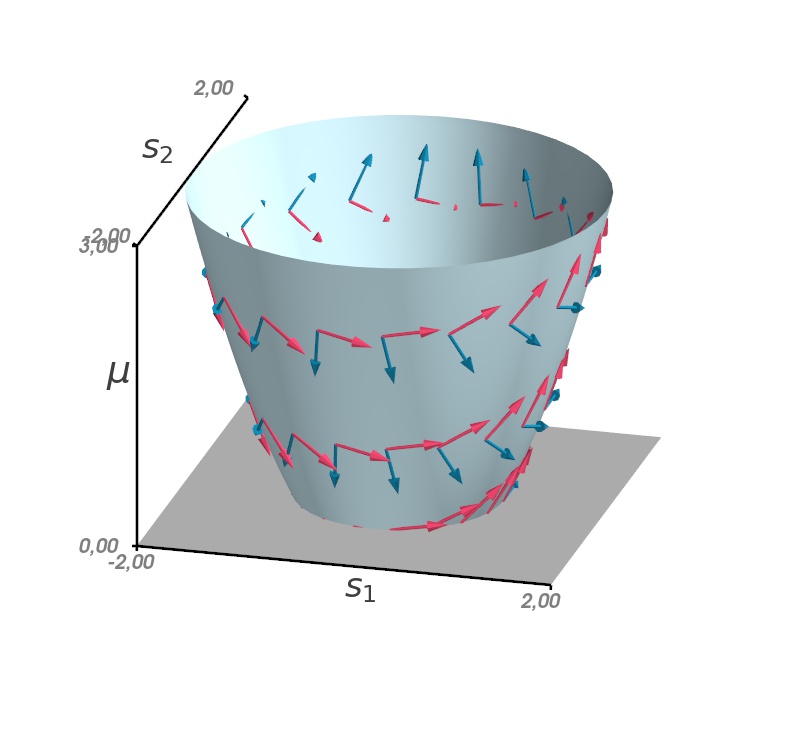}
    \\
    \includegraphics[width=0.185\linewidth, height=3cm, trim=1.9cm 2.5cm 4cm 9cm, clip]{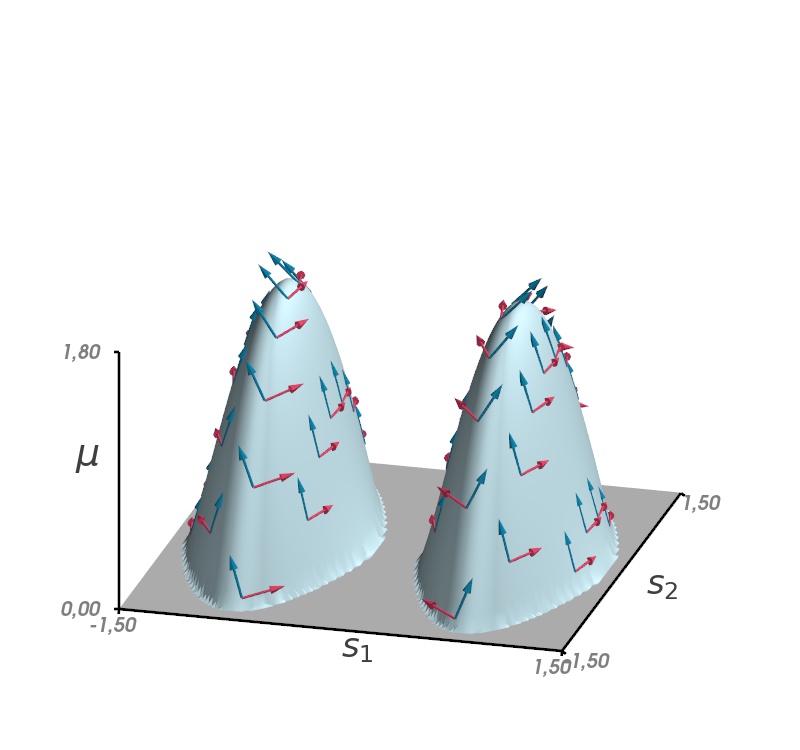} & 
    \includegraphics[width=0.17\linewidth, trim=1.2cm 2cm 1.0cm 0.5cm, clip]{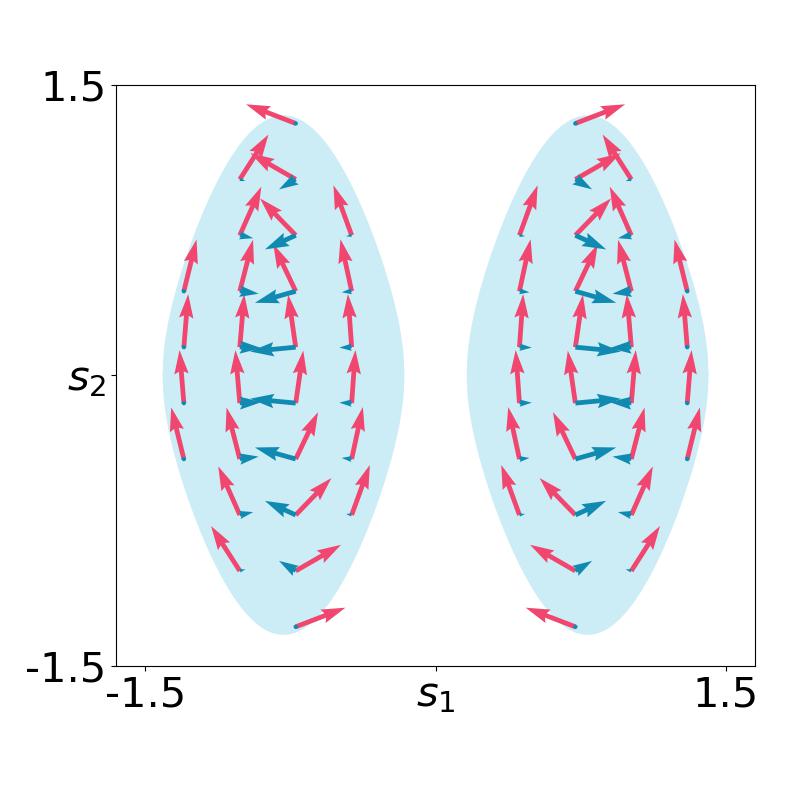} &
    \includegraphics[width=0.17\linewidth, trim=1.2cm 2cm 1.0cm 0.5cm, clip]{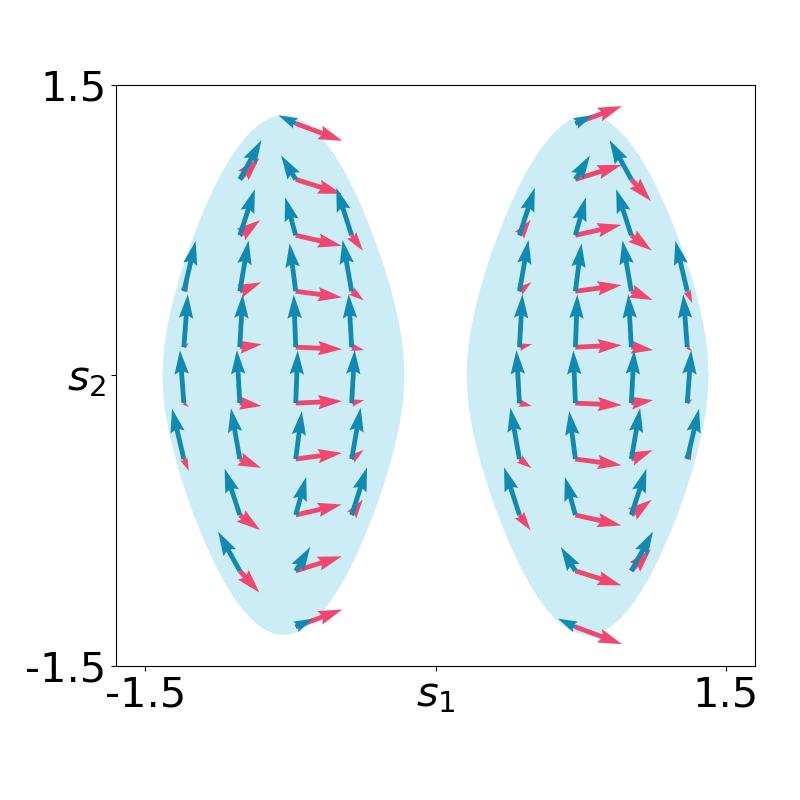} &
    \includegraphics[width=0.17\linewidth, trim=1.2cm 2cm 1.0cm 0.5cm, clip]{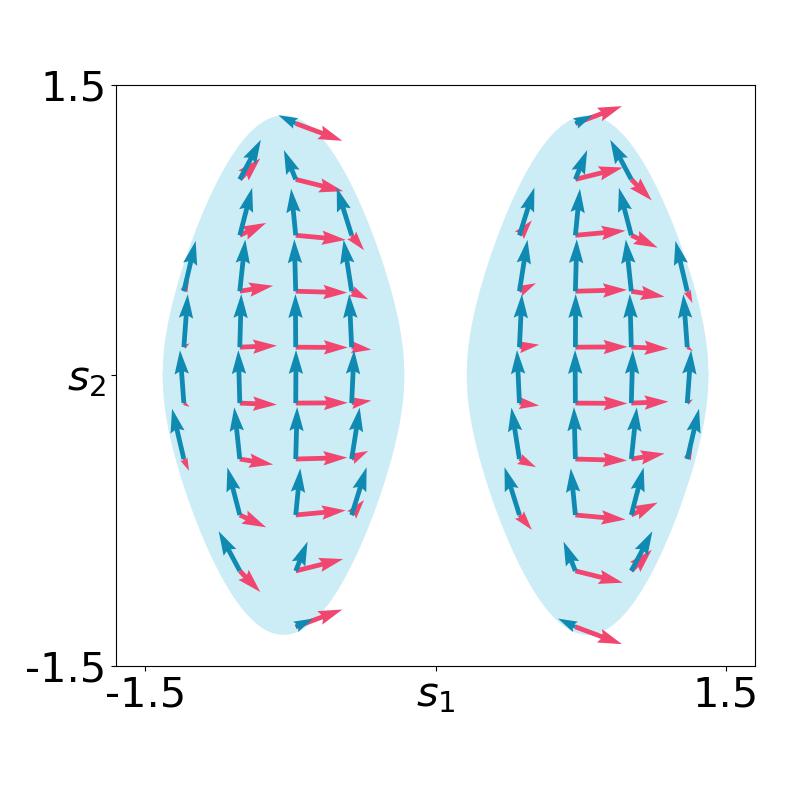} &
    \includegraphics[width=0.185\linewidth, height=3cm, trim=2.5cm 5cm 5cm 7cm, clip]{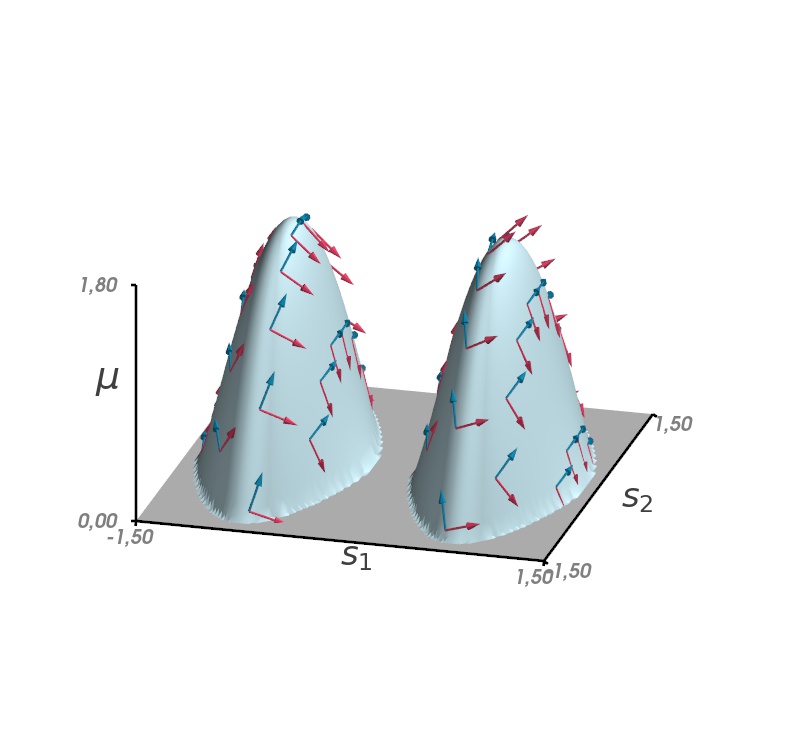} \\
    (a) 3D N-LIN  & (b) 2D N-LIN & (c) 2D S-LIN & (d) 2D S-EXP & (e) 3D S-EXP
    \end{tabular}
    \caption{Comparison of Nonsmooth (N) and Smooth (S) basis using the Linear (LIN) and the Exponential (EXP) slack dynamics function for two different constraints. \textbf{Top Row}: $-(s_1^2 + s_2^2) + 1 \leq 0$, \textbf{Bottom Row}: $\cos(4s_1)+s_2^2 - 0.8 \leq 0$. \textbf{(a)} 3D manifold with linear slack dynamics, the tangent space bases are obtained from QR decomposition. The tangent space bases do not vary smoothly. \textbf{(b)} The tangent basis project onto the original $S_1-S_2$ space. \textbf{(c)} Projected smooth tangent space bases with linear slack dynamics computed by Alg.~\ref{alg:tangent_basis}. The projected tangent space is not orthogonal in the projected space. \textbf{(d)} Projected smooth tangent space bases with exponential slack dynamics. The tangent space bases are less deformed when the state is away from the boundary. \textbf{(e)} The smooth tangent space bases in $S_1-S_2-\mu$ space.}
    \label{fig:comp_smooth_basis}
\end{figure*}

Various approaches have been discussed in the literature on determining the bases of the tangent space (kernel of the Jacobian), such as the Gaussian Elimination method and QR/SVD decomposition. However, the Gaussian Elimination method will result in a non-orthogonal base. Standard QR/SVD-based methods do not generate smooth varying bases as discussed by \cite{coleman1984note}. In addition, any linear combination of the bases constructs other bases. Several techniques have been introduced to construct a continuous function to determine the basis of the kernel given a sequence of the matrices \cite{coleman1984note, gill1985properties, byrd1986continuity}. However, these techniques are path-dependent. Rheinboldt~\cite{rheinboldt1988computation} introduced a moving frame algorithm to compute the path-independent basis. Here, we briefly present the algorithm without proof. 

\begin{algorithm}[b]
\caption{Smooth Varying Basis of the Kernel}\label{alg:tangent_basis}
\begin{algorithmic}[1]
\State \reviseRone{\textbf{Input:} $\mJ$, \textbf{Initialize:} $\mT$}
\State \reviseRone{Compute the basis matrix using SVD/QR decomposition.}
\Statex \reviseRone{$\mB \leftarrow \mathrm{ker}\mJ$}
\State \reviseRone{Compute $\mU_0 \leftarrow \mB^{\intercal} \mT$.}
\State \reviseRone{Compute $\mA, \mSigma, \mB^\intercal \leftarrow \text{SVD}(\mU_0)$}
\State \reviseRone{Obtain the solution $\mQ \leftarrow \mA \mB^\intercal$.}
\State \reviseRone{Form the smooth varying bases $\mB' \leftarrow \mB \mQ$.}
\State \reviseRone{\textbf{Output: } $\mB'$}
\end{algorithmic}
\end{algorithm}

Let $\BasisTangentInputMat\in\RR^{\DimAugmentedState \times \DimControl}$ be an orthonormal matrix of which columns span the kernel of $\JacInputMat \in \RR^{\DimState \times \DimAugmentedState}$ at $(\StateVar, \SlackVar) \in \RR^N$. The matrix $\BasisTangentInputMat$ is not expected to depend continuously on $(\StateVar, \SlackVar)$ and can be obtained, for example, by QR or SVD decomposition. For any orthogonal matrix $\mQ \in \RR^{\DimControl \times \DimControl}$, the matrix $\BasisTangentInputMat\mQ$ is another orthonormal basis of $\mathrm{ker}\JacInputMat$. The objective is to construct $Q: \mathcal{D} \rightarrow \RR^{\DimControl \times \DimControl}$ such that the "rotated" bases $\BasisTangentInputMat\mQ$ depend continuously on $(\StateVar, \SlackVar)$. Rheinboldt \cite{rheinboldt1988computation} suggest formulating an Orthogonal Procrustes Problem as 
\begin{equation*}
    \min_\mQ \Vert (\BasisTangentInputMat\mQ)^\intercal \mT - \mathbb{I}_\DimControl \Vert_F \qquad \mathrm{s.t.}\quad  \mQ^\intercal \mQ = \mathbb{I}_\DimControl
\end{equation*}
where $\mT \in \RR^{\DimAugmentedState \times \DimControl}$ is a matrix with orthonormal columns that span the coordinate space of the manifold, $\Vert \cdot \Vert_F$ is the Frobenius norm, and $\mathbb{I}_\DimControl$ is $\DimControl$-dimensional identity matrix. This problem can be solved using Algorithm~\ref{alg:tangent_basis}. 

Practically, it is desirable to choose the reference coordinate frame $\mT$ as simply as possible. We can choose the $\mT$ to be aligned with the original control space, i.e., the diagonal entry $\mT_{[i, i]}=1, i\in (1, \cdots, U)$ and elsewhere 0. In the following example, we compare the smooth varying basis with the one obtained from QR decomposition.

\begin{example}
Consider a velocity-controlled dynamical system $\dot{\StateVar} = \ControlInput$. We compare the tangent bases for two different constraints, $- (s_1^2 + s_2^2) + 1 \leq 0$ and $\cos(4 s_1) + s_2^2 - 0.8 \leq 0$. The first constraint constructs a connected manifold, while the second builds a disconnected one. The constraint manifolds and their tangent space bases obtained from different approaches are shown in Fig.~\ref{fig:comp_smooth_basis}. We can observe that the smoothed tangent space bases with exponential slack dynamics are less deformed and continuously varying compared to the ones obtained from QR decomposition with linear slack dynamics.
\end{example}

\subsection{Drift Clipping}
\label{sec:drift_clipping}
As described in \eqref{eq:atacom_controller}, the first term $-\JacInputMat^{\dagger} \DriftVec$ compensates for the constraint drift caused by the system drift. Let's consider the effect of the system drift when no control is applied, i.e., $[\ControlInput \; \ControlSlack]^\intercal = \vzero$. 
\begin{align*}
    \dot{\ConstrEqVec}_i &= \dot{\ConstrVec}_i + \dot{\SlackVar}_i = \mJ_{k, i} \DynfVec + \mJ_{k, i} \DynGMat\ControlInput + \DynSlackMat_i\ControlSlack = \DriftVec_i, 
\end{align*}
for $i \in \{1, \cdots, K \}$. Starting from a safe state where $\ConstrVec < \vzero$, we can obtain that when $\DriftFun_i > 0$, we have $\dot{\Constr}_i > 0$, indicating that the drift term is pushing the state to the boundary of the constraint $\partial\MM$. On the other hand, when $\DriftFun_i < 0$, the drift term pulls the state away from the constraint boundary. Therefore, the safe controller only needs to compensate for the drift that tends to the boundary. This can be done by clipping the drift term to be non-negative.
\begin{equation*}
    \widehat{\DriftFun}_i = \max(\DriftFun_i, 0), \quad i \in \{1, \dots, \DimConstr\}
\end{equation*}
We illustrate the benefit of drift clipping in Example~\ref{example:separabel_state_space} after the extension of separable state space (Section~\ref{sec:atacom_dynamic}).

\section{Extensions}
\label{sec:extensions}
We introduced the fundamental concept of \gls{atacom} in Section~\ref{sec:atacom}. In this section, we will expand the scope of its application in several directions. Previously, we assumed full knowledge of the dynamic system and that the constraints were solely determined by the system's state variables. However, in many real-world applications, constraints are influenced by both the robot's state and external states, and we often have only partial knowledge of the dynamic system. For example, the safety constraints in \gls{hri} can be defined by the distance between the human and the robot, but the model of the human motion is difficult to obtain. 
Section~\ref{sec:atacom_dynamic} will discuss extending \gls{atacom} for scenarios with only partial knowledge of the dynamics. Next, we introduce an extension to the second-order system in Section~\ref{sec:second_order_dynamics}. We show by an example why \gls{atacom} controller does not guarantee safety for a second-order system with position-based constraint and then provide a simple solution that modifies the constraint to ensure safety. Section~\ref{sec:equality_constraint} discusses how \gls{atacom} can be applied to the problem with equality constraints. The summary of the extensions is shown in Table~\ref{tab:atacom_extension}.

{\small
\SetTblrInner{rowsep=2pt}
\begin{table*}[t]
    \centering
    \resizebox{0.95\textwidth}{!}{%
    \begin{tblr}{|l|c|c|c|c|c|}
        \hline
         Problem & State $\StateVar$ & Constraint $\ConstrEqVec = \vzero$ & Dynamics $\dot{\StateVar}$ &  Jacobian $\JacInputMat$ & Drift $\DriftVec$ \\ 
         \hline
         \gls{atacom} & $\StateVar$ &  $\Constr(\StateVar) + \SlackVar$ & $\Dynf(\StateVar)+\DynG(\StateVar)\ControlInput$ & $[\JacConstrMat \DynGMat \;\; \DynSlackMat ]$ & $\JacConstrMat \DynfVec$ \\ \hline

         Sec. Order & $ \begin{bmatrix}\StateVar \\ \dot{\StateVar} \end{bmatrix}$ & $ \begin{matrix}    
         \zeta(\Constr(\StateVar)) + \JacConstr(\StateVar) \dot{\StateVar} \\ + \SlackVar \end{matrix} $ & $\begin{bmatrix} \dot{\StateVar} \\ \Dynf(\StateVar, \dot{\StateVar}) \end{bmatrix} + \begin{bmatrix} \vzero \\ \DynG(\StateVar, \dot{\StateVar})  \end{bmatrix} \ControlInput $ & $ [\JacConstrMat \DynGMat \;\; \DynSlackMat] $ & $ \begin{matrix} \JacConstrMat\DynfVec + (\mJ_\zeta \JacConstrMat  \\ + \nabla_{\StateN} \JacConstrMat \dot{\StateVar}) \dot{\StateVar} \end{matrix} $\\ \hline
         
         Dyn. Env. & $[\StateControllableVar \; \StateUncontrollableVar]^\intercal$ & $\Constr(\StateControllableVar, \StateUncontrollableVar) + \SlackVar$ & $\Dynf(\StateControllableVar) + \DynG(\StateControllableVar) \ControlInputControllable $ & $[\JacConstrControllableMat \DynGMat \;\; \DynSlackMat]$  & $\JacConstrControllableMat\DynfVec + \JacConstrUncontrollableMat\dot{\StateUncontrollableVar} $\\ \hline 
         
         Equal. Constr. & $\StateVar$ & $\begin{bmatrix} \Constr(\StateVar) + \SlackVar \\ \ConstrEquality(\StateVar) \end{bmatrix}$ & $\Dynf(\StateVar)+\DynG(\StateVar)\ControlInput$ & $\begin{bmatrix} \JacConstrMat\DynGMat & \DynSlackMat \\ \JacConstrEqualityMat \DynGMat & \vzero \end{bmatrix}$ & $\begin{bmatrix} \JacConstrMat\DynfVec \\ \JacConstrEqualityMat\DynfVec \end{bmatrix}$\\
         \hline
    \end{tblr}
    }
    \vspace{0.5em}
    \normalsize
    \caption{Extension of \gls{atacom} controller with different environment setup.}
    \label{tab:atacom_extension}
    \vspace{-1em}
\end{table*}
}

\subsection{Dynamical Environment with Separable State Space}
\label{sec:atacom_dynamic}
In many applications, robots are interacting in a dynamic environment. Therefore, not all of the state can be directly controlled by the robot. For instance, humans move dynamically in an \gls{hri} environment. To tackle this problem, we assume the state $\StateVar \in \StateSpace$ is separable by a \gls{dcs} $\StateControllableVar \in \StateSpaceControllable$ and a \gls{dus} $\StateUncontrollableVar \in \StateSpaceUncontrollable$ as
\begin{equation*}
    \StateVar = \begin{bmatrix} \StateControllableVar \\ \StateUncontrollableVar \end{bmatrix}, \quad \StateSpace = \StateSpaceControllable \times \StateSpaceUncontrollable
\end{equation*}
The constraint manifold is defined as 
\begin{equation*}
    \ConstrManifold = \left\{(\StateControllableVar, \StateUncontrollableVar, \SlackVar)\in \AugmentedSpace: \ConstrEq(\StateControllableVar, \StateUncontrollableVar, \SlackVar)=\Constr(\StateControllableVar, \StateUncontrollableVar) + \SlackVar=0 \right\}
\end{equation*}
We assume the dynamic system for the \gls{dcs} is known and affine w.r.t. the control, i.e.,
\begin{equation}
    \dot{\StateControllableVar} = f(\StateControllableVar) + G(\StateControllableVar) \ControlInputControllable 
    \label{eq:control_affine_system_q}
\end{equation}
Similar to \eqref{eq:system_requirement}, the following requirement should be satisfied 
\begin{align*}
    \DriftFun(\StateControllableVar, \StateUncontrollableVar) + \JacInput(\StateControllableVar, \StateUncontrollableVar, \SlackVar)\begin{bmatrix} \ControlInputControllable \\ \ControlSlack \end{bmatrix} = \vzero
\end{align*}
where $\DriftFun(\StateControllableVar, \StateUncontrollableVar) = \JacConstrControllable(\StateControllableVar, \StateUncontrollableVar)f(\StateControllableVar) + \JacConstrUncontrollable(\StateControllableVar, \StateUncontrollableVar) \dot{\StateUncontrollableVar}$, $\JacConstrControllable(\StateControllableVar, \StateUncontrollableVar)=\frac{\partial}{\partial \StateControllableN} \Constr(\StateControllableVar, \StateUncontrollableVar)$ and $\JacConstrUncontrollable(\StateControllableVar, \StateUncontrollableVar)=\frac{\partial}{\partial \StateUncontrollableN}\Constr(\StateControllableVar, \StateUncontrollableVar)$ are partial derivatives. $\JacInput(\StateControllableVar, \StateUncontrollableVar, \vmu) = \begin{bmatrix} \JacConstrControllable(\StateControllableVar, \StateUncontrollableVar)\DynG(\StateControllableVar) & \DynSlackFun(\vmu) \end{bmatrix}$. 

Compared to the previous derivation \eqref{eq:system_requirement}, the drift term $\psi(\StateControllableVar, \StateUncontrollableVar)$ has an additional source from the uncontrollable state's motion $\JacConstrUncontrollable(\StateControllableVar, \StateUncontrollableVar) \dot{\StateUncontrollableVar}$. In practice, we can estimate the velocity of the uncontrollable states by the finite difference or by using a state observer. A comparison of different velocity observations is shown in Section~\ref{sec:experiment_dyn_env_obs}.
We can derive the \gls{atacom}-controller similarly to Eq.~\eqref{eq:atacom_controller} where $\mJ_u$ and $\vpsi$ should be adapted. We illustrate the effect of drift clipping with an example involving a dynamic moving obstacle in a 2D environment.

\begin{figure*}[t]
    \centering
    \subfloat[]{
    \includegraphics[width=0.192\textwidth, trim={2.cm 1.9cm 1.0cm 1.0cm}, clip]{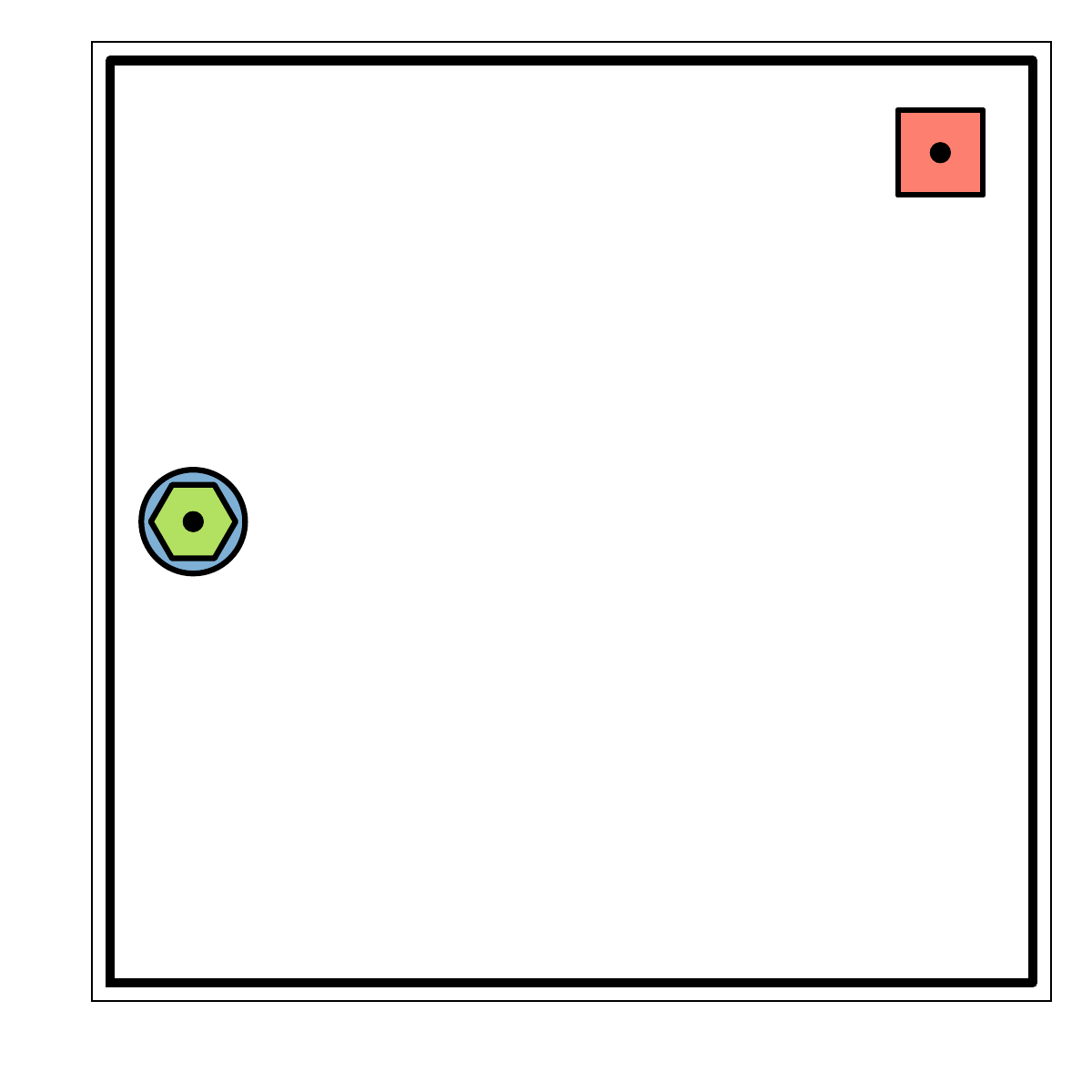}
    }
    \subfloat[]{
    \includegraphics[width=0.192\textwidth, trim={2.cm 1.9cm 1.0cm 1.0cm}, clip]{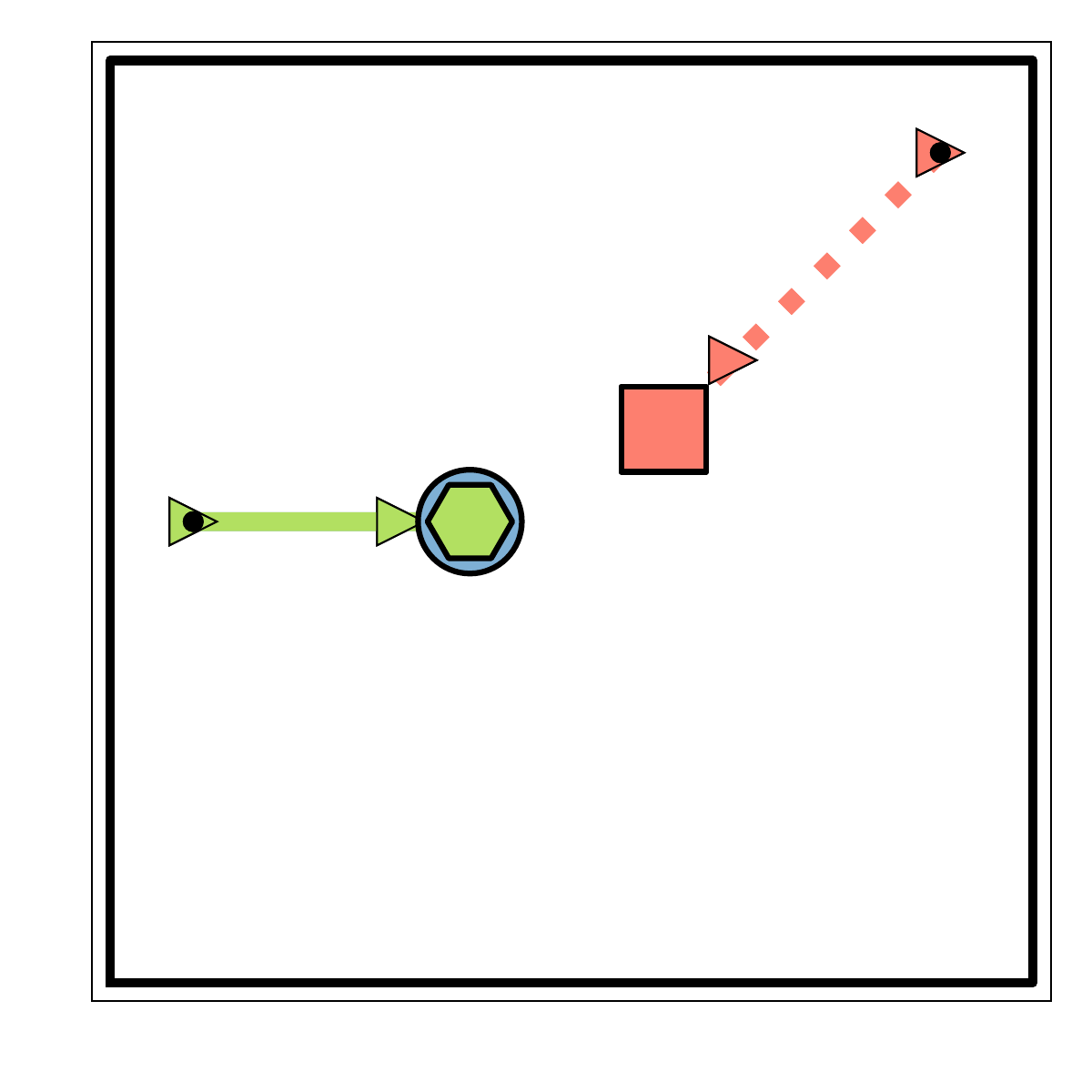}}
    \subfloat[]{
    \includegraphics[width=0.192\textwidth, trim={2.cm 1.9cm 1.0cm 1.0cm}, clip]{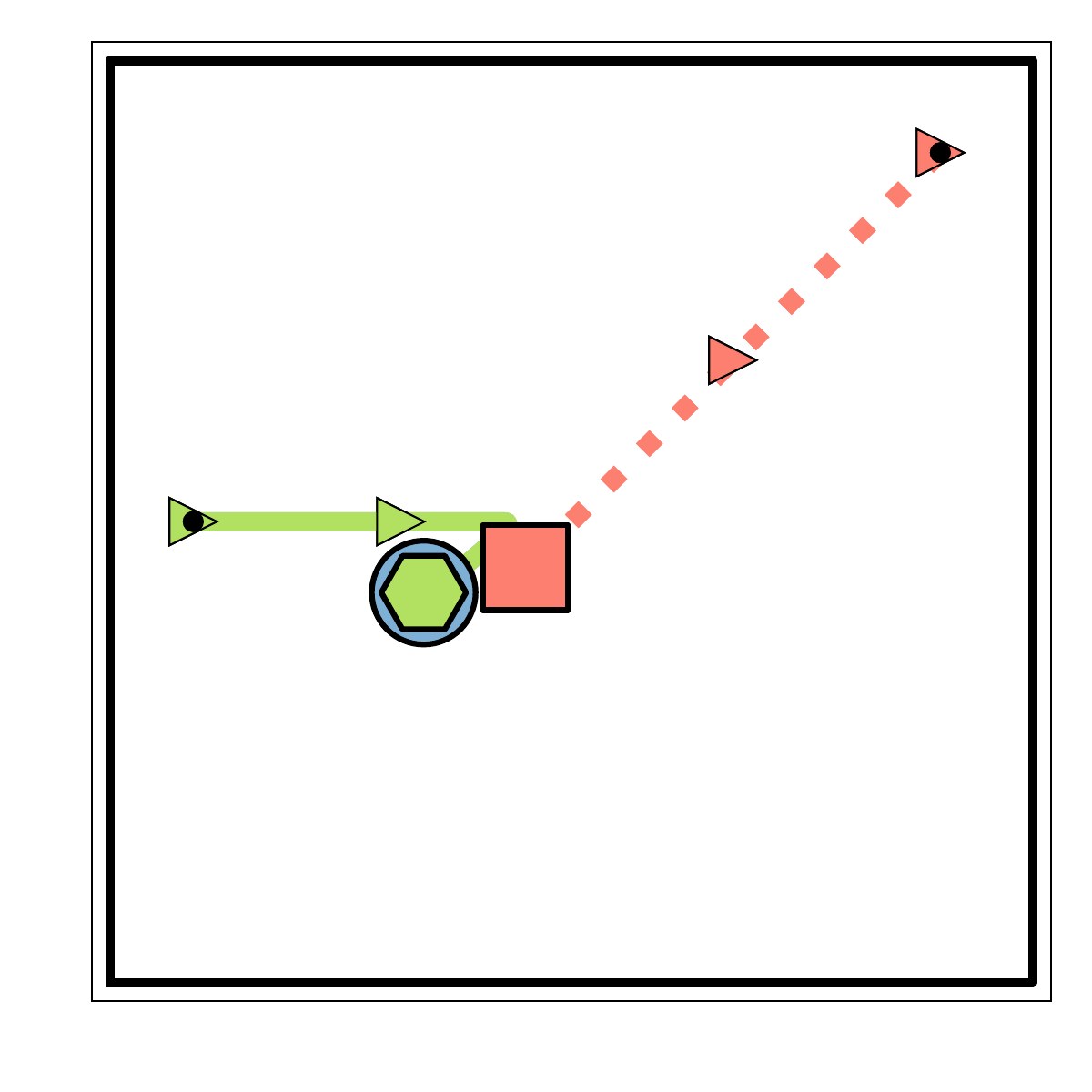}}
    \subfloat[]{
    \includegraphics[width=0.192\textwidth, trim={2.cm 1.9cm 1.0cm 1.0cm}, clip]{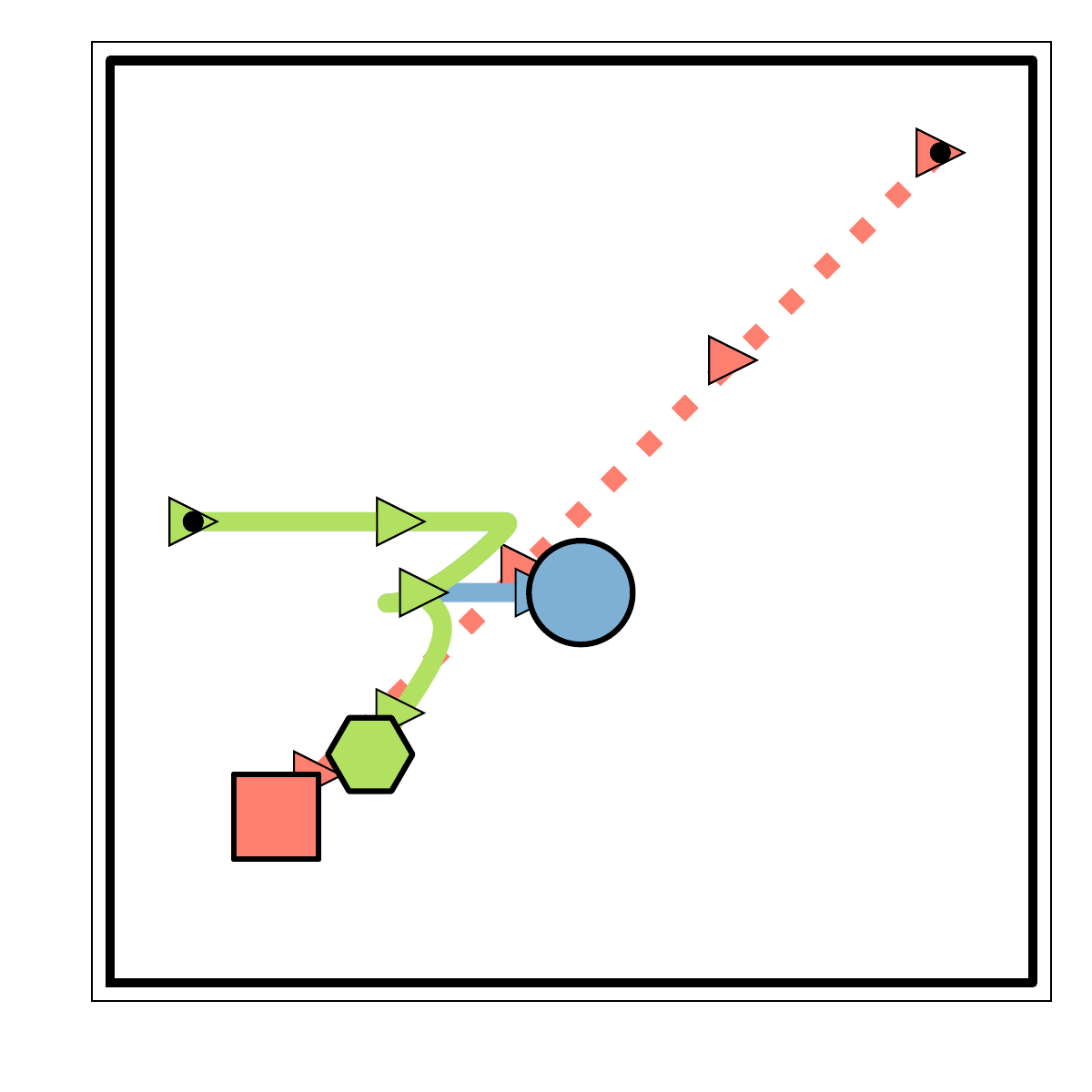}}
    \subfloat[]{
    \includegraphics[width=0.192\textwidth, trim={1.9cm 1.9cm 1.0cm 1.0cm}, clip]{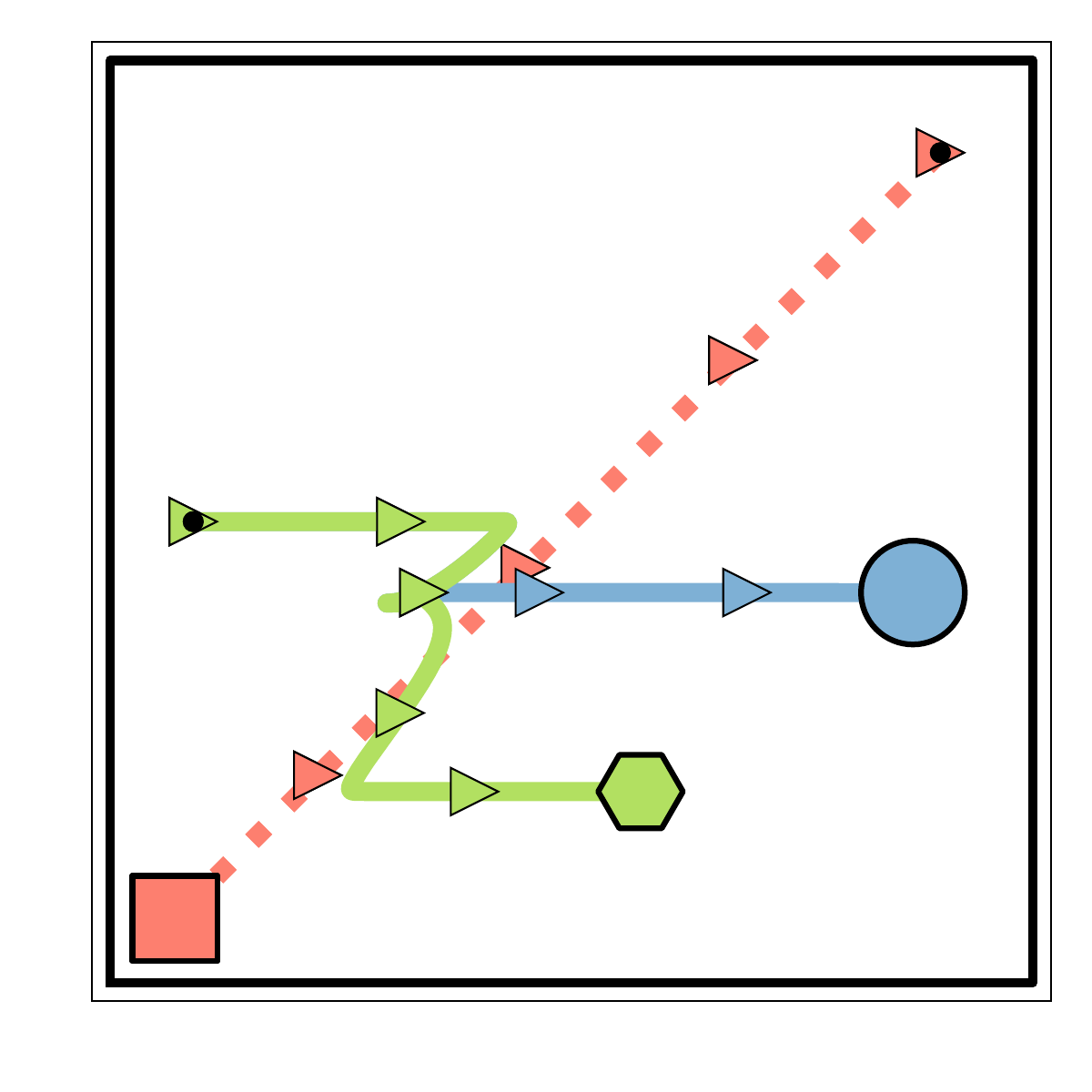}}
    
    \includegraphics[width=\textwidth, trim={5cm 2.2cm 5cm 0.8cm}, clip]{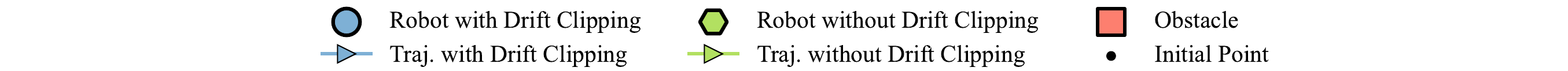}
    \caption{Comparison of \gls{atacom} controller with and without \textit{Drift Clipping} in a 2D environment with a moving obstacle. The robot (blue circle or green hexagon) has a constant control input $\vu = [1 \; 0]^\intercal$. Each figure represents the positions of the robot and the obstacle at a different time step. \textbf{(a)} The robot begins from the right side, and the obstacle (red square) starts from the upper right corner to the lower left corner. \textbf{(b)} The robot and the obstacle are approaching each other. The drift term is positive since the distance between the robot and the obstacle decreases as the obstacle moves. \textbf{(c)} The drift compensation term drives the robot in the lower-left direction. \textbf{(d)} The distance between the obstacle and the robot increases due to the movement of the obstacle, so the drift term is negative. \textbf{(e)} The robot with drift clipping recovers to the original direction quickly, while the robot without drift clipping follows the obstacle's movement.}
    \label{fig:drift_clipping}
    \vspace{-0.5em}
\end{figure*}

\begin{example}
\label{example:separabel_state_space}
Fig.~\ref{fig:drift_clipping} shows an environment with an obstacle (red square) moving from the upper right corner to the lower left corner with a constant velocity. The robot (blue circle or green hexagon) is initialized from the right side. The \gls{dcs} is the $XY$-position of the robot $\StateControllableVar = [x_r \; y_r]^\intercal$, and the \gls{dus} is the $XY$-position of the obstacles $\StateUncontrollableVar = [x_o \; y_o]^\intercal$. We define the constraint as $\Constr(\StateControllableVar, \StateUncontrollableVar) = - \Vert \StateControllableVar - \StateUncontrollableVar \Vert + \eta < 0$. The robot is directly controlled by the velocity $\dot{\StateControllableVar} = \ControlInputControllable$.  We applied a constant control input $\vu = [1 \; 0]^\intercal$ to the robot and compared the effect of the drift clipping (Section~\ref{sec:drift_clipping}) on the robot's trajectory. Since $\DynfVec = \vzero$, the drift term only contains $\JacConstrUncontrollableMat \dot{\StateUncontrollableVar}$.
The drift term is positive when the obstacle moves toward the robot and negative when the obstacle moves away from the robot. Fig.~\ref{fig:drift_clipping} illustrates the robot's trajectory with (blue) and without (green) drift clipping. We can see that the robot with drift clipping recovers to the original motion direction quickly. In contrast, the robot without drift clipping follows the movement of the obstacles to compensate for the drift.
\end{example}

\subsection{Second-Order Dynamics}
\label{sec:second_order_dynamics}
The previous method focuses on first-order affine control systems with state constraints, but it cannot be directly applied to higher-order systems by simply converting them to first order. We first show through an example that using only position input violates the assumptions. Then, we present a solution for second-order systems, which can be easily extended to higher-order systems.

\begin{example}
Consider an acceleration-controlled system, i.e., $\ddot{\StateVar} = \ControlInput$, which can be written as a first-order affine system
\begin{equation*}
    \dot{\widehat{\StateVar}}=\begin{bmatrix} \dot{\StateVar} \\ \ddot{\StateVar} \end{bmatrix} = \begin{bmatrix} \dot{\StateVar} \\ \vzero \end{bmatrix} +
    \begin{bmatrix} \vzero \\ \mathbb{I} \end{bmatrix} \ControlInput = \DynfVec + \DynGMat \ControlInput
\end{equation*}
For the constraint $k(\StateVar)<\vzero$ that only takes $\StateVar$ as input, the Jacobian with respect to $\widehat{\StateVar}$ is $\JacConstrMat = \begin{bmatrix} \frac{\partial}{\partial \StateN}\Constr & \vzero \end{bmatrix}$. Assumption~\ref{asm:constant_rank} does not hold as $(\JacInputGMat) = \mathrm{rank}(\JacConstrMat \DynGMat) = 0$. \hfill \qedsymbol
\end{example}

Consider a general second-order control affine system 
\begin{equation} 
    \ddot{\StateVar} = \Dynf(\StateVar, \dot{\StateVar}) + \DynG(\StateVar, \dot{\StateVar}) \ControlInput
    \label{eq:second_order_system}
\end{equation}
For a state-only constraint $\Constr(\StateVar)$ of class $C^2$, we can convert the constraint for the second-order system as: 
\begin{equation}
    \Constr^*(\StateVar, \dot{\StateVar}) = \zeta(\Constr(\StateVar)) + \dot{\Constr}(\StateVar) = \zeta(\Constr(\StateVar)) + \JacConstrMat \dot{\StateVar} \leq \vzero
    \label{eq:constraint_second_order}
\end{equation}
where $\zeta(\cdot)$ is a class $\mathcal{K}$ function, following the idea of High-order \gls{cbf} from \cite{xiao2022high_order}. We again introduce the slack variable $\SlackVar$ to construct the constraint manifold and set the time-derivatives of the new constraint $\ConstrVec^{*}$ to zero and obtain 
\begin{equation*}
    (\mJ_{\zeta} + \nabla_{\StateN} \JacConstrMat \dot{\StateVar}) \dot{\StateVar} + \JacConstrMat \DynfVec + \JacConstrMat \DynGMat \ControlInput + \DynSlackMat \SlackVar = \vzero
\end{equation*}

By grouping the terms together, we obtain a requirement in the same form as  Eq.~\eqref{eq:system_requirement}. The Jacobian $\JacInputMat$ and the drift $\DriftVec$ are defined as
\begin{equation*}
    \JacInputMat = [\JacConstrMat \DynGMat \;\; \DynSlackMat], \; \DriftVec = \JacConstrMat\DynfVec + (\mJ_\zeta \JacConstrMat + \nabla_{\StateN} \JacConstrMat \dot{\StateVar}) \dot{\StateVar}
\end{equation*}
we notice that $\JacInputMat$ has exactly the same structure as the first-order system, the only difference is the drift contains an additional term due to the velocity $\dot{\StateVar}$. If Assumption~\ref{asm:constant_rank} holds, $\JacInputMat$ is full rank, Theorem~\ref{theorem:atacom_safety} synthesizes a safety controller to the constraint function defined in Eq.~\eqref{eq:constraint_second_order}.

\subsection{Equality Constraints}
\label{sec:equality_constraint}
\gls{atacom} controller can be extended for equality constraint $\ConstrEquality(\StateVar)=\vzero$, 
where $\ConstrEquality: \StateSpace \rightarrow \RR^{\DimConstrEquality}$ is of class $C^1$ with the Jacobian $\JacConstrEquality: \StateSpace \rightarrow \RR^{\DimConstrEquality \times \DimState}$.  We assume the system with equality constraint is under-constrained, i.e., $\DimConstrEquality < \DimState$. We can formulate the constraint manifold 
\begin{equation*}
    \ConstrManifold = \{(\StateVar, \SlackVar)\in \AugmentedSpace: c(\StateVar, \SlackVar) = \vzero \}
\end{equation*}
where $c(\StateVar, \SlackVar) = \begin{bmatrix} k(\StateVar) + \SlackVar &  l(\StateVar) \end{bmatrix}^{\intercal}$. 

The Jacobian and the drift are adapted correspondingly as
\begin{equation*}
\JacInputMat = \begin{bmatrix} \JacConstrMat \DynGMat & \DynSlackMat \\ \JacConstrEqualityMat \DynGMat & \vzero \end{bmatrix}, \; \DriftVec = \begin{bmatrix} \JacConstrMat \DynfVec \\ \JacConstrEqualityMat \DynfVec \end{bmatrix} 
\end{equation*} 
To ensure the Eq.~\eqref{eq:system_requirement} is solvable, the following additional assumptions on the equality constraint should be satisfied. 
\begin{assumption}
    The Jacobian $\JacConstrEqualityMat\DynGMat$ is full row rank \mbox{$\forall (\StateVar, \SlackVar) \in \Omega_\eta$}.
\end{assumption}
We can again construct the \gls{atacom} controller using Eq.~\eqref{eq:atacom_controller}. The controlled system will approach $\ConstrManifold$ from the neighborhood $\Omega_\eta$ and stay on the manifold. 

However, simply choosing a matrix $\mT \in \RR^{\DimAugmentedState \times \DimControl}$ whose diagonal entry is $\mT_{[i, i]}=1$ is no longer a valid choice for the coordinate space in this setting. For example, for a \mbox{1-sphere} in 2D, the coordinate frame $\mT = \begin{bmatrix} 1 & 0 \end{bmatrix}^\intercal$ is not a valid coordinate as there are multiple points mapped to the same point. The \mbox{2-sphere} does not even contain a smooth varying basis, according to the \textsl{Hairy Ball Theorem}~\cite{Brouwer1912}. Unfortunately, the choice of the reference matrix is not trivial. We leave this discussion for future work.
\section{Experiments}
\label{sec:experiments}

In this section, we will analyze the effectiveness of different techniques introduced in previous sections in some simple environments to provide a thorough understanding of the algorithm. 
In the first experiment, we will compare the effect of different slack dynamics functions and their hyperparameters in a 2D static collision avoidance environment (Fig.~\ref{fig:2d_static_col_avoid}). 
Then, we will compare the effects of different velocity observation methods in a dynamic environment containing moving obstacles (Fig.~\ref{fig:2d_dynamic_col_avoid}). 
Additionally, we evaluate how the dynamic mismatch will impact the performance of \gls{atacom} in a simulated air hockey task (Fig.~\ref{fig:air_hockey_sim}).
Finally, we show in the real-world robot air hockey experiment (Fig.~\ref{fig:air_hockey_real_world_env}) that we can fine-tune the RL agent from online real-world interactions with the help of \gls{atacom}. 
\reviseRthree{
Videos of experiments are available at \href{https://puzeliu.github.io/TRO-ATACOM}{https://puzeliu.github.io/TRO-ATACOM}.
}

\reviseRtwo{The experiments demonstrated in this paper majorly focus on simple tasks with low-dimensional dynamics. Nevertheless, experiments of applying \gls{atacom} to more complex, high-dimensional, dynamic environments can be found in~\cite{liu2023safe}. We demonstrate that \gls{atacom} enables the TIAGo robot, using a unicycle model, to safely navigate to the target while avoiding blindly moving mobile robots. Furthermore, using the learned \gls{sdf}~\cite{liu2022regularized} as constraints, \gls{atacom} allows the TIAGo robot to avoid collisions with a human while holding a cup vertically in a shared workspace or to reach the target inside a shelf with complex geometry.}

\begin{figure*}
    \centering
    \subfloat[2D-StaticEnv]{\includegraphics[height=0.18\textwidth, trim=0.7cm 0.75cm 0.8cm 0.75cm, clip]{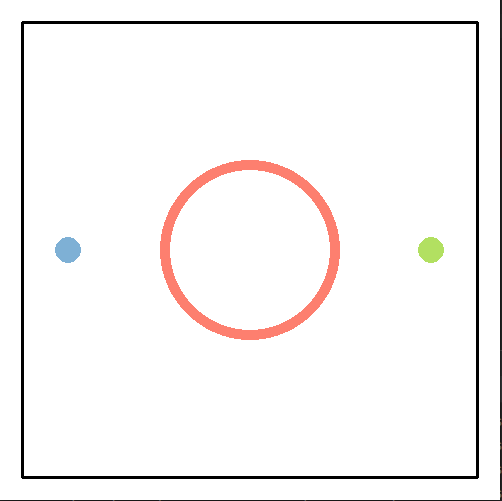}
    \label{fig:2d_static_col_avoid}}
    \hfill
    \subfloat[2D-DynamicEnv]{\includegraphics[height=0.18\textwidth, trim=0.6cm 0.75cm 0.5cm 0.72cm, clip]{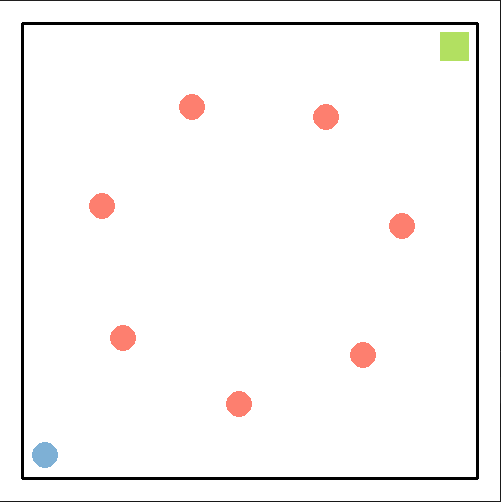}
    \label{fig:2d_dynamic_col_avoid}}
    \hfill
    \subfloat[QuadrotorEnv]{\includegraphics[height=0.18\textwidth, width=0.26\textwidth]{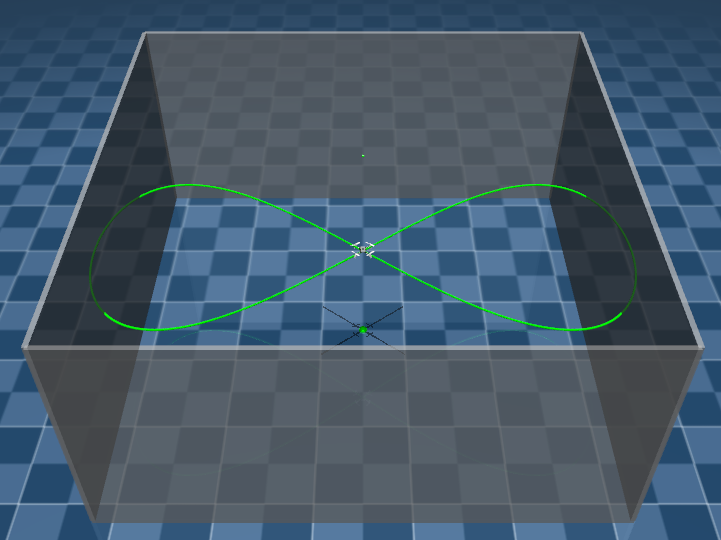}
    \label{fig:quadrotor_env}}
    \hfill
    \subfloat[AirHockeySim]{\includegraphics[height=0.18\textwidth, width=0.32\textwidth, trim=0cm 3cm 0 3cm, clip]{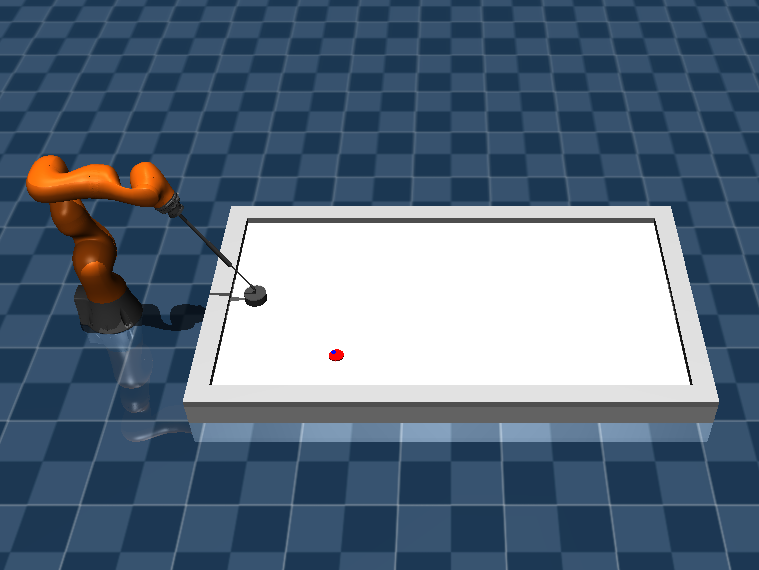}
    \label{fig:air_hockey_sim}}
    \\
    \subfloat[AirHockeyRealWorld]{\includegraphics[width=0.99\linewidth, trim=0 0 0 0, clip]{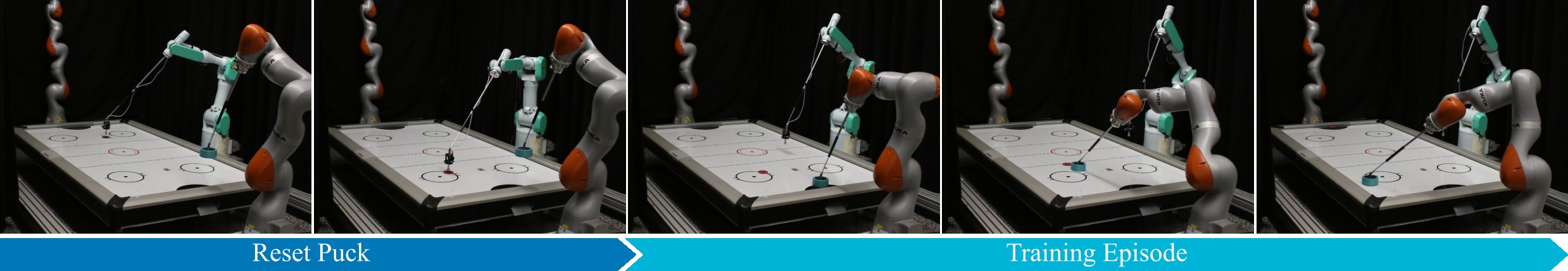}
    \label{fig:air_hockey_real_world_env}
    }
    \caption{Experiment environments. (a) The task of \textsl{2D-StaticEnv} is to move the robot (blue point) to the target (green point) while avoiding the obstacle (red circle). The robot is controlled directly by the velocity. (b) In \textsl{2D-DynamicEnv}, multiple obstacles (red points) move randomly or in a fixed pattern with different velocity scales. The robot tries to reach the target without colliding with the obstacles. (c) The KUKA iiwa Robot tries to hit the puck (red) to the goal while avoiding collisions with the table. The RL agent produces the desired joint velocity, which is tracked by a torque controller. (d) In the \textsl{AirHockeyRealWorld}, the resetting is conducted by a pre-programmed PA10 robot. The learning agent controls the KUKA robot to hit the puck to the goal.}
    \vspace{-1em}
\end{figure*}

\subsection{\reviseRtwo{Different Slack Dynamics in 2D-StaticEnv}}\label{sec:experiment_slack_comp}

We introduced multiple types of slack dynamics functions to ensure safety in Section~\ref{sec:slack_dynamics_functions}. 
In this experiment, we will compare the effect of different types of slack dynamics functions (Linear and Exponential) and their hyperparameters $\beta$ in a 2D collision avoidance environment, shown in Fig.~\ref{fig:2d_static_col_avoid}. 
A circular obstacle shown in red stays in the middle of the environment.
The task is to control the planar robot (blue point) to reach the target (green point). The robot is controlled by velocity 
$\begin{bmatrix}\dot{x} \; \dot{y}\end{bmatrix}^\intercal = \begin{bmatrix}u_x \; u_y\end{bmatrix}^\intercal$.
The constraints are collision avoidance w.r.t. the fixed obstacle, i.e., $\Vert \vp_r - \vp_o \Vert > r_{\text{o}}$, and keeping the robot inside the boundaries, i.e.,  $l_i<p_{r, i}<u_i, i\in \{x, y\}$. 
The reward is the negative distance to the target $r = -\Vert \vp_r - \vp_t\Vert$. We use SAC~\cite{haarnoja2018soft} as \gls{rl} algorithm. The control frequency is set to be $100Hz$ in all experiments. The learning parameters are the same among all experiments except for the parameters in comparison, i.e., the type of slack dynamics function and its parameters $\beta$. Details of the hyperparameters can be found in Table~\ref{tab:2d_static_hyperparameters} in Appendix~\ref{app:hyperparameters}.
We ran 25 seeds for each experiment setting. 

\begin{figure}[t]
    \centering
    \includegraphics[width=0.87\linewidth, trim=1cm 1cm 1cm 1cm, clip]{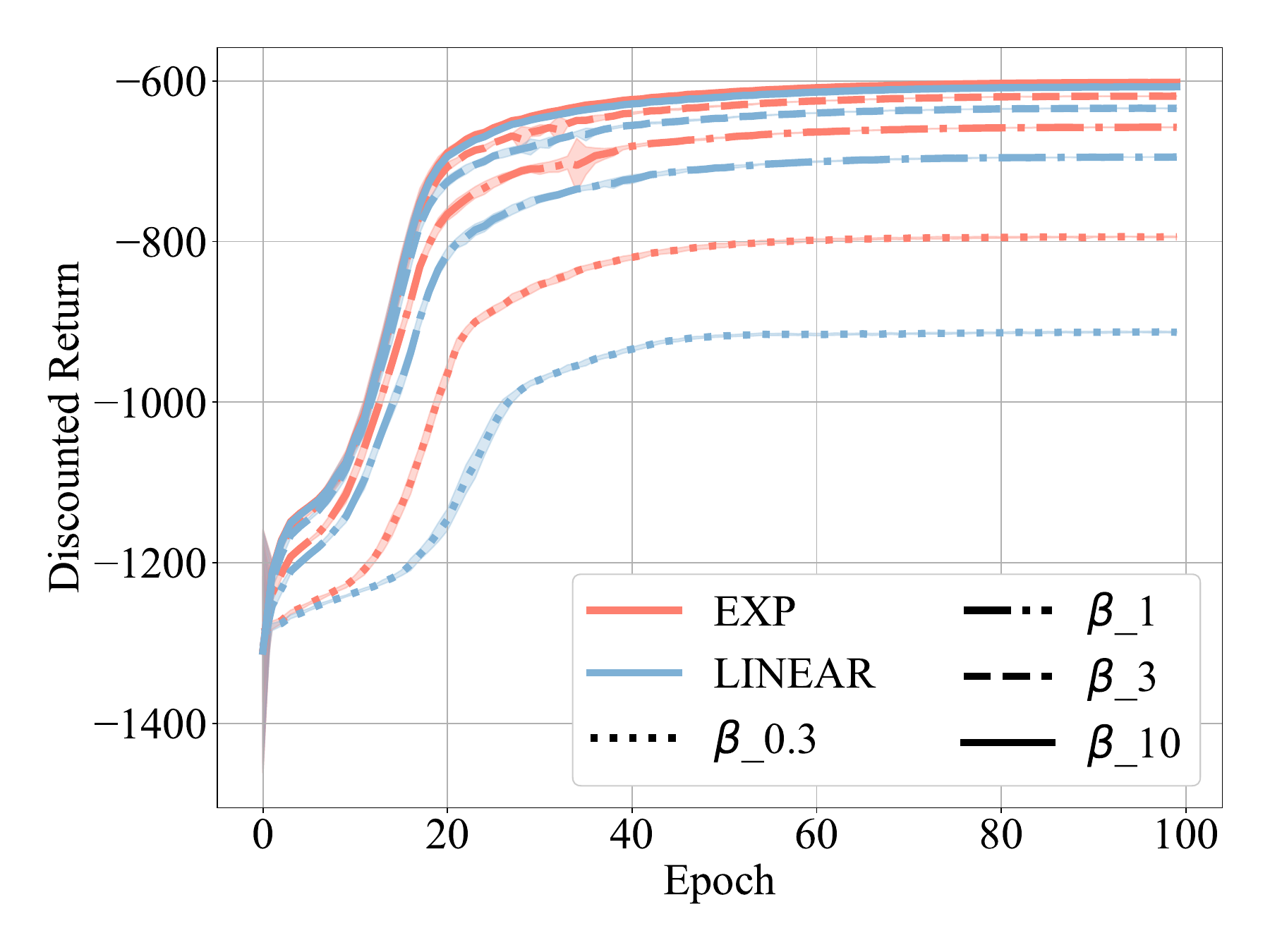}
    \caption{Learning Curve of 2D-Static Environment. Different colors represent different slack dynamics functions. The line types represent different $\beta$ settings. The shaded area represents the $95\%$ confidence interval of 25 independent runs.}
    \label{fig:2d_static}
    \vspace{-1em}
\end{figure}

The learning curves are shown in Fig.~\ref{fig:2d_static}. In all experiment settings, the robot has not collided with the obstacle during training. 
The robot can reach the target in all experiment settings within a fixed horizon length except for the setting of linear slack dynamic function with $\beta=0.3$. The reason is that the action space is heavily morphed by the slack dynamics function and the robot is taking a conservative behavior. 
The lower reward result indicates that the robot is reaching the target slower than the higher reward result. We can clearly see that the slack dynamics function with less morphing of the action space (higher $\beta$ and exponential slack dynamics function) leads to better training performance. However, a higher $\beta$ parameter requires a higher control frequency as the manifold will have a higher curvature, causing the system to deviate from the manifold due to the discretization error.

{
\SetTblrInner{rowsep=2pt}
\newcolumntype{?}{!{\vrule width 2pt}}
\tiny
\begin{table*}[b]
    \centering
    \resizebox{0.95\textwidth}{!}{
    \begin{tblr}{|c|c|[2pt]ccc|[2pt]ccc|[2pt]ccc|[2pt]ccc|[2pt]}
    \hline
    \SetCell[c=2, r=2]{b,r} \rotatebox[origin=c]{90}{\# Obs.} & & \SetCell[c=3]{c} Success Rate & & & \SetCell[c=3]{c} Episode Length & & & \SetCell[c=3]{c} Success Rate & & & \SetCell[c=3]{c} Episode Length & & 
    \\ \cline{3-14}
    & & Exact & FD & None & Exact & FD & None & Exact & FD & None & Exact & FD & None \\
    \hline
    \SetCell[r=3]{} \rotatebox[origin=c]{90}{LOW} 
    & 2 & \textbf{1.00} & 0.99 & 0.87 & \textbf{234.37} & 242.40 & 320.90 & \textbf{1.00} & 0.99 & 0.87 & \textbf{230.26} & 240.75& 320.45 \\
    & 6 & \textbf{1.00} & 0.97 & 0.74 & \textbf{243.46} & 267.56 & 420.00 & \textbf{1.00} & 0.97 & 0.68 & \textbf{245.45} & 270.49 & 465.93 \\
    & 10 & \textbf{1.00} & 0.93 & 0.55 & \textbf{281.33} & 335.36 & 552.90 & \textbf{1.00} & 0.93 & 0.51 & \textbf{272.36} & 320.36 & 598.09 \\
    \hline
    \SetCell[r=3]{} \rotatebox[origin=c]{90}{MEDIUM} 
    & 2 & \textbf{1.00} & 0.95 & 0.87 & \textbf{230.99} & 265.69 & 320.47 & \textbf{1.00} & 0.97 & 0.86 & \textbf{231.64} & 250.29 & 323.36 \\
    & 6 & \textbf{1.00} & 0.87 & 0.53 & \textbf{247.46} & 334.70 & 568.30 & \textbf{0.99} & 0.92 & 0.63 & \textbf{249.14} & 300.66 & 504.11 \\
    & 10 & \textbf{0.90} & 0.64 & 0.34 & \textbf{285.03} & 507.50 & 716.45 & \textbf{0.98} & 0.83 & 0.45 & \textbf{274.69} & 382.43 & 640.84 \\
    \hline
    \SetCell[r=3]{} \rotatebox[origin=c]{90}{HIGH} 
    & 2 & \textbf{0.98} & 0.92 & 0.84 & \textbf{244.93} & 284.46 & 344.39 & \textbf{1.00} & 0.97 & 0.85 & \textbf{228.96} & 249.22 & 332.13\\
    & 6 & \textbf{0.93} & 0.71 & 0.38 & \textbf{296.02} & 441.88 & 686.10 & \textbf{0.98} & 0.90 & 0.60 & \textbf{255.28} & 319.77 & 524.45 \\
    & 10 & \textbf{0.87} & 0.58 & 0.24 & \textbf{356.48} & 540.02 & 417.23 & \textbf{0.96} & 0.80 & 0.46 & \textbf{297.39} & 393.61 & 632.07\\
    \hline
    \SetCell[c=2]{c} & & \SetCell[c=6]{c} FIXED PATTERN & & & & & & \SetCell[c=6]{c} RANDOM MOTION & & & & & \\
    \hline
    \end{tblr}
    }
    \vspace{0.5em}
    \normalsize
    \caption{Comparison of slack dynamic and hyperparameters in 2D collision avoidance environment.}
    \label{tab:experiment_dyn_env_obs}
\end{table*}
}
\subsection{\reviseRtwo{Different Velocity Observations in 2D-DynamicEnv}} \label{sec:experiment_dyn_env_obs}
As described in Section~\ref{sec:atacom_dynamic}, the velocity of the \gls{dus} $\dot{\StateUncontrollableVar}$ leads to additional drift to the system. An accurate observation of $\dot{z}$ is beneficial to ensure safety. However, in many robotic applications, observing such velocity is often challenging and inaccurate. In this experiment, we compare the safety performance with different types of velocity observation in a dynamic environment with multiple moving obstacles. 

We use a 2D Dynamic Environment as a study example, as shown in Fig.~\ref{fig:2d_dynamic_col_avoid}. The robot (blue point) is controlled by velocity.
The constraints are defined as keeping the distance to each obstacle bigger than a threshold, i.e., $\Vert \vp_r - \vp_{o_i}\Vert > r_o, i = 1, 2, \cdots$. 
The target (green square) is randomly placed in the environment for each episode. 
We consider two types of moving obstacles: (1) The obstacles move in a fixed pattern, i.e., each obstacle moves along a circle around its initial position. (2) The obstacles move randomly in the environment. 
In this experiment, we use a hand-crafted policy that tries to reach the target with a linear attractor 
$$\vu = K_p \begin{bmatrix}x_t - x \\ y_t - y\end{bmatrix}$$
where $K_p$ is a positive definite matrix, $[x_t \; y_t]^\intercal$ is the target position. The action vector $\vu$ is then converted to $[u_x \; u_y]^\intercal$ using \gls{atacom}. The observation includes the robot's position/velocity and the obstacles' position/velocity.

In this experiment, we compare the safety performance of (a) exact velocity observations (EXACT), (b) no zero velocity (NONE), and (c) velocities obtained from the finite difference in position (FD). In the FD setting, we added a Gaussian noise with a standard deviation of $0.03$ to the position.  
We evaluate the performance with different numbers of obstacles (2, 6, 10) and different obstacle speeds. In the LOW-velocity setting, the maximum speed of the obstacles is $50\%$ of the robot's maximum speed. The MEDIUM and HIGH-velocity settings are $100\%$ and $150\%$, respectively. We ran 1000 episodes for each experiment setting. The experiment hyperparameters are in Table~\ref{tab:2d_static_hyperparameters} in Appendix~\ref{app:hyperparameters}. 

The results are shown in Table~\ref{tab:experiment_dyn_env_obs}. The episode is successful if the robot reaches the without collision within a fixed horizon. As expected, the robot performs the best in all experiment settings when the exact velocity is known. However, the \gls{atacom} does not achieve a $ 100\% $ success rate in the HIGH-velocity and in the MEDUIM-velocity settings with 10 obstacles. This is because, in the HIGH-velocity setting, the velocity of the obstacle is higher than that of the robot. The controller can not fully compensate for the drift in such cases. In the MEDIUM-velocity setting with 10 obstacles, collisions occur when multiple constraints are active, and no feasible action exists to avoid the collision. To ensure safety in these scenarios, we need a better design of the constraint function, using more advanced techniques such as reachability analysis. Additionally, we can also observe that the velocity obtained by the finite difference method (FD) performs comparably well to the exact velocity in the LOW/MEDIUM-velocity setting and the fewer obstacle settings. Instead, assuming a static environment at each time step (NONE) does not give satisfactory results. Therefore, using an accurate velocity observer for the \gls{dus} in the dynamic environment is advisable.

\subsection{\reviseRtwo{Safe Control with Disturbances in QuadrotorEnv}\label{sec:experiment_quadrotor}}
\reviseRtwo{
In this section, we compare \gls{atacom} with \gls{cbf}-QP in a quadrotor environment. The task is to control the \mbox{Crazyflie-2} robot to track the target that moves in an $8$-shape while avoiding collisions with walls, as shown in Fig.~\ref{fig:quadrotor_env}. We use the quaternion-based nominal dynamics of the quadrotor following~\cite{song2022policy}. The state space is $13$-dimensional, and the control space is the 4-dimensional.  To account for environmental disturbances, we introduce a 3-dimensional translational Gaussian noise. The reward function is $r=\exp (-10 \Vert \vp - \vp_{t} \Vert^2)$ with $\vp$ and $\vp_t$ being the position of the quadrotor and the target. To ensure safety, the position of the quadrotor should stay with the boundary \underbar{$\vp$} $\leq \vp \leq \bar{\vp}$. 
Designing a valid constraint considering the input limit is a challenging problem. We use the following constraint function to derive a controller
\begin{align*}
    &\underbar{$\vp$}_{2} \leq \vp_2 + \zeta_p \vv_2 \leq \bar{\vp}_{2} \\
    &\lambda_r(\underbar{$\vp$}_{i} - \vp_i) \leq k_{i} + \zeta_r \dot{k}_{i} \leq \lambda_r(\vp_{i} - \vp_i) \quad i \in \{0, 1\}
\end{align*}
where the first constraint ensures the position in $z$-direction within the boundary. The second constraint restricts the attitude of the robot with respect to the distance to the boundary in $x, y$-directions. Here $k_{i} = \langle R_z(\bar{q}), \vn_i \rangle$ is the projection of the $z$-axis in the robot frame $R_z(\bar{q})$ to the normal direction of the boundary $\vn_i$ and $\bar{q}$ is the quaternion of the robot. $\zeta_p$, $\zeta_r$, $\lambda_r$ are some constant bigger than 0.
}

\reviseRtwo{We compared CBF-QP with two optimizers, OSQP (non-differentiable solver~\cite{osqp} and CvxpyLayer (differentiable)~\cite{agrawal2020differentiating}. The results are illustrated in Fig.~\ref{fig:experiment_quadrotor}.
\gls{atacom} demonstrates safer behavior and greater robustness to disturbances compared to the \gls{cbf} method. This superiority arises because \gls{atacom} adjusts the desired action as it approaches the boundary, whereas the QP-based method only modifies the desired action after a violation occurs. Additionally, \gls{atacom} is computationally more efficient than \gls{cbf}-QP. The gradient between the actual action and the sampled action is directly computable from an explicit mapping, while \gls{cbf}-QP relies on an implicit optimization process, making gradient computation more resource-intensive. However, the heavier computation in the \gls{rl} algorithm makes the speed-up effect marginal, such as saving 7 minutes from a training of 5 hours.} 


\begin{figure}[tb]
    \centering
    \includegraphics[width=0.95\linewidth, trim=0.5cm 0.5cm 0.5cm 0.6cm, clip]{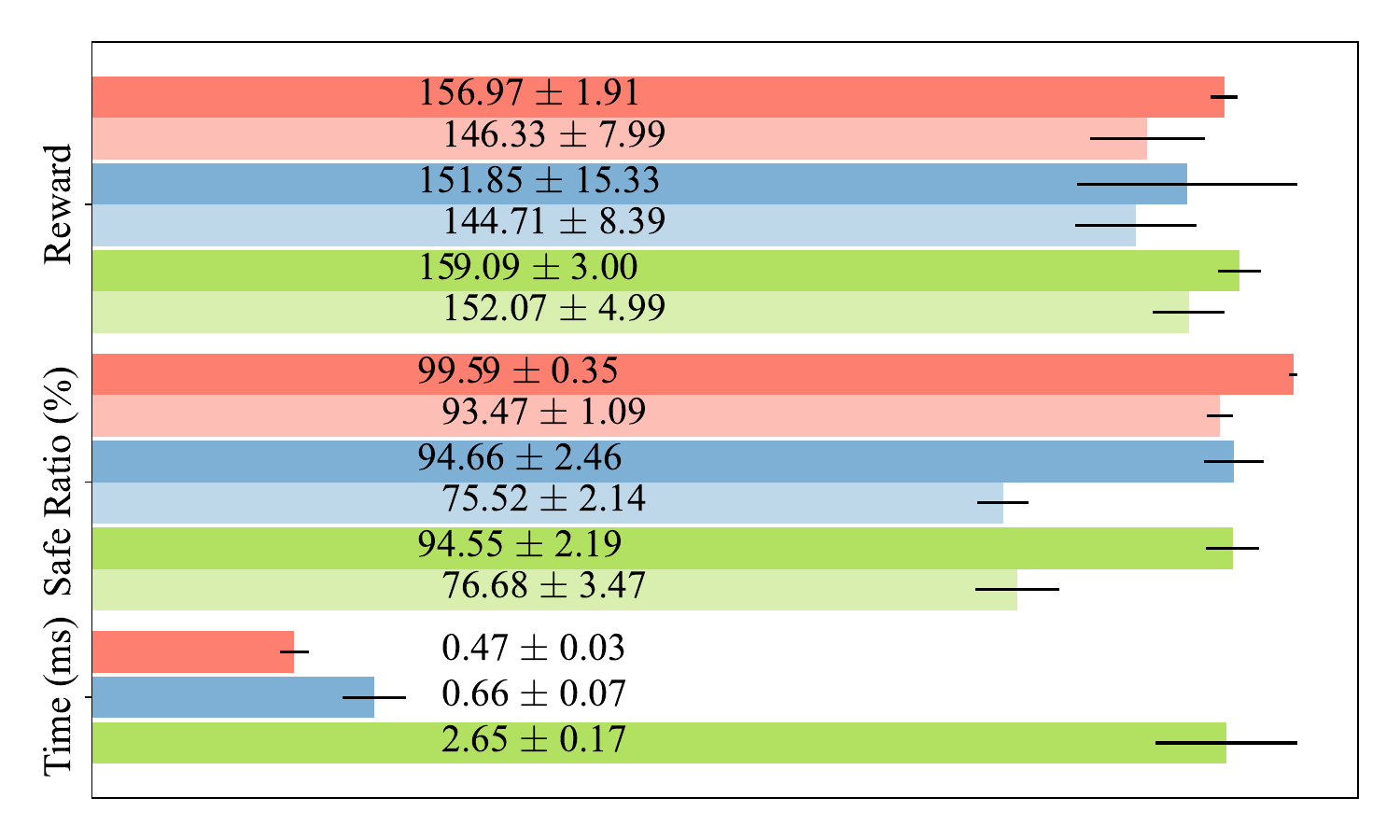}
    \includegraphics[width=0.8\linewidth, trim=1.2cm 2.0cm 2.8cm 1cm, clip]{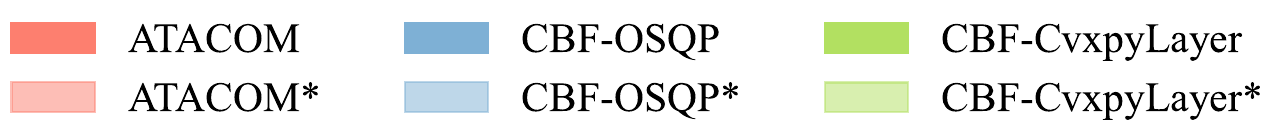}
    \caption{\reviseRtwo{Comparison of \gls{atacom} and \gls{cbf} was conducted using different optimizers: OSQP (non-differentiable) and CvxpyLayers (differentiable). The plots in solid colors represent results obtained from the ideal simulator, while the faded plots, marked with an asterisk (*), depict results from simulations with disturbances. The \textsl{Reward} represents the discounted cumulative reward of the final policy. The \textsl{Safe Ratio} indicates the proportion of safe episodes to the total number of episodes during the entire training process. \textsl{Time} measures the elapsed time for each action step.}}
    \label{fig:experiment_quadrotor}
    \vspace{-1em}
\end{figure}

\subsection{Dynamic Mismatch in Air Hockey Simulation} \label{sec:experiment_dyn_mismatch}
\gls{atacom} requires a known dynamics model of the \gls{dcs}. However, obtaining an accurate model of the robot dynamics can be challenging. The nominal model dynamics often differ from the real robot's ones. In the next experiment, we evaluate the performance under dynamics mismatch for \gls{atacom} in a simulated air hockey task.

The task requires controlling a KUKA iiwa14 robot to hit the randomly initialized puck into the opponent's goal. 
\reviseRone{The observation includes the robot's joint positions (7-dim) and velocities (7-dim), the puck's position (3-dim) and velocity (3-dim), the relative position (3-dim) and velocity (3-dim) between the puck and the robot's end-effector, and the action (7-dim) in the previous step. The total dimension of the observation space is 33. The action space of the \gls{rl} agent is the desired joint velocity (7-dim) of the robot.} 
A feedforward controller is then applied to generate joint torques tracking the desired joint position and velocity. 
The reward function is composed of three parts: one that rewards the end-effector for hitting the puck, one that expects the puck to move fast, and one that rewards for scoring.
\reviseRone{
Safety is ensured by imposing constraints in both the joint space and Cartesian space. In the joint space, 14 constraints limit the joint positions within their allowable ranges. For the Cartesian space, two constraints define the upper and lower bounds on the end-effector's height. Additionally, three constraints (two in the $Y$-direction and one in the $X$-direction) prevent the end-effector from colliding with the table's boundaries. To further ensure safety, two additional constraints restrict the height of the robot's elbow and shoulder, avoiding potential collisions with the table. A detailed description of the environment and hyperparameters can be found in Appendix~\ref{app:air_hockey_hyperparameters}. 
}

\begin{figure}
    \centering
    \includegraphics[width=\linewidth, trim=3.5cm 1.5cm 3cm 2.5cm, clip]{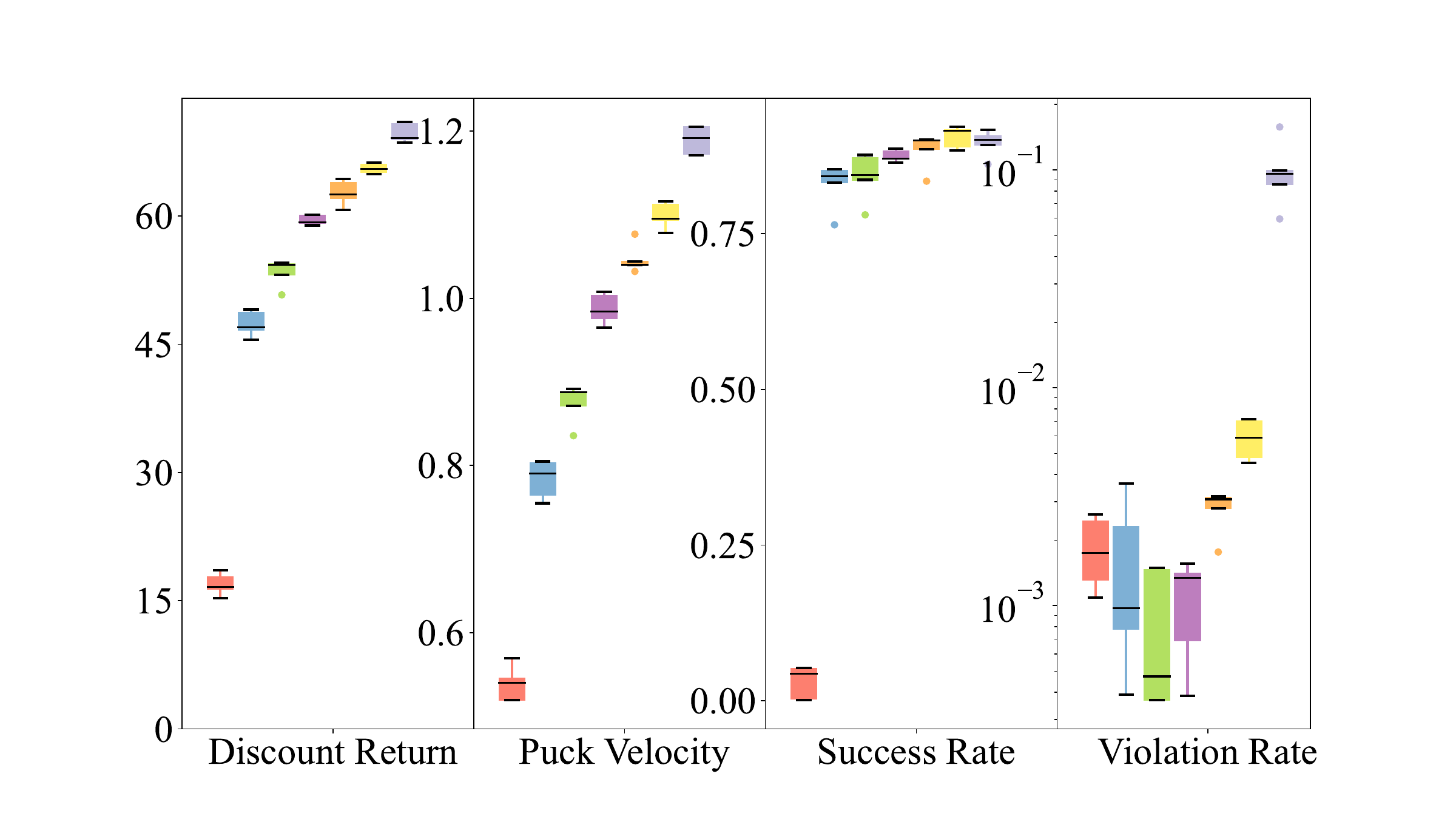}
    \includegraphics[width=0.8\linewidth, trim=1cm 0.5cm 1cm 1cm, clip]{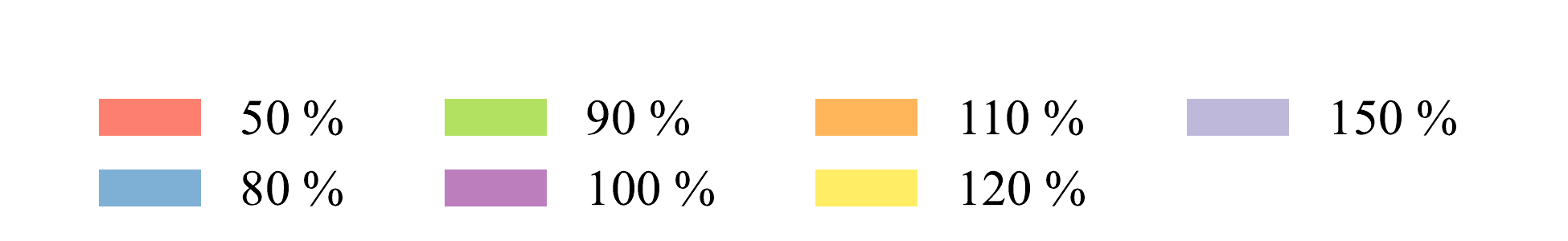}
    \caption{Comparison of the dynamic mismatch in the air hockey simulation. This figure shows the \textsl{Discounted Return}, the \textsl{Puck's Velocity} when crossing the middle line of the table, and the \textsl{Success Rate} evaluated over 1000 episodes. The \textsl{Violation Rate} is the percentage of episodes that one or more constraints are violated throughout the training process. The box plot shows the results over five independent runs.}
    \label{fig:air_hockey_dynamic_mismatch}
\end{figure}

To properly quantify the dynamics mismatch, we again use the velocity-controlled system as the nominal system for \gls{atacom} and change the maximum velocity of the nominal system to be different percentages of the actual velocity limit. We trained the policy using SAC for 3 million steps and then ran the evaluation for 1000 episodes. Fig.~\ref{fig:air_hockey_dynamic_mismatch} presents the results of 5 independent runs. We can observe that the performance (Discounted Return, Puck's Velocity, and Success Rate) of \gls{atacom} increases as the nominal system's velocity increases. The performance gain is because the agent can hit the puck at a higher velocity, and the puck will score in fewer steps. However, this performance improvement comes at the cost of constraint violations. The \gls{atacom} agent using an underestimated nominal dynamics model ($50\%, 80\%, 90\%$) generally has fewer violations than those using an overestimated model ($110\%, 120\%, 150\%$). The key reason is that the agent with an underestimated nominal dynamics model will always generate a feasible velocity command. In contrast, the agent with an overestimated nominal dynamics model will generate infeasible commands. The tracking controller will truncate the command to the velocity limit, which leads to unsafe behavior. 
Using a correct dynamics model, the agent achieves a good trade-off between performance and safety. In the experiment with an accurate nominal model, the violation rate is $0.1 \% $: these violations are caused mainly by the low control frequency ($50$Hz) and the controller's tracking error.

\subsection{Real Robot Air Hockey}
In the last experiment, we will show that \gls{atacom} achieves safe exploration on the \textit{Real Robot Air Hockey} task, enabling online \gls{rl} in the real world. The real robot setup is illustrated in Fig.~\ref{fig:air_hockey_real_world_env}. 
The performance of the Robot Air Hockey Hitting task is susceptible to the robot's and puck's dynamics. The success rate of the trained agent is $87\%$ in simulation, yet the success rate drops to $12\%$ when the policy is deployed in the real world. The main reason behind the performance drop is that successful hitting requires the robot to strike the puck in a specific direction precisely, but the dynamic mismatch between the simulator and the real world leads to deviations in the tracking controllers and the puck movements. While the robot can still hit the puck at high speeds, the success rate is low. Therefore, training with real-world interactions is essential to obtain a high-performance agent. 
Unfortunately, training the agent from scratch on a real robot is expensive and time-consuming. One training experiment with 3 million steps will take more than 100 hours, including the time for resetting the task. 
Instead, we train the agent in a simulated air hockey environment until convergence and then continue online fine-tuning using real-world interactions. 

Due to the sim-to-real gap, we observe that the value function obtained from the simulation has significant discrepancies with the real-world data. Direct training using the data collected from the real robot will destroy the value function landscape and rapidly lead to policy degradation. As a result, the agent will suffer from getting high-performance data and a slow convergence rate. Therefore, we first run a data-collection phase for 100 episodes using the pre-trained agent. We then train the value function using the newly collected data. This training phase only updates the value function and keeps the policy fixed. Then, we continue the training steps using the SAC algorithm. 
We apply two modifications to the SAC algorithm to fit the real-world setting: (1) We reduce the computation load during the episode rollout due to the real-time requirements: instead of updating the policy and the value function at each time step, we update the policy and the value function once the episode is finished. The number of updates equals the number of steps in the episode. (2) Due to the stochasticity of the real-world environment and the noise of the observations, outliers in the transition data can significantly affect the training of the value function. To reduce the impact of the outliers, we use a Huber loss instead of the Mean Square Error to train the value function. The Huber loss function is
\begin{align*}
    \mathcal{L}_{\delta}(y, \hat{y}) = 
    \left\{ \begin{aligned}
        &\frac{1}{2}(y - \hat{y})^2 & &\text{if } |y - \hat{y}| \leq \delta, \\
        &\delta |y - \hat{y}| - \frac{1}{2}\delta^2 & &\text{otherwise}.
    \end{aligned}
    \right.
\end{align*}

\begin{figure}[tb!]
    \centering
    \includegraphics[width=0.9\linewidth, trim=0cm 1.8cm 0cm 1cm, clip]{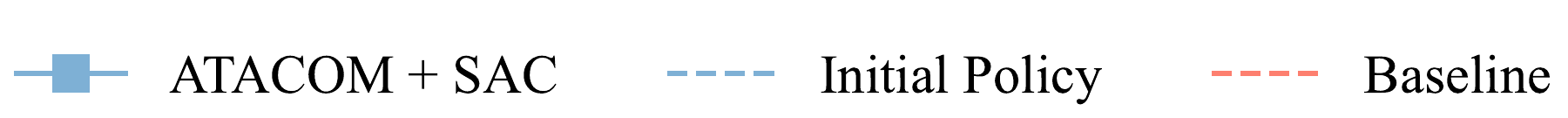} \\
    \includegraphics[width=\linewidth, trim=1cm 1cm 1cm 1cm, clip]{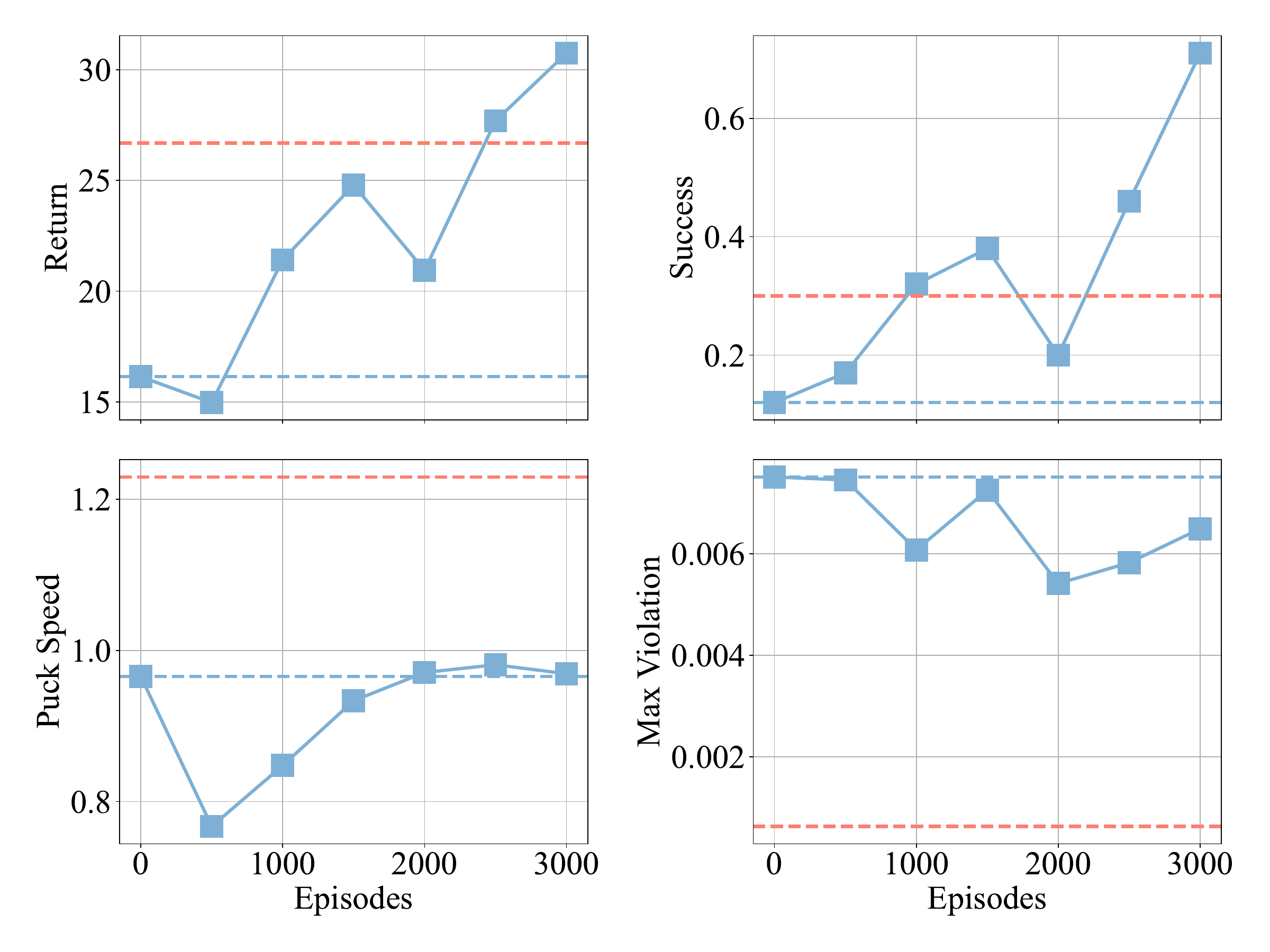}
    \caption{Learning Curve of the Real Robot Air Hockey Task. The evaluation is conducted every 500 episodes. In each evaluation, we ran 100 episodes with the puck placed uniformly in the pre-defined region. The baseline is a planning-based solution introduced in~\cite{liu2021efficient}.}
    \label{fig:air_hockey_real_learning}
\end{figure}

\begin{figure*}[tbp]
    \centering
    \includegraphics[width=0.24\linewidth, trim=2cm 7.5cm 2cm 7.5cm, clip]{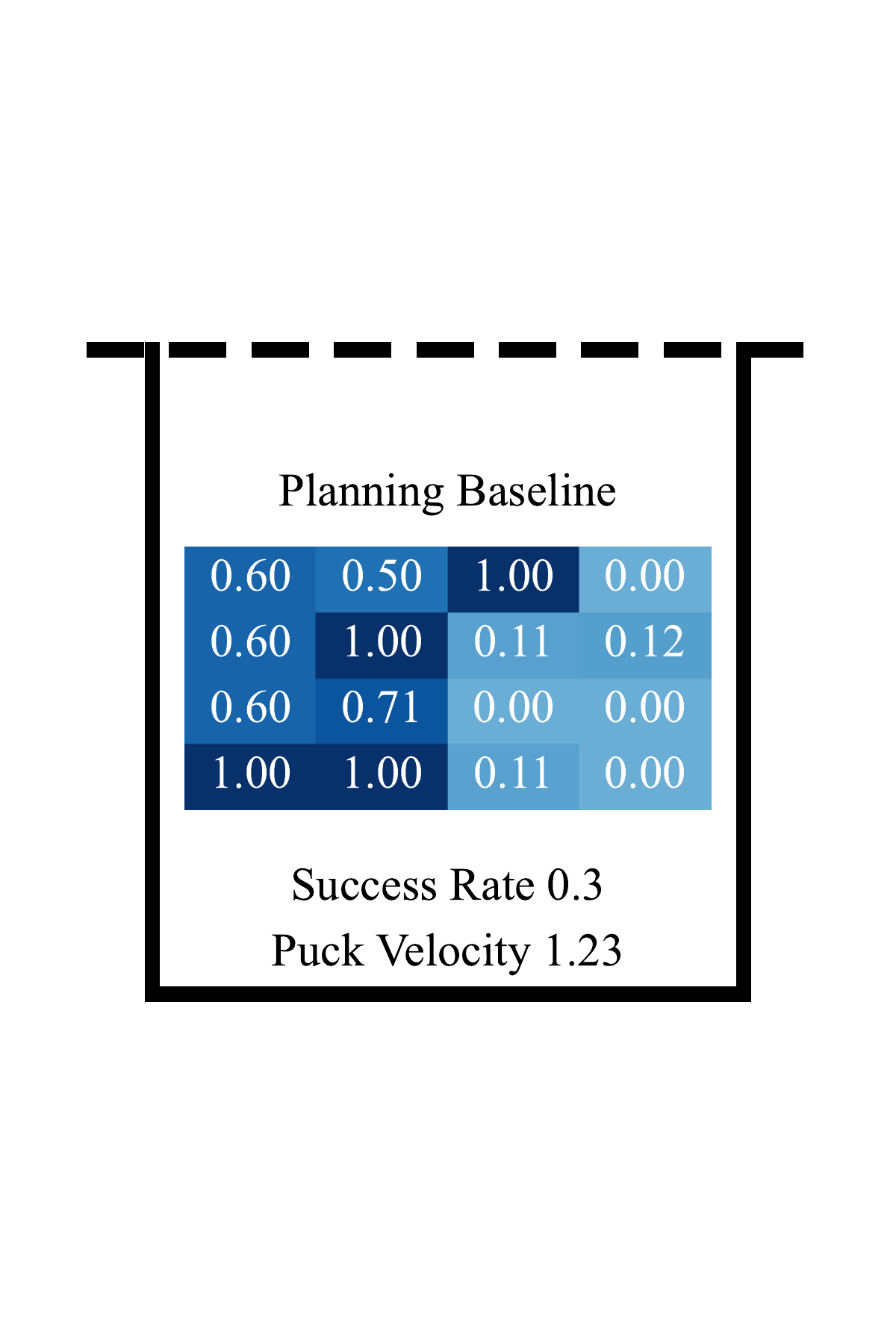}
    \hfill
    \includegraphics[width=0.24\linewidth, trim=2cm 7.5cm 2cm 7.5cm, clip]{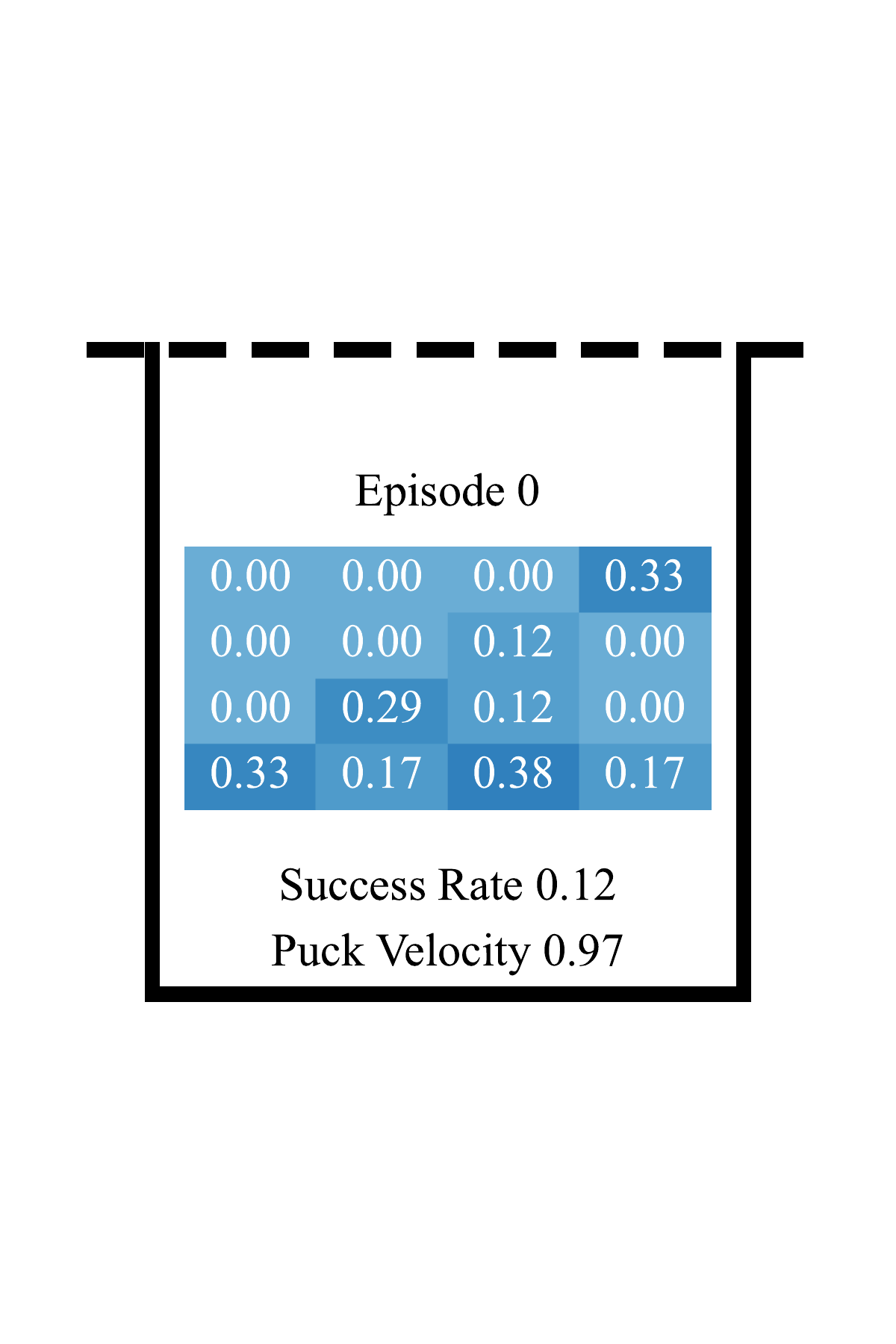}
    \includegraphics[width=0.24\linewidth, trim=2cm 7.5cm 2cm 7.5cm, clip]{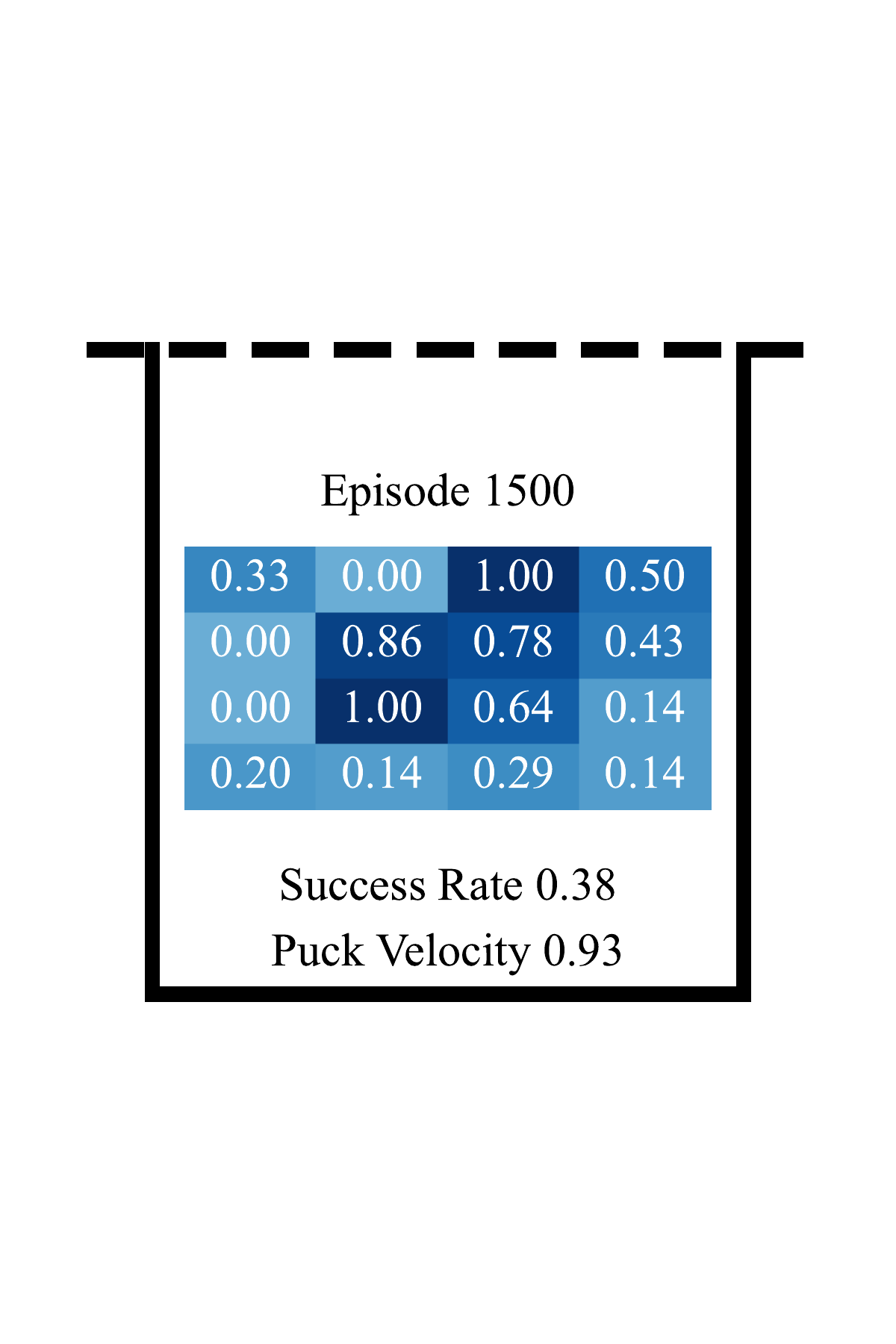}
    \includegraphics[width=0.24\linewidth, trim=2cm 7.5cm 2cm 7.5cm, clip]{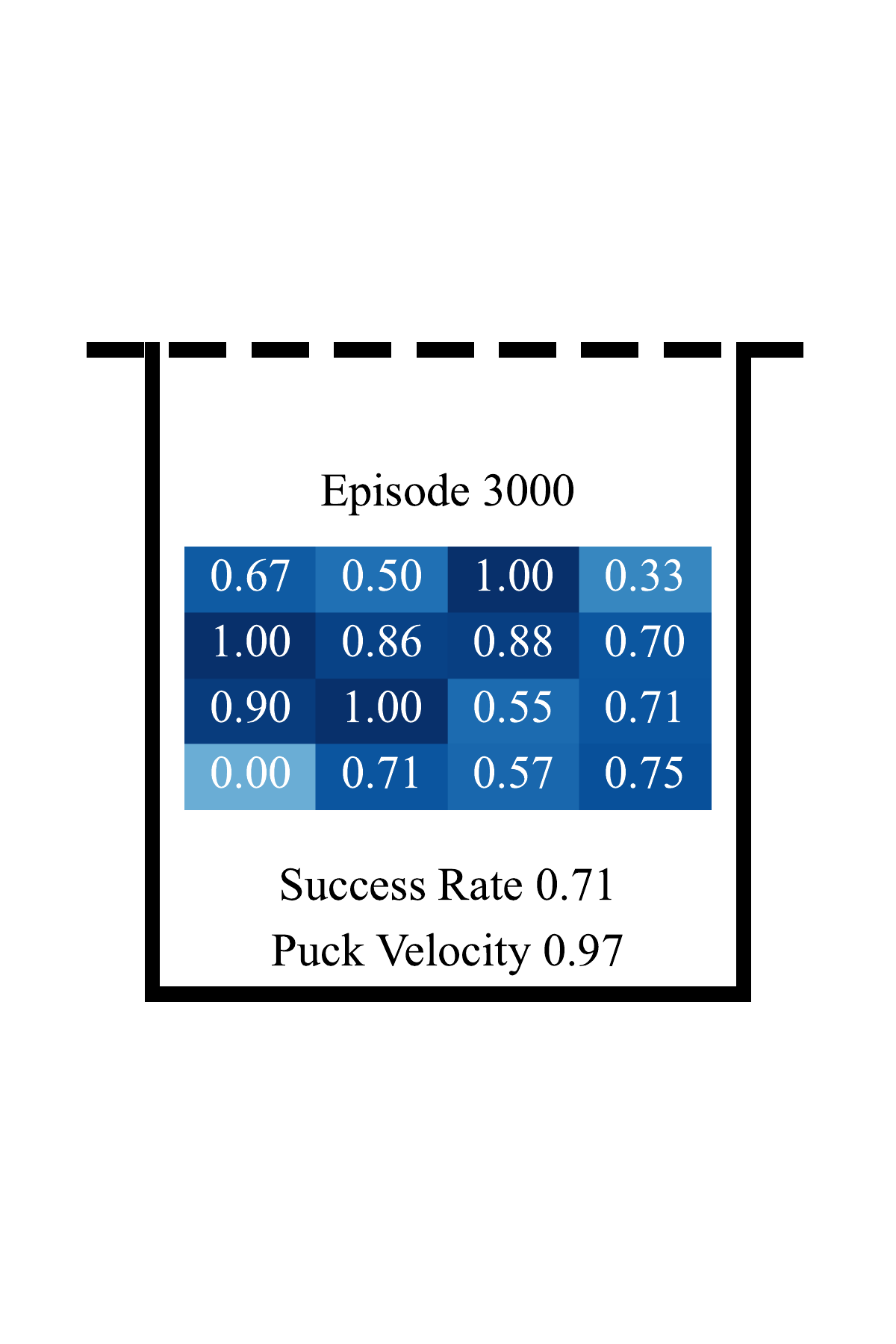}
    \caption{Success Rate of the Real Robot Air Hockey Task. This figure shows the table region of the trained agent's side. The dashed line represents the middle line of the table. The puck is placed uniformly on the table. The value in each cell represents the success rate in the corresponding areas. The figure on the left shows the success rate of the planning/optimization baseline. The second figure shows the result of the pre-trained agent. The third and last figures show the success rate of the trained agent after 1500 and 3000 episodes, respectively.} 
    \label{fig:air_hockey_real_success}
    \vspace{-1em}
\end{figure*}

Our study involved an extensive 3000 episodes of online fine-tuning on the real robot, a process that spanned approximately 24 hours. The resetting was executed by a Mitsubishi PA10 robot with a vacuum gripper, using a pre-programmed resetting motion. We conducted evaluations every 500 episodes, running 100 episodes with the puck placed uniformly in the pre-defined region. 

Fig.~\ref{fig:air_hockey_real_learning} shows the learning curve during the training process. The performance of the pre-trained agent is shown as a blue dashed line. The success rate of the pre-trained agent is $0.12$, and the average puck's velocity is $0.97$m/s. We also compare the planning/optimization-based baseline as introduced in~\cite{liu2021efficient}. The baseline has a success rate of $0.3$ and an average hitting velocity of $1.23$m/s. 
Notably, the performance at 500 episodes is worse than the initial policy. This is because the value function trained in the simulation does not reflect the performance of the policy in the real world, and the value function needs more iterations to reconstruct the landscape based on real-world dynamics. Since the agent has already acquired some high-performance data in the initial data collection phase, the value function will converge faster than if it is trained from scratch.
The success rate improved significantly to $0.71$ at the evaluation of 3000 episodes, and the puck's speed remains comparable at $0.97$m/s. The maximum violation throughout the evaluation is close to zero, as shown in the lower-right plot of Fig.~\ref{fig:air_hockey_real_learning}. Although the planning baseline is able to hit the puck at a higher speed, building a reactive and adaptive planner for high-speed dynamic motion is still challenging. The goal of the learning agent is focused primarily on scoring, which sacrifices hitting speed for better accuracy. We can also observe it from the last two steps in the learning curve. The success rate keeps increasing while the hitting velocity decreases slightly. 

We show the success rate when the puck is initialized at different table regions in Fig.~\ref{fig:air_hockey_real_success}. The planning/optimization-based solution is shown in the leftmost plot. The three plots on the right show the success rates at Episodes $0, 1500$, and $3000$. Due to the uneven air flow, the puck drifts heavier on the right side of the table than on the left. Therefore, the baseline agent performs worse on the right side as the planner does not adapt to the changes in the puck's motion. After the training, we can observe that the \gls{rl} agent successfully adapted its behavior, and the success rate is significantly improved. However, the puck's speed is still lower than the baseline as the baseline is optimized for the maximum speed, while the \gls{rl} agent only contains a small bonus for the high-speed puck. The RL agent needs to make a trade-off between high-speed motion and accuracy. Although other specialized solutions to solve the Air Hockey hitting task exist~\cite {kicki2023fast}, to the best of our knowledge, our method is the first approach able to learn and improve the performance of the real robot while ensuring safety at every timestep. We argue that the \gls{rl} agent could, in principle, outperform these baselines using more effective learning strategies and fine-tuning the reward design.
\section{Dicussions and Conclusions}
\label{sec:conclusions}
\subsection{Connections to Control Barrier Functions}
Control Barrier Functions (CBFs)~\cite{ames2019control} is a popular technique to enforce safety in control systems. Consider the safety specification $h(\StateVar) \geq 0$, where $h: \StateSpace \rightarrow \mathbb{R}$ is a continuously differentiable function. $h$ is a CBF if there exists an extended class $\mathcal{K}$ function $\alpha$ such that for the control system~\eqref{eq:nonlinear_affine_system}:
\begin{equation}
    \dot{h}(\StateVar, \ControlInput) \geq -\alpha(h(\StateVar)), \quad \exists \ControlInput \in \ControlSpace, \quad \forall \StateVar \in \StateSpace
    \label{eq:cbf_constraint}
\end{equation}
A safe controller is often obtained from a quadratic programming problem that finds the control input closest to the nominal one while satisfying the CBF constraint \eqref{eq:cbf_constraint}. \reviseRthree{Let $h(\StateVar)=-\Constr(\StateVar)=\SlackVar, \alpha(h(\StateVar))=-\ControlSlack\alpha'(h(\StateVar))$, the \gls{cbf} constraint results in a formulation similar to \gls{atacom}
\begin{equation*}
    \underbrace{\JacConstrMat\DynfVec}_{\DriftVec} + \underbrace{[\JacConstrMat \DynGMat \quad \text{diag}(\alpha'(\SlackVar)]}_{\JacInputMat}\begin{bmatrix} \ControlInput \\ \ControlSlack\end{bmatrix} \leq \vzero \quad \exists \ControlInput \in \ControlSpace
\end{equation*}
while \gls{atacom} considers the equality constraint~\eqref{eq:system_requirement}.
}


\reviseRthree{
However, the methods differ in how these constraints are solved. \gls{atacom} introduces a drift compensation term $-\JacInputMat^{\dagger}\DriftVec$ that solves the least-norm problem 
\begin{equation}
    \min \Vert [\ControlInput \quad  \ControlSlack]^{\intercal} \Vert ^2 \quad \mathrm{s.t. } \; \DriftVec + \JacInputMat \begin{bmatrix} \ControlInput \\ \ControlSlack\end{bmatrix} = \vzero, \ControlInput \in \ControlSpace
    \label{eq:atacom_driftcomp_problem}
\end{equation}
The external command for \gls{atacom} is applied in the tangent space, requiring only matrix decomposition and multiplication. 
Instead, a typical \gls{cbf} solution involves formulating a Least-square problem w.r.t the reference action $\vu_r$
\begin{equation*}
    \min \Vert \ControlInput - \vu_r||^2 \quad \mathrm{s.t.} \quad \DriftVec + \JacConstrMat \ControlInput \leq \vzero
\end{equation*}
Here, the external command is modified through an implicit optimization process, often involving iterative methods, such as the Active Set Method or the Interior Point Method~\cite{luenberger1984linear}. }

\reviseRthree{
A key advantage of \gls{atacom} is that it constructs a direct mapping (12), allowing the gradient $\nabla_{\vu}\ControlInput$ to be computed straightforwardly. In contrast, \gls{cbf} must rely on differentiable solvers, which can be computationally expensive. Gradient propagation is crucial when jointly learning constraints and policy~\cite{liu2024handling} or for constraint-agnostic task learning~\cite{emam2022safe}
Additionally, when no feasible action exists to satisfy the constraint. \gls{cbf}-based solution fails to find an alternative solution while \gls{atacom} provides a least square solution to~\eqref{eq:atacom_driftcomp_problem}}.

\subsection{Limitations}
\label{sec:limitations}
\gls{atacom} also exhibits some limitations. Our approach is derived from a time-continuous perspective. This requires a high control frequency. In the planar robot environment, the control frequency is set to be $\qty{100}{\hertz}$, and in the air hockey environment, the control frequency is $\qty{50}{\hertz}$. Extension to the time-discretized setting is left for future work. 
\reviseRtwo{
Furthermore, in our analysis, we assume a feasible action exists such that Equation~\eqref{eq:system_requirement} is solvable. The assumption can only be met by carefully designing the constraints. 
Constructing a safe constraint that considers the actuation limits is still an open problem. Learning-based methods have explored this problem, such as learning \gls{cbf} with actuation limit \cite{liu2023cbfsafe, dai2023convex}.}
Finally, \gls{atacom} focuses only on single-step constraints and does not account for long-term safety. For example, the robot may get trapped between multiple objects in the 2D Moving Obstacle environment. This type of safety can be addressed by restricting a cumulative cost in \gls{saferl} algorithms.

\subsection{Conclusions}
In this paper, we provided a theoretical foundation, a set of extensions, and through hyperparameters studies for \gls{atacom} algorithm. We showed that safety constraints can be constructed as a manifold. By exploiting the geometry of the tangent space, we can generate a safe action space, allowing learning agents to sample arbitrary actions while ensuring safety. The theoretical analysis demonstrated the existence of a region of attraction around the constraint manifold, guaranteeing the system converges to the manifold. Furthermore, we also analyze the safety of the system under disturbances will ensure bounded constraint violations. 
We demonstrated in a real-world experiment of the robot air hockey task that our approach can learn a safe policy in a high-dimensional task with complex safety constraints.

\section{Acknowledgments}
\noindent This project was supported by the CSTT fund from Huawei Tech R\&D (UK). We also acknowledge the support provided by the China Scholarship Council (No. 201908080039).
Research presented in this paper has been supported by the German Federal Ministry of Education and Research (BMBF) within the subproject ``Modeling and exploration of the operational area, design of the AI assistance as well as legal aspects of the use of technology'' of the collaborative KIARA project (grant no. 13N16274).

\bibliographystyle{IEEEtran}
\bibliography{bibliography.bib}

\appendices

\section{Experiment setup for the Robot Air Hockey Task}
\label{app:air_hockey_hyperparameters}

The observation space for the robot air hockey task is
$$
\StateVar = [\vp_{\text{puck}}, \vv_{\text{puck}}, \vq_{\text{robot}}, \dot{\vq}_{\text{robot}}, \vp_{\text{ee-puck}}, \vv_{\text{ee-puck}}, \vu_{\text{prev}}] ^\intercal
$$
where $\vp_{\text{puck}} \in \RR^3$ and $\vv_{\text{puck}} \in \RR^3$ are the position ($x,y$ position and yaw-angle) and velocity of the puck, $\vq_{\text{robot}} \in \RR^{7}$ and $\dot{\vq}_{\text{robot}} \in \RR^{7}$ are the joint positions and velocities of the robot, $\vp_{\text{ee-puck}} \in \RR^4$ and $\vv_{\text{ee-puck}} \in \RR^3$ are relative position (angle difference represent in ) and velocities between the puck and the velocity, and $\vu_{\text{prev}} \in \RR^7$ is the previous control input. The action space is $\ControlSpace \subset \RR^7$ is the desired joint velocity of the robot. To obtain a smooth action that is transferable to the real robot, we applied a low-pass filter to the sampled actions. 
$$\vu = r_u \vu_{\text{sample}} + (1 - r_u) \vu_{\text{prev}}$$
where $r_u$ is the smoothing ratio. 
The robot is controlled by torque in $1000$Hz and the RL action is sampled in $50$Hz. We applied a trajectory interpolator to convert the desired joint velocity to a chunk of trajectories in $1000$Hz. Then a PD controller is used to track the trajectory. 
The constraints are 
\begin{align*}
    &q_{i, l} < q_i < q_{i, u}, \quad i = 1, \cdots, 7 \\
    &x_l < x_{\text{ee}} ,  \quad y_l < y_{\text{ee}} < y_u,  \quad z_l < z_{\text{ee}} < z_u \\
    &z_l' < z_{\text{wrist}} \quad z_l' < z_{\text{elbow}} 
\end{align*}
where $q_i$, $q_{i, l}$ and $q_{i, u}$ are the position, lower and upper limits of the joint $i$. $[x_{\text{ee}}, y_{\text{ee}}, z_{\text{ee}}]^\intercal = \text{FK}(\vq)$ is the end-effector position obtained from the forward kinematics. $z_{\text{wrist}}$ and $z_{\text{elbow}}$ are the $z$-axis position of the wrist and elbow. The total number of constraints is 21. 

The reward function is defined as follows
\begin{align*}
    r(\StateVar, \ControlInput) &= r_{\text{pe}} + r_{\text{pv}} + r_{\text{pg}},  \\
    r_{\text{pe}} &= 10 \max(d_{\text{pe}} - \Vert \vp_{\text{puck}} - \vp_{\text{ee}} \Vert, 0) \\
    r_{\text{pv}} &= 1.5 \dot{x}_{\text{puck}}, \text{ if } x_{\text{puck}} > 0 \\
    r_{\text{pg}} &= (1.5 - \Vert \vp_{\text{puck}} - \vp_{\text{goal}} \Vert) / (1 - \gamma), \text{ if episode ends}
\end{align*}
where $r_{\text{pe}}$ defines the reward that encourages the end-effector to approach the puck, $d_{\text{pe}}$ is the minimum distance between the puck and the end-effector in the episode. $r_{\text{pv}}$ is the reward that encourages the puck to move at a higher speed. $r_{\text{pg}}$ is the reward at the final step of the episode, encouraging the puck to reach the goal. The parameters for the \gls{rl} training are shown in Table~\ref{tab:air_hockey_hyperparameters}.

\section{Hyperparameters}
\label{app:hyperparameters}

\begin{table}[H]
    \centering
    \footnotesize
    \begin{tabular}{lc}
        \hline
        Parameter & Value \\ 
        \hline
        slack dynamics function & [exp, linear] \\
        $\beta$ & [0.3, 1.0, 3.0, 10.0]\\
        \hline
        \# epochs & 100 \\
        \# steps per epoch & 10000 \\
        \# steps per fit & 1 \\
        \# evaluation episodes & 10 \\
        actor lr & 0.0001 \\
        critic lr & 0.0003 \\
        alpha lr & 5e-06 \\
        hidden layers & 128 128 128 \\
        batch size & 64 \\
        initial replay size & 10000 \\
        \# warmup transitions & 10000 \\
        max replay size & 200000 \\
        target network update ratio & 0.001 \\
        target entropy & -2 \\
        drift clipping & true \\
        slack tolerance & 1e-06 \\  
        \hline
    \end{tabular}
    \vspace{0.5em}
    \normalsize
    \caption{Parameters for the 2D Static Environment}
    \label{tab:2d_static_hyperparameters}
    \vspace{-1em}
\end{table}

\begin{table}[H]
    \centering
    \footnotesize
    \begin{tabular}{lc}
        \hline
        Parameter & Value \\ 
        \hline
        \# obstacles & [2, 6, 10] \\
        obstacles' velocity obs. & [EXACT, NONE, FD]\\
        obstacles' velocity scales & [50\%, 100\%, 150\%] \\
        obstacles' motion pattern & [FIXED, RANDOM] \\
        \hline
        \# episodes & 1000 \\
        \# horizon & 1000 \\
        drift clipping & true \\
        slack dynamics function & exp \\
        slack $\beta$ & 4 \\
        slack tolerance & 1e-06 \\
        \hline
    \end{tabular}
    \vspace{0.5em}
    \normalsize
    \caption{Parameters for the 2D Dynamic Environment}
    \label{tab:2d_dynamic_hyperparameters}
    \vspace{-1em}
\end{table}

\begin{table}[ht]
    \centering
    \footnotesize
    \begin{tabular}{lc}
        \hline
        Parameter & Value \\ 
        \hline
        dynamic mismatch (\%) & [50, 80, 90, 100, 110, 120, 150] \\
        \hline
        \# epochs & 300 \\
        \# steps per epoch & 10000 \\
        \# steps per fit & 1 \\
        \# evaluation episodes & 1000 \\
        actor lr & 0.0003 \\
        critic lr & 0.0003 \\
        alpha lr & 5e-05 \\
        hidden layers & 128 128 128 \\
        activation function & SELU \\
        batch size & 64 \\
        initial replay size & 10000 \\
        \# warmup transitions & 10000 \\
        max replay size & 200000 \\
        target network update ratio & 0.001 \\
        target entropy & -2.0 \\
        drift clipping & true \\
        slack tolerance & 1e-06 \\  
        slack dynamics function & exp \\
        slack $\beta$ & 2.0 \\
        action filter ratio $r_u$ & 0.3 \\
        \hline
    \end{tabular}
    \vspace{0.5em}
    \normalsize
    \caption{Hyperparameters for Real Robot Air Hockey Experiment}
    \label{tab:air_hockey_hyperparameters}
\end{table}

In the real-world experiment, the pre-training in simulation is conducted with the same hyperparameters as in Table~\ref{tab:air_hockey_hyperparameters}. Here we list the hyperparameters for the real robot experiment.

\begin{table}[H]
    \centering
    \footnotesize
    \begin{tabular}{lc}
        \hline
        Parameter & Value \\ 
        \hline
        \# episodes & 3000 \\
        \# pretraining episodes & 100 \\
        \# value function pre-training & 100 \\
        \# evaluation episodes & 100 \\
        initial replay size & 4096 \\
        warmup transitions & 4096 \\
        replay buffer size & 150000 \\
        actor lr & 0.0001 \\
        critic lr & 0.0001 \\
        target entropy & linear decay from -2 to -7 \\
        value function loss & Huber Loss \\
        \hline
    \end{tabular}
    \vspace{0.5em}
    \normalsize
    \caption{Parameters for the Real Robot Air Hockey Task}
    \label{tab:air_hockey_real_hyperparameters}
\end{table}
\color{black}
\begin{IEEEbiography}[{\includegraphics[width=1in,height=1.25in,clip,keepaspectratio]{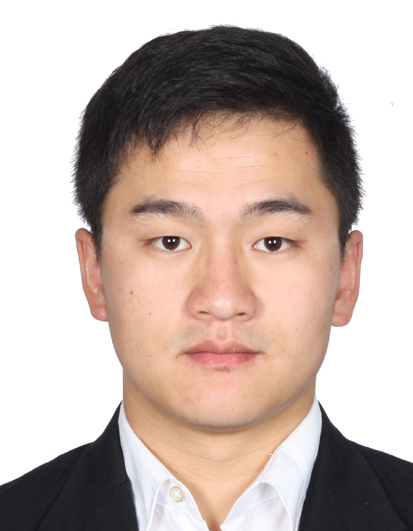}}]{Puze Liu}
is pursuing his Ph. D. degree at Intelligent Autonomous Systems Group, Technical University Darmstadt since 2019. Prior to this, Puze received his M. Sc. in Computational Engineering from Technical University Berlin and B. Sc from Tongji University, China. Puze's research interest lies in the interdisciplinary field of robot learning that tries to integrate machine learning techniques into robotics. His prior work focuses on optimization, control, reinforcement learning, and safety in robotics.
\end{IEEEbiography}

\vskip -1.9\baselineskip plus -1fil
\begin{IEEEbiography}[{\includegraphics[width=1in,height=1.25in,clip,keepaspectratio]{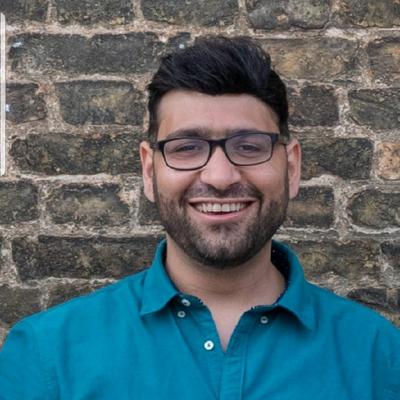}}]{Haitham Bou-Ammar}
leads the reinforcement learning team at Huawei Technologies Research \& Development UK and is an Honorary Lecturer at UCL. 
His primary research interests lie in the field of statistical machine learning and artificial intelligence, focusing on Bayesian optimization, probabilistic modeling, and reinforcement learning. He is also interested in learning using massive amounts of data over extended time horizons, a property common to "Big-Data" problems.
His research also spans different areas of control theory, nonlinear dynamical systems, social networks, and distributed optimization.
\end{IEEEbiography}

\vskip -1.9\baselineskip plus -1fil
\begin{IEEEbiography}[{\includegraphics[width=1in,height=1.25in,clip,keepaspectratio]{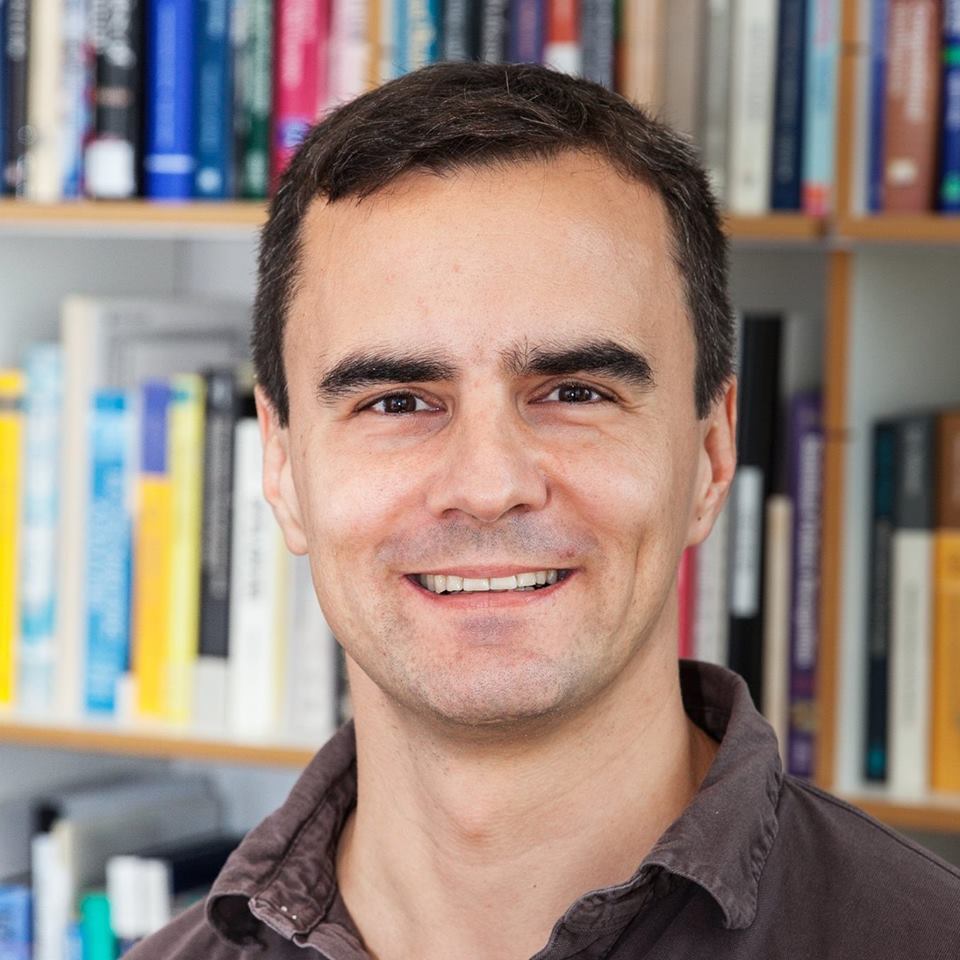}}]{Jan Peters}
is a full professor (W3) for Intelligent Autonomous Systems at the
Computer Science Department of the Technische Universitaet Darmstadt.
Jan Peters has received the Dick Volz Best 2007 US Ph.D. Thesis Runner-Up Award, the Robotics: Science \& Systems - Early Career Spotlight, the INNS Young Investigator Award, and the IEEE Robotics \& Automation Society's Early Career Award as well as numerous best paper awards. In 2015, he received an ERC Starting Grant and in 2019, he was appointed as an IEEE Fellow.
\end{IEEEbiography}

\vskip -1.9\baselineskip plus -1fil
\begin{IEEEbiography}[{\includegraphics[width=1in,height=1.25in,clip,keepaspectratio]{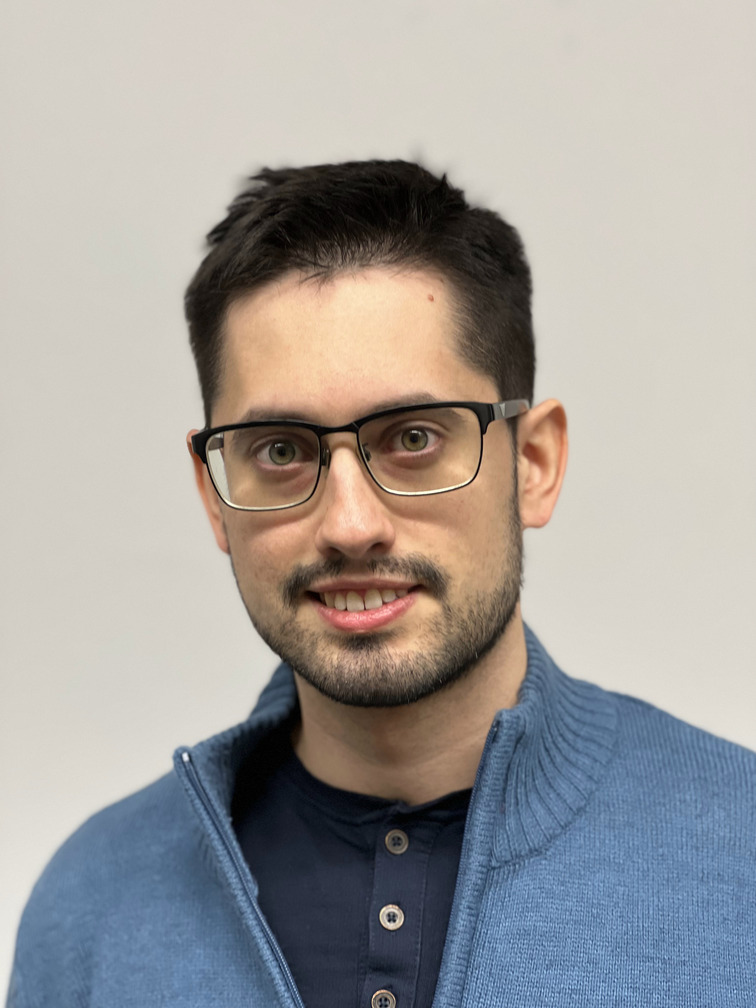}}]{Davide Tateo}
is a Research Group Leader at the Intelligent Autonomous Systems Laboratory in the Computer Science Department of the Technical University of Darmstadt. 
He received his M.Sc. degree in Computer Engineering at Politecnico di Milano in 2014 and his Ph.D. in Information Technology from the same university in 2019. 
Davide Tateo worked in many areas of Robotics and Reinforcement Learning, Planning, and Perception.
His main research interest is Robot Learning, focusing on high-speed motions, locomotion, and safety.
\end{IEEEbiography}

\end{document}